\documentclass[authoryear,preprint,11pt]{elsarticle}

\usepackage[utf8]{inputenc}
\usepackage{circuitikz}
\usetikzlibrary{circuits.logic}
\usetikzlibrary{circuits.logic.US}
\usetikzlibrary{calc}
\usepackage{amsthm,amsmath,amssymb,eurosym}
\usepackage{comment}
\usepackage{natbib}
\usepackage{graphicx}
\usepackage[a4paper, total={8in, 8in}]{geometry}
\usepackage{todonotes}
\usepackage{caption}
\usepackage[labelformat=simple]{subcaption}
\usepackage{grfext}
\usepackage{enumitem}
\usepackage{hyperref}
\usepackage{txfonts}
\usepackage[normalem]{ulem}

\usepackage[linesnumbered,ruled,vlined]{algorithm2e}

\newtheorem{theorem}{Theorem}

\newtheorem{lemma}{Lemma}

\newtheorem{example}{Example}

\newcommand{\ie}{i.e.}% id est  --       that is
\newcommand{\eg}{e.g.}%
\newcommand{\os}{\mathbb{R}^p}
\newcommand{\oss}{\mathbb{Z}_{\geqslant 0}} % single-objective
\newcommand{\osn}{\mathbb{Z}^p_{\geqslant 0}}
\newcommand{\osp}{\mathbb{R}^p_+}

\newcommand{\Nmp}{\mathbb{Z}_{\geqslant0}^n}
\newcommand{\Bmp}{\{0,1\}^n}

\newcommand{\Aeps}{A_{\varepsilon}}
\newcommand{\Ac}{\mathcal{A}}
\newcommand{\Lc}{\mathcal{L}}

\newcommand{\cl}{{\emph c}}
\newcommand{\hc}{{\phi^h}}

\newcommand{\sfc}{{\phi^s}}

\newcommand{\mcskp}{\texttt{MCS-Approx}\xspace}
\newcommand{\rib}{\texttt{IntRe}\xspace}
\newcommand{\rcb}{\texttt{CoRe}\xspace}
\newcommand{\bdd}{\texttt{BDD}\xspace}
\newcommand{\pmcs}{\texttt{ParetoMCS}\xspace}

 \newcommand{\wdom}{\preceqq} %weak dominance
 \newcommand{\dom}{\preceq} % dominance
 \newcommand{\sdom}{\prec} % strict dominance

\newtheorem{problem}{Problem}

\newboolean{cleancomments}
\setboolean{cleancomments}{false}    % Comments are shown

\ifthenelse{\boolean{cleancomments}}
{
\newcommand\ag[1]{}
\newcommand\agi[1]{}
\newcommand\jc[1]{}
\newcommand\jci[1]{}
\newcommand\vmm[1]{}
\newcommand\vmmi[1]{}
\newcommand\dv[1]{}
\newcommand\dvi[1]{}
\newcommand\cb[1]{}
\newcommand\cbi[1]{}
\newcommand\jrf[1]{}
\newcommand\jrfi[1]{}
}
{
\newcommand\ag[1]{\todo[color=cyan]{{\textbf{AG:} #1}}}
\newcommand\agi[1]{\todo[inline,color=cyan]{{\textbf{AG:} #1}}}

\newcommand\jc[1]{\todo[color=green,size=\footnotesize]{{\textbf{JC:} #1}}}
\newcommand\jci[1]{\todo[inline,color=green]{{\textbf{JC:} #1}}}

\newcommand\vmm[1]{\todo[color=olive]{{\textbf{VMM:} #1}}}
\newcommand\vmmi[1]{\todo[inline,color=olive]{{\textbf{VMM:} #1}}}

\newcommand\dv[1]{\todo[color=orange]{{\textbf{DV:} #1}}}
\newcommand\dvi[1]{\todo[inline,color=orange]{{\textbf{DV:} #1}}}

\newcommand\cb[1]{\todo[color=purple]{{\textbf{CB:} #1}}}
\newcommand\cbi[1]{\todo[inline,color=purple]{{\textbf{CB:} #1}}}

\newcommand\jrf[1]{\todo[color=yellow]{{\textbf{JRF:} #1}}}
\newcommand\jrfi[1]{\todo[inline,color=yellow]{{\textbf{JRF:} #1}}}

}

\usepackage{comment}

\graphicspath{
    {./results/set-covering-band/}
    {./results/sampleB/}
    {./results/evol-set-covering/}
    {./figures/}
    }
\PrependGraphicsExtensions*{.pdf}

\makeatletter
\def\ps@pprintTitle{%
  \let\@oddhead\@empty
  \let\@evenhead\@empty
  \def\@oddfoot{\reset@font\hfil\thepage\hfil}
  \let\@evenfoot\@oddfoot
}
\makeatother

\journal{Computers \& Operations Research}

\begin{document}

\begin{frontmatter}
    \title{Exact and approximate determination of the Pareto set using minimal
correction subsets}

    \author[inesc-id]{{\sc A. P. Guerreiro}\corref{cor1}}\ead{andreia.guerreiro@tecnico.ulisboa.pt}
    \author[inesc-id]{{\sc J. Cortes}}
    \author[lamsade]{{\sc D. Vanderpooten}}
    \author[lamsade]{{\sc C. Bazgan}}
    \author[inesc-id]{{\sc I. Lynce}}
    \author[inesc-id]{{\sc V. Manquinho}}
    \author[cegist]{{\sc J. R. Figueira}}
    \address[inesc-id]{INESC-ID, Instituto Superior T\'{e}cnico, Universidade de Lisboa, Portugal}
    \address[lamsade]{LAMSADE (CNRS UMR7243), University Paris-Dauphine, PSL Research University, France}
    \address[cegist]{CEGIST, Instituto Superior T\'{e}cnico,  Universidade de Lisboa, Portugal}

    \cortext[cor1]{Corresponding author}

    \begin{abstract}
    \noindent Recently, it has been shown that the enumeration of Minimal Correction Subsets (MCS) of Boolean formulas allows solving Multi-Objective Boolean Optimization (MOBO) formulations. However, a major drawback of this approach is that most MCSs do not correspond to Pareto-optimal solutions. In fact, one can only know that a given MCS corresponds to a Pareto-optimal solution when all MCSs are enumerated. Moreover, if it is not possible to enumerate all MCSs, then there is no guarantee of the quality of the approximation of the Pareto frontier. This paper extends the state of the art for solving MOBO using MCSs. First, we show that it is possible to use MCS enumeration to solve MOBO problems such that each MCS necessarily corresponds to a Pareto-optimal solution. Additionally, we also propose two new algorithms that can find a $(1+\varepsilon)$-approximation of the Pareto frontier using MCS enumeration. Experimental results in several benchmark sets show that the newly proposed algorithms allow finding better approximations of the Pareto frontier than state-of-the-art
    algorithms, and with guaranteed approximation ratios.
    \end{abstract}
    %%
    %\vspace{0.25cm}
    %%

\end{frontmatter}

% \vfill\newpage

% \tableofcontents

% \vfill\newpage

\section{Introduction}
\label{sec:intro}

\noindent In the last decade, new algorithms and tools for Boolean optimization
have been developed that rely on successive calls to a highly effective
Propositional Satisfiability (SAT) solver. As a result, these new algorithms
have been used in several application domains such as 
timetabling~\citep{asin2012curriculum}, fault localization in {\tt C} programs~\citep{majumdar-pldi11}, and
design debugging~\citep{veneris-fmcad07}, 
among others~\citep{DBLP:conf/ndss/FengBMDA17}.

Despite its success in Boolean optimization with a single objective function, 
there are few algorithms for Multi-Objective Boolean Optimization (MOBO)
that are based on successive calls to a SAT solver.
However, there are some exceptions. For instance, the Guided Improvement Algorithm
(GIA)~\citep{jackson2009guided}, is implemented in the optimization engine of the
Z3 solver~\citep{DBLP:conf/tacas/BjornerPF15} for finding Pareto-optimal solutions of 
Satisfiability Modulo Theories (SMT) instances with multiple objective functions.
More recently, new logic-based algorithms have been proposed 
that are based on the enumeration of Minimal Correction Subsets
(MCSs)~\citep{DBLP:conf/sat/Terra-NevesLM17} or 
$P$-minimal models~\citep{DBLP:conf/cp/SohBTB17}. 
These logic-based algorithms have shown to be successful at solving MOBO 
instances when the constraint set is hard to satisfy. 

Logic-based algorithms suffer from some drawbacks. For instance,
if not all MCSs are enumerated, there is no guarantee on the quality
of the provided solutions. Similarly, while a $P$-minimal model 
corresponds to a Pareto-optimal solution, if not all $P$-minimal models
are found, there is no guarantee that the Pareto-optimal solutions found
are representative of the Pareto front. Moreover, this approach also requires
an expensive encoding of the objective functions into SAT. In some cases,
the SAT formula becomes so large that the SAT solver is unable to handle the
formula.

In this work, logic-based algorithms for MOBO are extended in several ways.
First, we start by proving that if one uses the order encoding to encode the
objective functions into SAT, then there is a one-to-one correspondence
between an MCS and a Pareto-optimal solution. Moreover, we propose new
logic-based algorithms to obtain an \emph{a priori} guaranteed $(1+\varepsilon)$-approximation of the Pareto 
frontier. These algorithms are based on approximate encodings of the objective 
functions that are usually much smaller than the full encoding into SAT.
Finally, an extensive experimental analysis is performed on the proposed
algorithms which includes a comparison with other state-of-the-art logic-based algorithms 
for MOBO.

The paper is organized as follows. The background information is introduced in Section~\ref{sec:background}. New ways of encoding the objective functions and
their properties are explained in Section~\ref{s:contr}, these are at the core of the
approximation algorithms introduced in Section~\ref{s:algs}. The algorithms are then
experimentally compared against state-of-the art ones in Section~\ref{s:exp}. Finally,
concluding remarks are drawn in Section~\ref{s:conc}.

% \begin{itemize}
%    \item Motivate the paper (refer to the new constraint-based algorithms for MOCO).
%    \item Mention the several application domains.
%    \item List contributions
%     \item Paper structure
% \end{itemize}

\section{Background: Concepts, definitions, and notation}
\label{sec:background}

\noindent This section provides the relevant background information for the paper. 
The formalization of multi-objective Boolean optimization problems is introduced, 
as well as the main concepts, definitions and corresponding notation used in the remainder
of the paper.

\subsection{Multi-objective Boolean optimization}
\label{sec:mobo}

\noindent The Multi-Objective Boolean Optimization (MOBO) problem can be defined as
minimizing $p$ objective functions defined over a set of $n$ (Boolean) variables and $m$ constraints as follows:

\begin{problem}[Multi-Objective Boolean Optimization]
\begin{equation*}
    \begin{array}{llr}
        \min  & f(x)= \left(f_1(x) = \sum_{j=1}^n c_j^1 l_j,\ldots,f_k(x) = \sum_{j=1}^n c_j^k l_j ,\ldots,f_p(x) = \sum_{j=1}^n c_j^p l_j\right)\\
        \mbox{\rm{subject to}:} \\
        & \sum\limits_{j=1}^n a_{ij} l_j \geqslant b_i, &  i \in \{1 \ldots m \}\\
        & l_j \in \{x_j, \bar{x}_j\},\, x_j \in \{ 0, 1 \}, & j \in \{1 \ldots n\}
    \end{array}
\end{equation*}
\label{prob:MOBP}
\end{problem}

In this formulation, constraints are linear Pseudo-Boolean 
constraints~\citep{DBLP:series/faia/RousselM09}, and each objective function $f_k(x)$ is 
also a linear expression defined over a set of Boolean variables.
Each coefficient, $a_{ij}$, and each right-hand-side, $b_i$, are assumed to be non-negative integers. This is not restrictive since any  negative term such as $-a_{ij}x_j$ with $a_{ij}>0$ can be replaced by $a_{ij}\bar{x}_j - a_{ij}$, where $\bar{x}_j=1-x_j$, to obtain an equivalent expression with only non-negative coefficients.
For example, 
$-2x_1 - 3x_2 + 2 x_3 \geqslant -2 \Leftrightarrow 2(1-x_1) + 3(1-x_2) + 2 x_3 \geqslant -2 + 2 + 3 \Leftrightarrow 2\bar{x}_1 + 3\bar{x}_2 + 2 x_3 \geqslant 3$.
The transformed constraint always has a non-negative right-hand side. Otherwise, it would be trivially satisfied and could be safely removed.

Let $X\subseteq\Bmp$ denote the set of feasible solutions in the decision space ($\Bmp$),
\ie, the set of solutions that satisfy the problem constraints.
The vector-valued function $f:X\to\osn$ maps a feasible solution to an (objective) vector, or point, in the objective space, $\osn$, where $p$ is the number of objectives.
The set $Y\subseteq\osn$ is the image set of $X$, \ie, $Y=f(X)$, which represents the set of feasible vectors, or feasible points.
Hence, to each feasible solution $x\in X$ corresponds an objective vector 
$f(x) = (f_1(x),\ldots,f_p(x))$, where $f(x)\in Y$. 

Given two points in the objective space, $z,z'\in\osn$, $z$ is said to \emph{weakly dominate} $z'$ if $z_k\leqslant z'_k$ for all $k\in\{1,\ldots,p\}$, where $z_k$ and $z'_k$ denote the $k$-th
coordinate of $z$ and $z'$, respectively. This is represented as $z \wdom z'$. 
Point $z$ is said to \emph{dominate} $z'$ if $z \wdom z'$ and $z \neq z'$. 
This is represented as $z \dom z'$. Point $z$ is said to \emph{strictly dominate} $z'$ if $z_k < z'_k$ for all $k\in\{1,\ldots,p\}$.
This is represented as $z \sdom z'$.
The points $z$ and $z'$ are said to be \emph{incomparable}, or \emph{mutually nondominated}, if neither $z\wdom z'$ nor $z'\wdom z$~\citep[see, e.g.,][]{DBLP:books/daglib/Ehrgott05}.
Weak dominance may be extended to point sets. Given
$A,A'\subset\os$, $A$ is said to weakly dominate $A'$ if, for each $z'\in A'$,
there exists $z\in A$ such that $z\wdom z'$~\citep{ZitzlerTLFF03}. 

A feasible solution $x$ is said to be an \emph{efficient} (or \emph{Pareto-optimal}) solution
if there is no other solution $x' \in X$ such that $f(x')\dom f(x)$, in which case $f(x)$ 
is said to be a \emph{nondominated point}. 
The set of all efficient solutions is known as the
\emph{efficient set} (or \emph{Pareto set}). The corresponding image set is known as
\emph{the nondominated set} (or \emph{Pareto front}), which is here represented by $Y_N\subseteq Y$.
Given a MOBO formulation, the goal is (typically) to find the set of all nondominated points
and an efficient solution corresponding to each such point.
If enumerating all nondominated points is not possible (\eg, it takes too long, or $Y_N$ is too
large), an alternative is to find a $(1+\varepsilon)$-approximation.
A set~$A_{\varepsilon}\subseteq X$ of feasible solutions is called a \emph{$(1+\varepsilon)$-approximation}, if for any feasible solution~$x'\in X$, there exists a solution~$x\in A_{\varepsilon}$ such that:
\begin{align*}
f_j(x) \leqslant (1+\varepsilon) f_j(x') \text{ for all } j\in\{1,\dots,p\}
\end{align*}
where $\varepsilon\geqslant 0$.
An approximation ratio of an arbitrary point set, $A\subset\osp$,
to a given reference point set, $R\subset\osp$ (\eg, the Pareto front, $R=Y_N$),
can be calculated from the (multiplicative) $\epsilon$-indicator~\citep{ZitzlerTLFF03}:
$$
I_\epsilon(A,R) = \max_{r\in R} \min_{a\in A} \max_{k\in\{1,...,p\}} \left\{\frac{a_k}{r_k}\right\} 
$$
If $A_{\varepsilon}$ is a $(1+\varepsilon)$-approximation then we have $I_\epsilon(f(A_{\varepsilon}),Y_N)\leqslant 1+\varepsilon$.
A \emph{lower bound set}, $L\subset\os$, \ie, a set of mutually nondominated points that weakly
dominates $Y_N$, can be easier to compute than $Y_N$ and provides an upper bound on the
approximation ratio of any point set $A\subseteq Y$, \ie,
$I_\epsilon(A,L) \geqslant I_\epsilon(A,Y_N)$~\citep{ZitzlerKT08}.
With a slight abuse of notation, when referring to a solution $x\in X$, for which we explicitly assign 
values to the variables, we will refer to it as an assignment $\nu$ and define it using set notation
(\eg, $\nu=\{x_1=1, x_2=0\}$). We will also extend the notation to state that $\nu\in X$,
and use $f(\nu)$.
Next, we present an example illustrating some of the introduced definitions. 
\begin{example}
Let $f(x) = (2 x_1 + x_2 +  \bar{x}_3,\,\bar{x}_1 +x_2  + 2 x_3)$ denote the vector with two objective functions to minimize such that the set of constraints $\{ x_1 + x_2 + x_3 \geqslant 2 \}$ needs to be satisfied. The four feasible solutions correspond to the following assignments $\nu^i$ with objective vectors $f(\nu^i)$:
\begin{itemize}
    \item $\nu^1 = \{ x_1 = 1, x_2 = 1, x_3 = 0 \}$  with $f(\nu^1) = (4,1)$
    \item $\nu^2 = \{ x_1 = 1, x_2 = 0, x_3 = 1 \}$  with $f(\nu^2) = (2,2)$
    \item $\nu^3 = \{ x_1 = 0, x_2 = 1, x_3 = 1 \}$  with $f(\nu^3) = (1,4)$
    \item $\nu^4 = \{ x_1 = 1, x_2 = 1, x_3 = 1 \}$  with $f(\nu^4) = (3,3)$
\end{itemize}
In this case, assignments $\nu^1$, $\nu^2$ and $\nu^3$ correspond to the three efficient solutions of this problem. Hence, $Y_N=\{(1,4), (2,2),(4,1)\}$. Moreover, $A_\varepsilon=\{\nu^2\}$  is a $(1+\varepsilon)$-approximation for any $\varepsilon \geqslant 1$. Observe that this set has smaller cardinality than $Y_N$ and that  $I_\epsilon(f(A_\varepsilon),Y_N)=2$. Remark also that $\nu^4$, which is not efficient, also forms a $(1+\varepsilon)$-approximation of small cardinality, but for any $\varepsilon \geqslant 2$.
\label{ex:mobo-example}
\end{example}

\subsection{Boolean optimization problems as maximum satisfiability problems}
\label{s:bg:bo2ms}

\noindent Boolean optimization problems can be encoded as logic-based problems. 
Section~\ref{s:bg:maxsat} introduces the Maximum Satisfiability (MaxSAT) problem, defines Minimal Correction Subsets (MCSs), and explains how an optimal solution to a MaxSAT problem can be found by
enumerating all MCSs.
Section~\ref{sec:mcs-ur} describes an encoding of a (single-objective) Boolean optimization 
problem as a MaxSAT problem, which will be later extended to the multi-objective case in Section~\ref{s:contr}.

\subsubsection{Maximum satisfiability and minimal correction subsets}
\label{s:bg:maxsat}
\noindent In propositional logic, a \emph{formula} is %a set of clauses
defined over a set of Boolean variables $x_1,\ldots,x_n \in \{0,1\}$. 
In formulas expressed in \emph{Conjunctive Normal Form (CNF)}, a \emph{clause} represents a disjunction of literals, where each \emph{literal} is a variable (\eg, $x_1$) or its negation  (\eg, $\bar{x}_1$), and the formula represents a conjunction of clauses.
For example, $\phi=\{(x_1 \vee \bar{x}_2),(\bar{x}_1\vee x_2 \vee x_3)\}$ is a CNF formula with two clauses. 
A literal $x_j$ is satisfied by an assignment $\nu$ if $\nu$ assigns 1 to $x_j$
(\ie, $x_j=1$), whereas a negated literal $\bar{x}_j$ is satisfied by $\nu$
if $\nu$ assigns 0 to $x_j$ (\ie, $x_j=0$). %, where $i\in\{1,\ldots,m\}$.
An assignment is said to satisfy a clause if it satisfies at least one literal 
in the clause, and it is said to satisfy a formula if it satisfies all clauses 
in the formula. 
A complete assignment that satisfies a given formula is also called a \emph{feasible} 
assignment or a \emph{model}.
A formula $\phi$ is said to be \emph{satisfiable} if there exists an
assignment that satisfies $\phi$. Otherwise, $\phi$ is said to be \emph{unsatisfiable}.

\begin{example}
Consider the following CNF formulas $\phi=\{(x_1\vee \bar{x}_2),(\bar{x}_1\vee \bar{x}_2)\}$,
and  $\phi'=\{(x_1\vee \bar{x}_2),(\bar{x}_1\vee \bar{x}_2),(x_2)\}$. Formula $\phi$ is
satisfiable and $\nu=\{x_1=1, x_2=0\}$, denoted more compactly as $\nu = (1,0)$, is a model of $\phi$, whereas formula $\phi'$ is
unsatisfiable.
\end{example}

Given a CNF formula $\phi$, the Propositional Satisfiability (SAT) problem
consists of finding a satisfiable assignment to $\phi$ or prove that
such assignment does not exist.
The Maximum Satisfiability (MaxSAT) problem is an optimization version of the 
SAT problem.
In general, a MaxSAT formula $\phi$ consists of two sets of clauses,
the hard ($\hc$) and the soft ($\sfc$) clauses such that $\phi=\hc \cup \sfc$.
This paper assumes unweighted MaxSAT formulas, but weighted MaxSAT is also 
treated in the literature~\citep{manya-handbook09}.
Assuming that a feasible assignment for $\hc$ exists, the goal is to find 
an assignment that satisfies all hard clauses and as many soft clauses 
as possible, \ie, minimizes the number of unsatisfied soft clauses.
Observe that usually there is no assignment that simultaneously satisfies all hard 
clauses in $\hc$ and all soft clauses in $\sfc$.

A \emph{correction subset} of a MaxSAT formula $\phi$ is a
subset of soft clauses $C\subseteq\sfc$ such that $\hc\cup(\sfc\setminus C)$ is satisfiable.
Additionally, if $\hc \cup (\sfc\setminus C) \cup \{c\}$ is unsatisfiable for all $c\in C$,
then $C$ is a \emph{Minimal Correction Subset (MCS)}.
That is, $C$ is minimal in the sense that none of the proper subsets of $C$ is a correction subset.

A complete assignment $\nu$ is said to \emph{correspond} to the MCS $C$ if $\nu$ 
satisfies $\hc\cup(\sfc\setminus C)$. Note that since $C$ is an MCS, then $\nu$ does 
not satisfy any clause in $C$.
Given a MaxSAT formula $\phi$, an optimal solution to $\phi$ corresponds to 
the MCS $C$ with minimum size. Hence, if we were to enumerate all MCSs 
of $\phi$~\citep{DBLP:conf/sat/PrevitiMJM17,DBLP:conf/ijcai/GregoireIL18},
then we can simply select an MCS of minimum size as an optimal solution for~$\phi$.
\begin{example}
Let $\hc=\{(\bar{x_1} \vee \bar{x_2} \vee \bar{x_3}), (x_1 \vee x_2), 
(\bar{x}_1 \vee x_2 \vee x_3)\}$ and $\sfc=\{(\bar{x}_1),(\bar{x}_2),(\bar{x}_3)\}$ 
define the set of hard and soft clauses of a MaxSAT formula $\phi$.
There are two MCSs, $C^1=\{(\bar{x}_1),(\bar{x}_3)\}$, and $C^2=\{(\bar{x}_2)\}$. 
%The assignments $\nu^1=\{x_1=1, x_2=0,x_3=1\}$, and $\nu^2=\{x_1=0,x_2=1,x_3=0\}$ 
The assignments $\nu^1=(1,0,1)$, and $\nu^2=(0,1,0)$ 
correspond to MCSs $C^1$ and $C^2$, respectively. 
Therefore, assignment $\nu^2$, corresponding to $C^2$, is an optimal solution for the MaxSAT instance $\phi$.
\label{ex:maxsat}
\end{example}

Note that linear single-objective Boolean optimization problems can be easily
formulated as MaxSAT. For instance, the MaxSAT formula in Example~\ref{ex:maxsat}
corresponds to the problem of satisfying $\hc$ while minimizing $x_1 + x_2 + x_3$.
Observe that the cost function is encoded into the soft clauses $\sfc$.
Arbitrary linear objective function are formulated using a more
general version of MaxSAT, where each soft clause has an associated weight representing the
cost of not satisfying it. For example,
$2x_1+3x_2$ is encoded into the soft clauses $(\bar{x_1})$ and $(\bar{x_2})$ with weights 2 and 3,
respectively.
In fact, any 0-1 Integer Linear Programming (0-1 ILP) can be encoded into
MaxSAT by encoding the linear constraints into hard clauses and the objective
function as soft clauses. The encoding of linear constraints into CNF formulas
has been extensively studied~\citep{bailleux-cp03,seq,asin-constraints11,totalizer-ictai13,swc,watchdog,bailleux-jsat06,een-jsat06,abio-jair12}.

There exist different ways of encoding the objective function(s) into soft clauses, which has implications on how MCSs and (Pareto-)optimal solutions relate.
In the above example, the encoding relies on directly using the decision variables present in the
objective function as soft clauses (and their coefficients as the corresponding weights). This
idea has already been extended to the multi-objective case, for which it was shown that
each Pareto-optimal solution of a MOBO problem corresponds to an MCS~\citep{DBLP:conf/sat/Terra-NevesLM17}. Hence, all Pareto-optimal solutions can be found by enumerating all MCSs of the
resulting MaxSAT formula.
However, this encoding has the disadvantage that not all MCSs correspond to Pareto-optimal solutions, which means that some dominated solutions may be enumerated as well.
One way to overcome this issue is to use a different encoding for which there is a one-to-one correspondence between nondominated points and MCSs. The encoding of objective functions using a unary representation have this property, as seen in the next sections.

\subsubsection{Encoding the objective function(s) using a unary representation}
\label{sec:mcs-ur}

\noindent Objective functions can be encoded in the MaxSAT formulation by adding,
as soft clauses, new variables that represent (bounds on) the value of the objective function
\citep[see, \eg,][]{DBLP:conf/cp/SohBTB17}.
Assume that the objective function $f: \Bmp \to \oss$
is lower and upper bounded by $\ell$ and $u$,
respectively, \ie, $\ell \leqslant f(\nu) \leqslant u$ for all $\nu\in\Bmp$.
A way of encoding $f(x)$ to CNF with this approach is to include a new set 
of Boolean variables, $y_{d}$, and encode the equivalences $y_{d}=1 \iff f(x) < d$
for $d \in D \subseteq \{\ell,\ell,\ldots,u^+\}$ where, throughout the paper, $u^+$ is assumed to be
a value greater than $u$ (\ie, $u^+ > u$). As discussed at the end of this section, such equivalences can be encoded with an additional set of hard clauses. Variables
$y_{\ell}$ and $y_{u^+}$ are not strictly required in the encoding, as $y_{\ell}$ is
trivially false and $y_{u^+}$ is trivially true. To make the analysis smoother, we will assume
however that $\ell,u^+ \in D$, which is a weak assumption since $\ell$ and $u$ are
supposed to be known lower and upper bounds of the objective function value. The precise definition of
$D$ depends whether we aim at solving our discrete optimization problem exactly or approximately. 
In the first case $D$ must be (a superset of) all values that can be reached by the objective function, \eg, be the set of all integers between $\ell$ and $u^+$ (in this case we assume $u^+=u+1$). If so, $D$ will be called \emph{complete}. In the second case, depending on the required degree of approximation,
$D$ should be defined as an appropriate subset of all integers between $\ell$ and $u^+$.

With this unary representation of the objective function, a MaxSAT formulation of any single 
objective Boolean optimization problem is obtained by defining $\{(y_{d}) \:|\: d \in D\}$ as the
set of soft clauses, with $D \subseteq \{\ell,\ell+1,\ldots,u^+\}$ and $\ell,u^+\in D$, while 
constraints and the previous equivalences are represented by hard clauses. Then, solving the single 
objective Boolean optimization problem amounts to finding the (unique) MCS of the MaxSAT formula. 
Assuming indeed that the optimal objective function value is $d^*$, the corresponding MCS will
be of the form $C=\{(y_{d}) \:|\: d \in D \mbox{ and } d < d^*\}$. Observe first that $C$ always
exists since it contains at least $y_{\ell}$. 
Let $d'$ be the largest index among variables $y_d \in C$ and $d''$ be the smallest index
among variables in $D \setminus C$. Then, any feasible assignment $\nu$ corresponding to $C$ is
such that $d' \leqslant f(\nu) < d''$. Hence, the solutions corresponding to $C$ are within
an approximation guarantee provided by $d'$ and $d''$.
In the particular case where $D$ is complete, any solution $\nu$ corresponding to $C$ is optimal
with objective function value $f(\nu)=d'$.
Consider the following example where $D$ is complete.

\begin{example}
Let $f(x) = 3 x_1 + 2 x_2 + 2 x_3$ be the function to minimize such that 
$\{x_1 + x_2 \geqslant 1, \bar{x}_2 + x_3 \geqslant 1\}$
must be satisfied. 
Observing from the definition of $f$ that it can only take values in $D' = \{ 0, 2, 3, 4, 5, 7\}$ we can define a complete domain $D = D' \cup \{8\}$ and variables $y_d$, for $d\in D$.
Let CNF($f(x)$) denote the CNF formula encoding that $y_{d}=1 \iff f(x) < d$. 
Hence, one can define a MaxSAT formula $\phi = \hc \cup \sfc$ such that 
$\hc = \{(x_1 \vee x_2), (\bar{x}_2 \vee x_3) \} \cup \text{CNF}(f(x))$ and 
$\sfc = \{ (y_0), (y_2), (y_3), (y_4),  (y_5), (y_7), (y_8) \}$.
Note that the optimal solution of $\phi$ corresponds to the optimal solution of the original problem.
Observe that $\phi$ only has a single MCS corresponding to the optimal solution.
In this case, the MCS is $\{(y_0), (y_2), (y_3)\}$ since the optimal solution has a cost of 3.
\label{ex:singleMCS}
\end{example}

\noindent It remains to explain how to actually encode (into hard clauses) the equivalences
$y_{d}=1 \iff f(x) < d$, for all $d\in D$. Since, in this paper, $f(x)$ is a linear expression
defined over a set of Boolean variables, it is possible to develop solvers that natively handle
these expressions.
However, many Boolean solvers perform an encoding of linear Pseudo-Boolean expressions into CNF
in order to take advantage of effective  state of the art SAT solvers.
% ----
It is known that such expressions can be represented into
CNF using a polynomial size encoding~\citep{watchdog}.
However, despite its worse case exponential size, it has been observed that 
some unary representations of the value of linear expressions are better handled 
by SAT solvers~\citep{DBLP:journals/jsat/BacchusJM19}. 
% ----
Hence, in this paper, the actual encoding of the aforementioned equivalences is done only after encoding the value of the objective function $f(x)$ into CNF. First, we use an encoding using selection networks~\citep{DBLP:journals/constraints/KarpinskiP19} that has been shown to be more compact. Next, the selection networks encoding is extended such that a unary encoding is produced where $y_{d}=1 \iff f(x) < d$. 
Moreover, besides adding the clauses to encode the equivalence $y_{d}=1 \iff f(x) < d$ for all $d\in D$, an additional order encoding~\citep{tamura2008sugar} is also added such that $y_d = 1 \implies y_{d+1} = 1$ and $y_d = 0 \implies y_{d-1} = 0$.
We refer the interested reader to the literature for further details on this and other encodings~\citep{DBLP:series/faia/RousselM09,DBLP:conf/cp/0001MM15,DBLP:journals/constraints/KarpinskiP19,DBLP:conf/sat/KarpinskiP20}.
Finally, recall that in MOBO there are several objective functions.
Therefore, each objective function is represented using the described encoding
into CNF.

\section{Pareto MCSs} %needs a better title
\label{s:contr}

\noindent In this section we investigate the relation between MCSs and the nondominated points of a MOBO problem (see Section~\ref{s:contr:pmcs}). In the special case where the domains $D_k$ of the objective functions are complete, we prove that MCSs and the nondominated points are in one to one correspondence. A similar result was obtained in~\cite{DBLP:conf/cp/SohBTB17}, in a different setting, using so-called $P$-minimal models, instead of  MCSs. In the general case, the domains are not  complete, meaning that some (or possibly all) values taken by objective function $f_k$ are not present in $D_k$ for $k \in \{1,\ldots,p\}$. Then, we show that MCSs provide some guarantee on the quality of the associated solutions. This lays the foundation for the generation of $(1+\varepsilon)$-approximation sets investigated in Section~\ref{s:approx}.

In the following, let $\ell_k\in\oss$ and $u_k\in\oss$ denote the lower and upper bound values of the $k$-th
objective function of a MOBO problem, \ie, $\ell_k \leqslant f_k(x) \leqslant u_k$, for all $k\in\{1,\ldots,p\}$, and
let $D_k\subseteq\{\ell_k,\ell_k+1,\ldots, u_k, u^+_k\}$ and $\ell_k, u^+_k\in D_k$ where
$u^+_k > u_k$. Moreover, let the variables $y_{k,d}$,
where $k\in\{1,\ldots,p\}$ and $d \in D_k$, encode the values of the $k$-th objective function (see Section~\ref{sec:mcs-ur}),
\ie, $y_{k,d}=1 \iff f_k(x) < d$.
% Such an encoding may be constructed as described in Section~\ref{sec:mcs-ur}.
% 
Problem~\ref{prob:MOBP} is encoded as a MaxSAT problem, where $\hc$ is a set of hard clauses
encoding (in CNF) the problem constraints,
and encoding each of the objective functions and their corresponding values in $D_1,\ldots,D_p$.
Hence, $\hc$ enforces $x\in X$ and $y_{k,d}=1 \iff f_k(x) < d$, for all $k\in\{1,\ldots,p\}$ and $d \in D_k$. 
Finally, the set of soft clauses, $\sfc$, consists of the set of all literals $y_{k,d}$, \ie, $\sfc=\{(y_{k,d})\:|\: k=1,\ldots,p \mbox{ and } d\in D_k\}$.

We first give two preliminary simple results related to MCSs.
\begin{lemma}\label{lem:consistency}
Let $C\subset \phi^s$ be an MCS for our MaxSAT formulation of a MOBO problem.
\begin{itemize}
    \item[$(i)$] If $y_{k,d} \in C$, then  $y_{k,e} \in C$ for $e \leqslant d$.
    \item[$(ii)$] If $y_{k,d} \notin C$, then $y_{k,e} \notin C$ for $e \geqslant d$.
\end{itemize}
\end{lemma}
\begin{proof}~\\
\vspace{-0.5cm}
\begin{itemize}
    \item[$(i)$] Since $y_{k,d} \in C$, the solution $x$ associated with the current model is such that $f_k(x) \geqslant d$. Then we necessarily have $f_k(x) \geqslant e$ for $e \leqslant d$.
    \item[$(ii)$]  Since $y_{k,d} \notin C$, the solution $x$ associated with the current model is such that $f_k(x) < d$. Then we necessarily have $f_k(x) < e$ for $e \geqslant d$.
\end{itemize}%
\vspace{-0.5cm}
\end{proof}
\noindent Let $D=D_1\times\ldots\times D_p$. To any MCS $C$ we associate the two following points:
\begin{itemize}
    \item its \emph{representative point} $r \in D$ whose components are $r_k =\max\{d \in D_k\:|\:y_{k,d} \in C\}$ 
    \item its \emph{successor point} $r'\in D$ whose components are $r'_k =\min\{d \in D_k\:|\: y_{k,d} \not\in C\}$, for all $k\in\{1,\ldots,p\}$.
\end{itemize}
The following result shows that $r$ and $r'$ correspond to lower and upper bound points of the image of the solutions associated to $C$.
\begin{lemma}\label{lem:bounds}
Let $r$ and $r'$ be the representative and successor point of an MCS $C$. Then any assignment $\nu\in\Bmp$ associated  to $C$ is such that $r \wdom f(\nu) \sdom r'$. Moreover, if domains $D_k$ are complete, for all $k\in\{1,\ldots,p\}$, then $f(\nu) = r$.
\end{lemma}
\begin{proof}
By definition, $y_{k,r_k} \in C$ implying $f_k(\nu) \geqslant r_k$, and  $y_{k,r'_k} \notin C$ implying $f_k(\nu) < r'_k$, for all $k \in \{1,\ldots,p\}$. When domains $D_k$ are complete, there is no feasible solution $x\in X$, such that $r_k < f_k(x) <r'_k$, implying the equality.
\end{proof}

\subsection{Finding the nondominated set}
\label{s:contr:pmcs}
\noindent Assuming that $D_k$ is complete and $\ell_k,u^+_k\in D_k$, for all $k\in\{1,\ldots,p\}$, we show that each MCS corresponds to a nondominated point. It follows that all nondominated points can be obtained by enumerating all MCSs.
We assume that Problem~\ref{prob:MOBP} %(see Section~\ref{sec:mobo}) 
is feasible. Therefore, all hard constraints can be satisfied. MCSs are computed only with respect to variables $y_{k,d}$ for all $k \in \{1,\ldots,p\}$ and $d\in D_k$.

\begin{theorem}
Consider a MOBO problem, with feasible point set $Y$, stated using our MaxSAT formulation with complete domains $D_k$ for all $k\in\{1,\ldots,p\}$. Then nondominated points in $Y_N$ and MCSs of the MaxSAT formulation are in one to one correspondence.
\label{t:mcs-nondom}
\end{theorem}

\begin{proof}
Domains $D_k$ being complete, any point in $Y$ is of the form $(d_1,\ldots,d_k,\ldots, d_p)$, where $d_k \in D_k$, for all $k\in\{1,\ldots,p\}$. We prove that a point $d = (d_1,\ldots,d_k,\ldots, d_p)$ belongs to $Y_N$ if and only if $$C=\{y_{1,\ell_1},\ldots,y_{1,d_1},\ldots, y_{k,\ell_k},\ldots,y_{k,d_k},\ldots, y_{p,\ell_p},\ldots,y_{p,d_p}\}$$
is an MCS of the MaxSAT formulation. 

\medskip
Let $r = d$ and $r'$ be the representative and successor points of $C$, respectively.

\noindent $\Leftarrow:$ 
By Lemma~\ref{lem:bounds}, point $d$, which is the representative point of $C$, belongs to $Y$. Moreover, $d \in Y_N$, since the existence of a feasible solution $x'$ such that $f(x')\dom d$, would require $f_k(x') \leqslant d_k$ for all $k \in \{1,\ldots,p\}$ and $f_q(x') < d_q$ for some $q \in \{1,\ldots,p\}$, meaning that we could remove variable $y_{q,d_q}$ from $C$, contradicting that $C$ is an MCS.

\medskip
\noindent $\Rightarrow:$ 
Since $d$ is feasible, $C$ is a correction subset. Since $d$ is nondominated, we cannot remove from $C$ any variable $y_{q,d_q}$ for some $q \in \{1,\ldots,p\}$. By Lemma~\ref{lem:consistency}$(i)$, we cannot
remove any variable from $C$, which shows that $C$ is an MCS.
\end{proof}

As seen in Theorem~\ref{t:mcs-nondom}, the representative points of MCSs always correspond to nondominated points when $D_k$ is complete for all $k\in\{1,\ldots,p\}$.
A similar reasoning can be used to show that, in the more general case of $D_k$ defined as an
arbitrary subset of $\{\ell_k,\ell_{k+1}\ldots,u_k,u^+_k\}$ where $\ell_k,u^+_k\in D_k$,
the representative points of the MCSs are still mutually nondominated, though are not necessarily
feasible, and the correspondence between nondominated points and MCSs becomes one-to-many.
That is, an MCS is associated to each nondominated point, but more than one
nondominated points may be associated to an MCS (and some dominated points may be as well).
The next section explores this more general definition of $D_k$.

\subsection{Finding a $(1+\varepsilon)$-approximation set}
\label{s:approx}

\noindent Having variables $y_{k,d}$ represent, in an exact way, all possible values of the corresponding
objective function, allows the enumeration of all nondominated points through MCS enumeration.
Although it may be desirable to find the Pareto front, this may require large encodings, \ie,
a large set of hard and soft clauses, and the resulting formula may still be difficult to solve.
Alternatively, finding a $(1+\varepsilon)$-approximation set may be enough, and even helpful if used
as an intermediate step in the search for the Pareto front.
In this section we describe two encodings related to the one used
in Section~\ref{s:contr:pmcs}, but which are smaller and (potentially) easier to solve,
and for which enumerating all MCSs is equivalent to finding a $(1+\varepsilon)$-approximation set.
One of the encodings consists in considering a smaller set of soft clauses
(see Section~\ref{s:approx:int}), and in the other the variables $y_{k,d}$ encode an approximate
version of the objective function values (see Section~\ref{s:approx:coeff}).

\subsubsection{Interval-based approximation}
\label{s:approx:int}

\noindent We assume here that each objective function $f_k$ takes \emph{positive} integer values.
An approximate version of the method, with {\it a priori} guarantee, can be obtained by modifying
the domains $D_k$ of each objective function $f_k$. 
Assume for instance that we want to generate a $(1+\varepsilon)$-approximation $A_{\varepsilon}\subseteq X$,
with $\varepsilon \geqslant 0$.
In this case, we define $D_k$, for each $k\in\{1,\ldots,p\}$, as follows:
$$D_k = \{d_{k,1},\ldots,d_{k,u}\}$$
with $d_{k,1} = \ell_k$, $d_{k,i} = \max\{d_{k,i-1} +1,\lfloor (1+\varepsilon)\cdot d_{k,i-1}\rfloor\}$
for $i \in \{2,\ldots,u\}$ where $u$ is the smallest integer such that $d_{k,u} > u_k$. 
Each value $d_{k,i} \in D_k$, for $i\in \{1,\ldots,u-1\}$, defines the beginning of a new interval 
which is upper bounded by the successor value in $D_k$, \ie, $d_{k,i+1}$.
This definition of $D_k$ admits $\varepsilon=0$, in which case $D_k=\{\ell_k,\ldots,u_k+1\}$,  
which corresponds to the exact version.

Consider the following bi-objective Boolean optimization problem, which is illustrated
in Figure~\ref{fig:approx:int:eps0}:
\begin{equation}\label{prob:ex2}
% \left\{
\begin{array}{rl}
    \min f_1(x) &= 3\bar{x}_1 + \bar{x}_2 + 2\bar{x}_3 + x_4 + x_5 + 2\\
    \min f_2(x) &= x_1 + 2x_2 + 16x_3 + x_5 + 2\\
    
    \mbox{\rm{subject to}:}\\
    &x_1 + x_2 + x_3 + x_4 + x_5 \geqslant 3\\
    &x_1 + x_2 + x_3 + x_4 + x_5 \leqslant 4\\ 
    &x \in \Bmp
\end{array}
% \right.
\end{equation}
\begin{figure}[t!]
        \centering
        \begin{subfigure}[b]{0.27\linewidth}
           \centering
            \includegraphics[width=\linewidth]{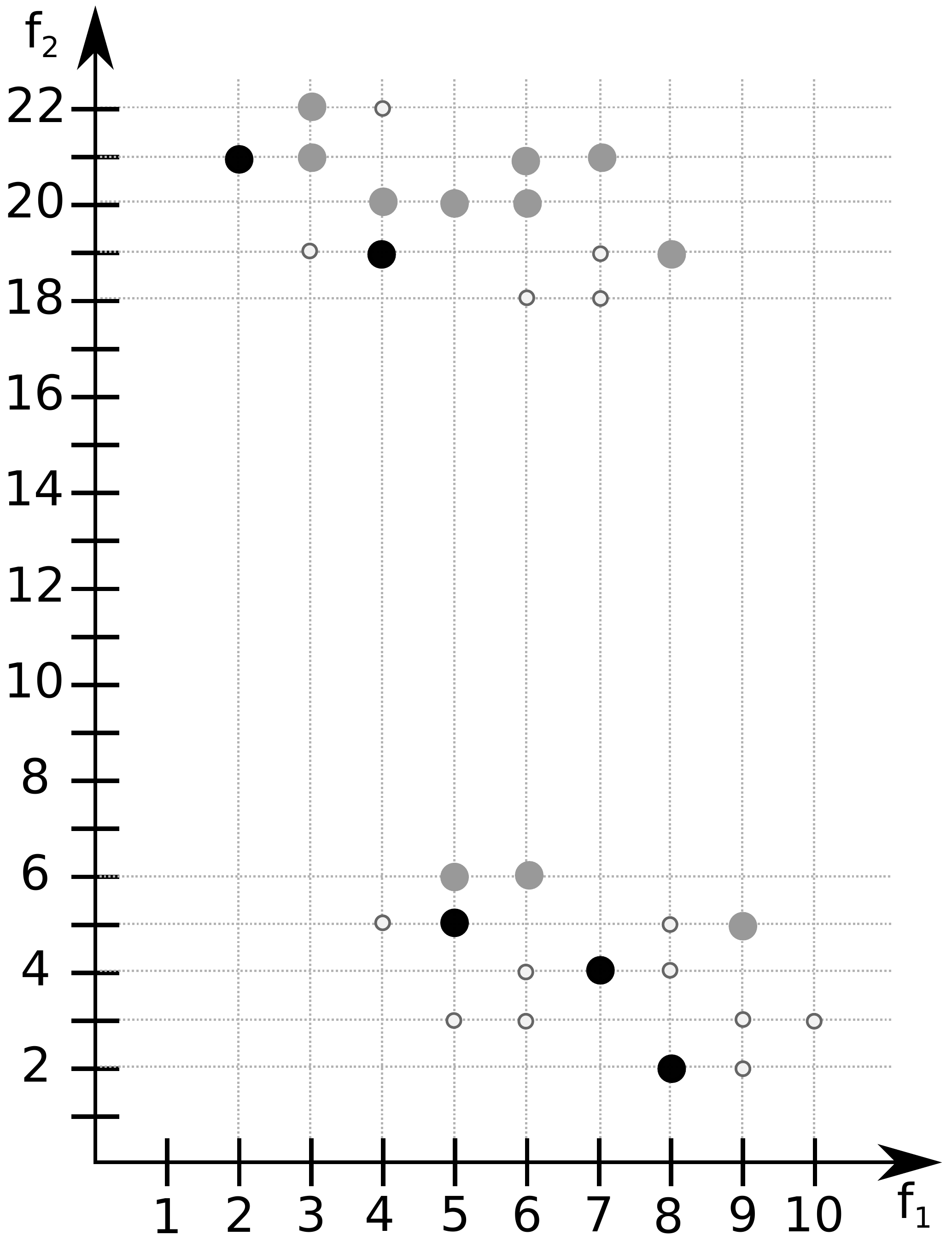}
            \caption[]{$\varepsilon=0$, $D_1=\{2,\ldots,10\}$, $D_2=\{2,\ldots,6,18,\ldots,22\}$}
            \label{fig:approx:int:eps0}    
        \end{subfigure}
     \hspace{1cm}
        \begin{subfigure}[b]{0.27\linewidth}
            \centering
            \includegraphics[width=\linewidth]{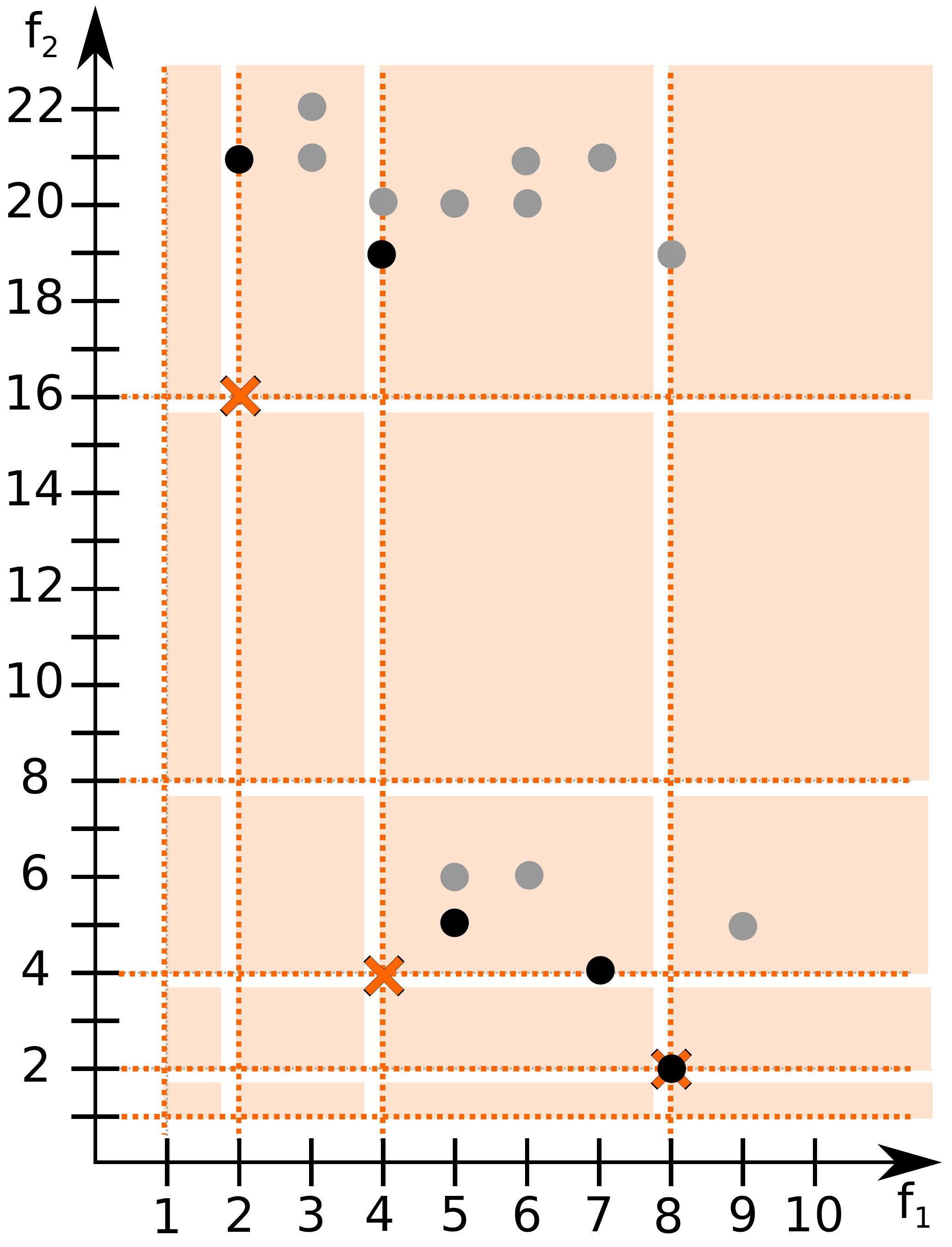}
            \caption[]{$\varepsilon=1$, $D_1=\{1,2,4,8,16\}$, $D_2=\{1,2,4,8,16,32\}$}
            \label{fig:approx:int:eps1}    
        \end{subfigure}
        \caption{An example of a bi-objective problem with different settings of
        $D_1$ and $D_2$, depending on $\varepsilon$. 
        Large and small circles represent the images of feasible and infeasible solutions, respectively.
        The vertical and horizontal lines represent the values in $D_1$ and
        $D_2$, respectively, and the shadowed regions illustrate
        the intervals associated to $D_1$ intersected with those associated to $D_2$.
        The crosses depict the representative points of the MCSs.
        }
        \label{fig:approx:int}
    \end{figure}
The figure shows the image $f(\nu)$ of all assignments $\nu\in\Bmp$, including infeasible ones
(small circles), \ie, that do not satisfy the two cardinality constraints. The nondominated points are coloured black. 
In this case, $f_1$ takes values between $\ell_1=2$ and $u_1=10$,
and $f_2$ takes values between $\ell_2=2$ and $u_2=22$.

In the examples for the particular case of $\varepsilon=0$, we will consider
$D_k=\{f_k(x)\:|\:x \in \Bmp\}\cup\{u_k+1\}$ which is complete.
In Figure~\ref{fig:approx:int:eps0}, these values of $D_k$ are highlighted with dashed lines.
For $\varepsilon=0.5$, we have $D_1=\{2,3,4,5,7,9,13\}$ and $D_2=\{2,3,4,5,7,9,12,16,24\}$. 
For $\varepsilon=1$, $D_1=\{2,4,8,16\}$ and $D_2=\{2,4,8,16,32\}$, 
as illustrated by the dashed lines in Figure~\ref{fig:approx:int:eps1}.
The values $16\in D_1$ and $32\in D_2$ were omitted for convenience.
Note that, in Figure~\ref{fig:approx:int:eps1}, there are empty intervals that
could be excluded from $D_1$ and $D_2$ while still ensuring the approximation ratio
of $2$. In the example, setting $D_1$ and $D_2$ to $\{2,4,8\}$ and $\{2,4,19\}$, respectively,
where $D_k\subseteq\{f_k(x) \:|\: x \in \Bmp\}$ for $k=1,2$, would ensure such approximation ratio.
However, as the image set of $f_k(x)$ may be much larger than $D_k$, computing $D_k$ in this way may not be
advantageous in general, and be advantageous only to particular instances with particular values
of $\varepsilon$.

% MCS - nondominated point: 1 - N
For $\varepsilon=0$, each MCS $C$ corresponds to a single nondominated point,
which is its representative point.
This is not necessarily the case for $\varepsilon > 0$.
Firstly, although each nondominated point corresponds to a single MCS, to each MCS may correspond
multiple feasible points, even dominated ones, all of which are weakly dominated by the representative 
point of the MCS and are within a $(1+\varepsilon)$ ratio from it.
Secondly, the representative point of an MCS does not correspond, in general, to a feasible point.
% 
% subregion
Let $D=D_1\times\ldots\times D_p$, let
$r,r'\in D$ be such that $r'$ is the successor of $r$, and let the feasible point
% $s\in Y$ be such that $r_k \leqslant s_k < r'_k$ for all $k=1,\ldots,p$.
$s\in Y$ be such that $r \wdom s \sdom r'$.
In Figure~\ref{fig:approx:int:eps1}, $r$ and $r'$ are
the lower and upper corners of a rectangular orange region, and $s$ is any point in that region.
Any such point $s$ is associated to a model where $y_{1,r_1}=0,\ldots,y_{p,r_p}=0$ and
$y_{1,r'_1}=1,\ldots,y_{p,r'_p}=1$. Therefore, $s$ is weakly dominated by $r$ and is
within a $(1+\varepsilon)$ ratio from $r$. If $r$ is a representative point of
an MCS $C$, then such feasible points $s$ are the only points corresponding to $C$.
In Figure~\ref{fig:approx:int:eps1}, the points $(2,21)$, $(3,21)$, and $(3,22)$ 
correspond to the MCS represented by point $(2,16)$ (orange cross), as their corresponding
models include $y_{1,2}=0$, $y_{2,16}=0$, $y_{1,4}=1$, and $y_{2,32}=1$.

%Lower bound set
Since the set of soft clauses, $\sfc$, only includes variables $y_{k,d}$ for which $d \in D_k$,
enumerating all MCSs and finding a solution associated to each one of them is equivalent to
finding a $(1+\varepsilon)$-approximation set.
Additionally, since each nondominated point $s\in Y_N$ is associated to an MCS whose representative point $r\in D$ weakly dominates $s$, the set of the representative points of
all MCSs is a lower bound set.
In the example of Figure~\ref{fig:approx:int:eps1}, the lower bound set is
$\{(2,16), (4,4), (8,2)\}$ which is represented by the set of crosses.
Hence, enumerating all MCSs not only provide an approximation set $A_\varepsilon\subseteq X$,
but also a lower bound set $\Lc\subset\osn$, 
for which $I_\varepsilon(f(A_\varepsilon),Y_N) \leqslant I_\varepsilon(f(A_\varepsilon),\Lc) \leqslant (1+\varepsilon)$.

\subsubsection{Coefficient-based approximation}
\label{s:approx:coeff}

\noindent An alternative approximate version, also with \emph{a priori} guarantee, can be obtained by modifying
the coefficients $w_1,\ldots,w_n$ of the objective function $f_k(x)=\sum_{j=1}^{n} w_j x_j$, 
where $k\in\{1,\ldots,p\}$, and $w\in\Nmp$. % and $X\subseteq\Bmp$. 
Consider $W_k$ defined as $W_k = \{\hat{w}_1,\ldots,\hat{w}_u\}$
with $\hat{w}_1 = \min\{w_1,\ldots,w_n\}$, and $\hat{w}_i = \max\{\hat{w}_{i-1} +1,\lfloor (1+\varepsilon)\cdot \hat{w}_{i-1}\rfloor\}$ for $i \in \{2,\ldots,u\}$ where $u$ is the smallest integer such that $\hat{w}_u \geqslant \max\{w_1,\ldots,w_n\}$.
Let $w'_1,\ldots,w'_n\in W_k$ denote the modified coefficients where $w'_j=\max\{\hat{w}_i\in W_k \:|\: \hat{w}_i \leqslant w_j\}$. 

\begin{lemma}\label{l:ap:coeffs}
Given the function $f_k(x)=\sum_{j=1}^{n} w_j x_j$ where $w\in\Nmp$ and $x\in \Bmp$,
let the function $f'_k:~\Bmp\to\oss$ be such that $f'_k(x)=\sum_{j=1}^{n}w'_jx_j$
where $w'_j=\max\{\hat{w}_i\in W_k \:|\: \hat{w}_i \leqslant w_j\}$. Then, for any $\varepsilon \geqslant 0$, we have:
\begin{align*}
f_k(x) \leqslant (1+\varepsilon)f'_k(x) \: \text{ for all } \: x\in\Bmp
\end{align*}
\end{lemma}

\begin{proof}
The lemma is trivially true for $\varepsilon=0$, as $f(x)=f'(x)$. For $\varepsilon>0$, a slightly stronger version of the lemma, with a strict inequality, will be proved. Assume by contradiction that there exists $x\in\Bmp$ such that $f_k(x) \geqslant (1+\varepsilon)f'_k(x)$, \ie,
$\sum_{j=1}^{n} w_j x_j \geqslant \sum_{j=1}^{n} (1+\varepsilon) w'_j x_j$.
Then, there must exist an index $j\in\{1,\ldots,n\}$ such that
\begin{equation}
(1+\varepsilon) w'_j \leqslant w_j
\label{eq:modified_weights}
\end{equation}
Modified weight $w'_j$ corresponds to a weight $\hat{w}_{i}$ where either $i=u$ or $i<u$. If $i=u$ then $w'_j = w_j (= \hat{w}_u)$ in which case \eqref{eq:modified_weights} is not possible, leading to a contradiction. If $i<u$ then we have $\hat{w}_{i} = w'_j \leqslant w_j < \hat{w}_{i+1}$. Then, by definition, $\hat{w}_{i+1}$ is either $\lfloor(1+\varepsilon) \hat{w}_i\rfloor$ or $\hat{w}_i+1$. In the first case we get $w_j < \lfloor(1+\varepsilon) \hat{w}_i\rfloor = \lfloor(1+\varepsilon) w'_j \rfloor \leqslant (1+\varepsilon) w'_j$ contradicting~(\ref{eq:modified_weights}). The second case implies that $w_j=\hat{w}_i=w'_j$ because $w_j$ must be an integer such that $\hat{w}_i\leqslant w_j <  \hat{w}_{i+1} =\hat{w}_i+1$ also leading to a contradiction, which completes the proof.
\end{proof}

Consider, for example, the objective function $f_1(x)=2x_1+3x_2+5x_3+7x_4$, and let $\varepsilon=1$,
then we have $W_1 = \{2, 4, 8\}$ leading to $f'_1(x)=2x_1+2x_2+4x_3+4x_4$, and any solution
$x\in X$ satisfies $f_1(x)\leqslant 2f'_1(x)$.

\medskip
Let us now prove that a $(1+\varepsilon)$-approximation of Problem~\ref{prob:MOBP}
can be obtained by finding the set of efficient solutions for $f'=(f'_1,\ldots,f'_p)$ where functions $f'_k$ are defined as in Lemma~\ref{l:ap:coeffs} for $k = 1,\ldots,p$.

\begin{theorem}\label{t:ap:coeffs}
Let $\Aeps\subseteq X$ be the set of efficient solutions for $f'=(f'_1,\ldots,f'_p)$ where functions $f'_k$ are defined as in Lemma~\ref{l:ap:coeffs} for $k = 1,\ldots,p$.
Then, $\Aeps$ is a $(1+\varepsilon)$-approximation of Problem~\ref{prob:MOBP}.
\end{theorem}

\begin{proof}
Since  $\Aeps$ is the set of efficient solutions for $f'$, we have that for any $x \in X$ there exists $x' \in \Aeps$ such that $f'(x') \wdom f'(x)$ and thus such that $(1+\varepsilon) f'(x') \wdom (1+\varepsilon) f'(x)$.\\
Then, by Lemma~\ref{l:ap:coeffs} applied to $x'$, we have $f(x') \wdom (1+\varepsilon) f'(x')$. Finally, by definition of $f$ and $f'$, which ensures that for any $x \in X$ we have $f'(x) \wdom f(x)$, we have $(1+\varepsilon) f'(x) \wdom (1+\varepsilon) f(x)$.\\
It follows that for any $x \in X$ there exists $x' \in \Aeps$ such that :
$$f(x') \wdom (1+\varepsilon) f'(x') \wdom (1+\varepsilon) f'(x) \wdom (1+\varepsilon) f(x)$$
establishing that $\Aeps$ is a $(1+\varepsilon)$-approximation for $f$.
\end{proof}

Let $y_{k,d}$ be defined for all $d\in D'_k$ and $k=1,\ldots,p$, where $D'_k$ is complete for $f'_k$,
\ie, $f'_k(X)\subseteq D'_k$.
Then, by Theorems~\ref{t:mcs-nondom} and~\ref{t:ap:coeffs}, the Pareto front for the approximate
problem can obtained by enumerating all MCSs given $\sfc=\{y_{k,d} \:|\: d \in D'_k \wedge 1\leqslant k \leqslant p\}$,
which corresponds to a $(1+\varepsilon)$-approximation set, $A_\varepsilon$, for the original problem.
Finally, note that $f'(\Aeps)$ is a lower bound set of Problem~\ref{prob:MOBP}.

\section{Computing $(1+\varepsilon)$-approximation sets and tighter approximation ratios}
\label{s:algs}

\noindent This section proposes the algorithms to solve MOBO problems through the corresponding MaxSAT 
formulation, and based on the two approximation versions of the unary encodings of the objective
function values (see Section~\ref{s:contr}).
Such algorithms are guaranteed to find a $(1+\varepsilon)$-approximation set, for any setting
of $\varepsilon \geqslant 0$.
The base algorithm, which is independent from the approximation version, is described in Section~\ref{s:algs:ap}, and then, exact algorithms based on iterative calls to the base
algorithm are proposed in Section~\ref{s:algs:reap}. Section 4.3 discusses the advantages and
disadvantages of each of the two approximation versions.

\subsection{$(1+\varepsilon)$-approximation and lower bound sets}
\label{s:algs:ap}

\noindent The previous section detailed the ideas of how to encode the values of
the objective functions of an optimization problem in such a way that any MCS enumeration algorithm
can be straightforwardly used to find a $(1+\varepsilon)$-approximation set, $\Aeps\subseteq X$,
where $\varepsilon \geqslant 0$.
This includes the case where $\varepsilon=0$, for which the approximation set is the Pareto front.
The algorithm does not have to be aware of the approximation ratio and not even which of the two
approximation methods is used. An advantage of these approximation methods, and consequently of
the algorithm, is to intrinsically provide a lower bound set, $\Lc\subset\osn$, within a ratio
lower than or equal to $(1+\varepsilon)$.
Hence, in practice, the algorithm ensures an approximation ratio tighter than $(1+\varepsilon)$,
which is upper bounded by $I_\epsilon(f(\Aeps),\Lc)$, \ie,
$I_\epsilon(f(\Aeps),Y_N)\leqslant I_\epsilon(f(\Aeps),\Lc)\leqslant (1+\varepsilon)$.

This section describes the MCS enumeration algorithm and gives examples of its application to
three scenarios: 
1) with $\varepsilon=0$;
2) with $\varepsilon > 0$ using the interval-based approximation; and
3) with $\varepsilon > 0$ using the coefficient-based approximation.

\subsubsection{Algorithm}

\noindent Algorithm~\ref{alg:approx}, named \mcskp, shows the pseudocode of the algorithm based on MCS enumeration to compute a $(1+\varepsilon)$-approximation set $A_\varepsilon$ (stored in $\Ac$), as well as a lower bound set $\Lc$.
In this case, $\Ac\subseteq X \times f(X)$ is a set of pairs, where each 
element $(\nu,f(\nu))\in\Ac$ represents a feasible solution $\nu$ and the corresponding image under $f$, $f(\nu)$.
For simplicity, $\Ac$ will be loosely referred as a $(1+\varepsilon)$-approximation set.
The algorithm works for any of the two approximation versions presented in the previous section. If $\varepsilon=0$, then Algorithm~\ref{alg:approx} computes the Pareto front, and in a such case $\Lc$ also corresponds to the Pareto front.

The input of Algorithm~\ref{alg:approx} is the following.
Let $f$ be the vector of objective functions and $f'$ be the modified vector as defined
in Lemma~\ref{l:ap:coeffs} if the coefficient-based approximation is used, otherwise $f'=f$. 
We assume that $f'_k$ is lower and
upper bounded by $l'_k$ and $u'_k$, respectively, for all $k\in\{1,\ldots,p\}$.
Let $\hc$ be the set of hard clauses that encode the constraints of Problem~\ref{prob:MOBP}
and the approximate objective functions $f'_k$ (as described in Section~\ref{sec:mcs-ur})
for all $k\in\{1,\ldots,p\}$. The set $D_k$ denotes the domain of $f'_k$.
In the case of the interval-based approximation (see Section~\ref{s:approx:int}), $D_k$
depends on $\varepsilon$, the greater $\varepsilon$ is, the smaller is $|D_k|$.
For the coefficient-based approximation (see Section~\ref{s:approx:coeff}),
$D_k$ is complete with respect to $f'_k$,
\eg, $D_k=\{l'_k,\ldots,u'_k+1\}$.

First, Algorithm~\ref{alg:approx} verifies if $\hc$ is satisfiable. If so,
the set of soft clauses $\sfc$ is created by adding a unary clause, $(y_{k,d})$,
for each of the values $d\in D_k$ for all $k=1,\ldots,p$ (line~\ref{alg:ap:fs}).
Then, while there are MCSs to find, the algorithm iteratively: i) finds an MCS $C$ and a corresponding (optimal) assignment $\nu$
(line~\ref{alg:ap:mcs}); ii) updates the set of solutions found by adding the assignment $\nu$
and the corresponding image in objective space $f(\nu)$ to $\Ac$ (line~\ref{alg:ap:pt}); iii)
updates the lower bound set by adding the representative point $r$ of $C$ (line~\ref{alg:ap:lbs});
iv) blocks the MCS $C$ (line~\ref{alg:ap:block}).
This blocking clause ensures that the next MCS found is such that there is a $q\in\{1,\ldots,p\}$
such that $y_{q,r_q}=1$. Hence, it prevents the algorithm from finding a solution weakly dominated by $r$
and, therefore, prevents it from visiting the MCS $C$ again.
Note that, typically, all literals in $C$ would have been added to the
clause, but since $y_{k,i}=1$ implies $y_{k,j}=1$ for $i<j$ then the remaining literals are redundant.

\DontPrintSemicolon
\SetKwFunction{SAT}{SAT}
\SetKwFunction{nondom}{NonDominated}
\SetKwFunction{MCS}{MCS}
\SetKwFunction{satf}{SAT}
\SetKwData{unsat}{UNSAT}
\SetKwData{sat}{SAT}
\SetKwData{true}{true}
\SetKwData{st}{st}

\SetKwFunction{BlockClause}{BlockingClause}

\SetVlineSkip{1pt}

\begin{algorithm}[!t]
  \small
  \KwIn{$f$, $f'$, $\hc$, $D$}
  \KwOut{A $(1+\varepsilon)$-approximation set $\Ac$ and a lower bound set $\Lc$}
  
  \If{\satf{$\hc$}}{
    $\sfc \gets \{y_{k,d} \:|\: d \in D_k \wedge 1\leqslant k \leqslant p\}$        \label{alg:ap:fs}\;
    $\Ac \gets \Lc \gets \{\}$\;
    $\st \gets \true$\;
    \While{$(\st = \true)$}{                                                        \label{alg:ap:beginiter} 
        $(\st, \nu, C) \gets$ \MCS{$\hc, \sfc$}                                     \label{alg:ap:mcs}
            \tcp*[r]{MCS $C$ and a corresponding assignment $\nu$ (if $st$ is $\true$)}
        \If{$(\st = \true)$}{
        $\Ac \gets \Ac \cup \{(\nu, f(\nu)\}$\tcp*[r]{$f(\nu)$ is a feasible point of Problem~\ref{prob:MOBP}}\label{alg:ap:pt}
            \For{$(k=1)$ \KwTo $(p)$}{                                              \label{alg:ap:forlb:i}
                $r_k \gets \max \{d \in D_k \:|\: y_{k,d} \in C\}$
                \tcp*[r]{$r$ is the representative point of the MCS $C$}            \label{alg:ap:forlb:e}
            }
        $\Lc \gets \Lc \cup \{(r_1,\ldots,r_p)\}$\tcp*[r]{\footnotesize $(r_1,\ldots,r_p)\wdom f'(\nu)\wdom f(\nu)$ holds} \label{alg:ap:lbs}

            $\cl \gets y_{1,r_{1}} \vee \ldots \vee y_{p,r_{p}}$                    \label{alg:ap:clause}\;
        $\hc \gets \hc \cup \{\cl\}$ \label{alg:ap:block} \tcp*[r]{\footnotesize Block the MCS (\ie, the region weakly dominated by $(r_1,\ldots,r_p)$)}
        }
    }
  }
  \Return{$(\Ac,\Lc)$}\tcp*[r]{\footnotesize Returns the approximation set and the lower bound set}
  \caption{\mcskp algorithm: Find a $(1+\varepsilon)$-approximation set ($\varepsilon\geqslant 0$),
  where $\hc$ encodes the problem constraints and the unary representation of the approximate
  objective functions, $f'=(f'_1,\ldots,f'_p)$, whose domains are given in $D=(D_1,\ldots,D_p)$.}\label{alg:approx}
\end{algorithm}

\subsubsection{Example}

\noindent Figure~\ref{fig:ex1} shows examples of approximation sets returned by Algorithm~\ref{alg:approx},
depending on the algorithms settings, \ie, the chosen approximation ratio and the approximation method, for the following problem:
    \begin{equation}\label{prob:ex1}
    \left\{\begin{array}{l}
        \min f_1(x) = 3x_1 + 3x_2 + x_3 + 2x_4 + 1\\
        \min f_2(x) = 4\bar{x_1} + 5\bar{x_2} + 5\bar{x_3} + 7\bar{x_4} + 1\\
        \mbox{subject to} \;\, x \in \Bmp
    \end{array}\right.
    \end{equation}

For the sake of simplicity, this unconstrained problem is used as example. Nevertheless,
the reasoning used next is the same for constrained problems since the MCS algorithm ensures
that only feasible solutions are computed.

Figure~\ref{fig:ex1:exact} shows all feasible points and, in black, the nondominated points which are
all returned by the exact version.
In such a case, the algorithm input is $\varepsilon=0$, $f=f'$, and $D_1=\{1,\ldots,11\}$
and $D_2=\{1,5,6,8,10,11,12,13,15,17,18,22,23\}$ 
which are complete, \ie, $f_1(X)\subset D_1$ and $f_2(X)\subset D_2$. 
Algorithm~\ref{alg:approx} finds a new nondominated point per iteration.
For example, in one of the iterations, line~\ref{alg:ap:mcs} returns the
MCS $C=\{y_{1,1},\ldots,y_{1,4},y_{2,1},\ldots,y_{2,10}\}$, 
the (only) corresponding model $\nu=(0,0,1,1)$
and $st=\text{``true"}$.
The latter indicates that a new MCS was found ($\hc$ is still satisfiable).
The corresponding assignment $\nu$ and $f(\nu)=(4,10)$
are added to $\Ac$ in line~\ref{alg:ap:pt}.
The representative point of $C$, which in this case is $(4,10)$, is computed and added to the
lower bound set in lines~\ref{alg:ap:forlb:i}-\ref{alg:ap:lbs}.
To prevent the algorithm from returning any other solution weakly dominated by $r$,
lines~\ref{alg:ap:clause}-\ref{alg:ap:block} block the MCS $C$ by constructing and adding the clause 
$(y_{1,4} \vee y_{2,10})$ to the current set of hard clauses, $\hc$. Hence,
any subsequent solution found must be strictly lower than $4$ in objective 1
or strictly lower than $10$ in objective 2.
The algorithm stops when $st=\text{``false"}$ is returned in line~\ref{alg:ap:mcs} which indicates that
there are no more MCSs and, consequently, no more feasible points to be found.
In the example, this will happen only after all 6 nondominated points are found:
$Y_N=\{(1,22), (2,17), (3,15), (4,10), (7,5), (10,1)\}$.

\begin{figure}[t]
        \centering
        \begin{subfigure}[b]{0.27\linewidth}
            \centering
            \includegraphics[width=\linewidth]{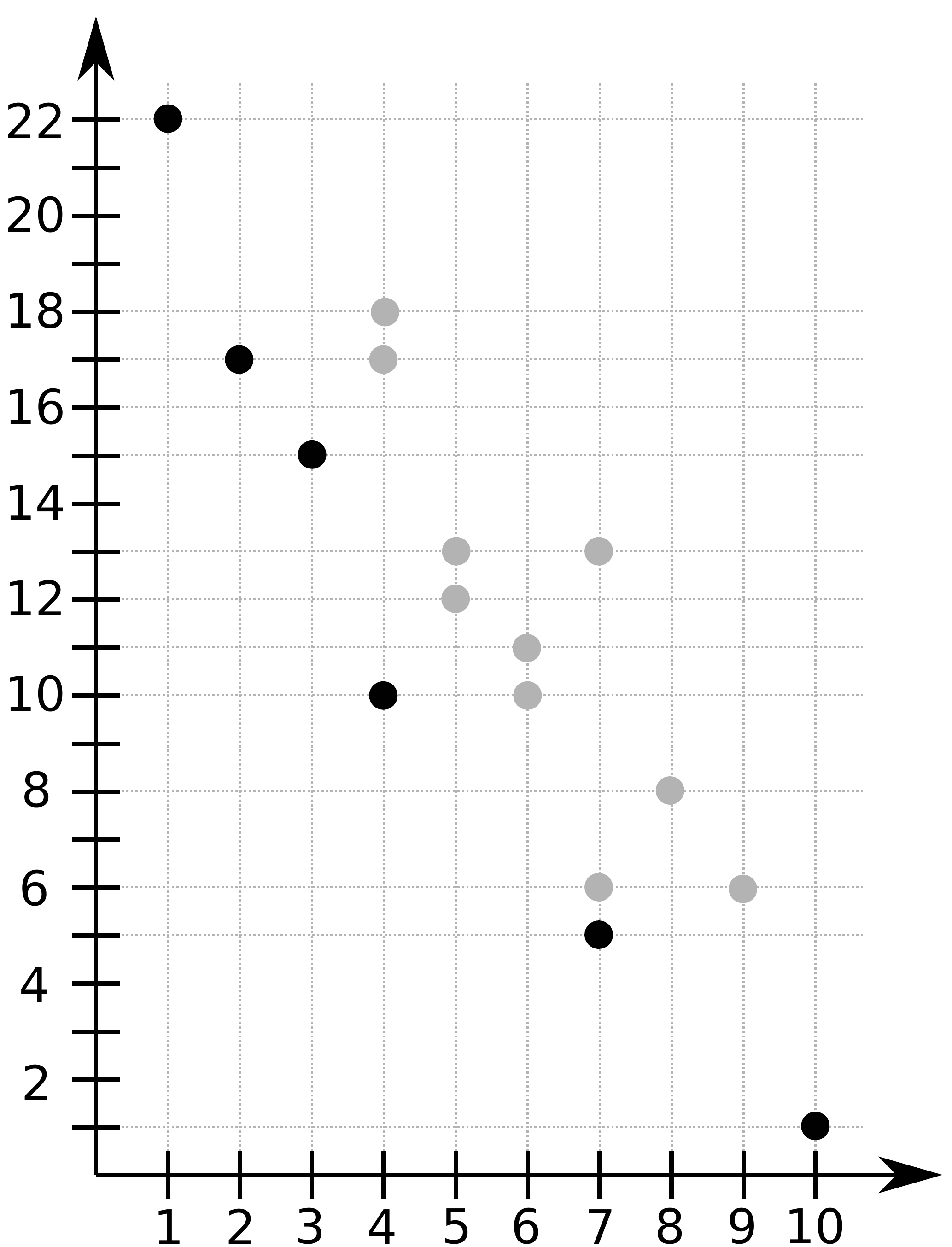}
            \caption[]{$\varepsilon=0$, exact version}
            \label{fig:ex1:exact}    
        \end{subfigure} %
%         \hfill
        \begin{subfigure}[b]{0.27\linewidth}  
            \centering 
            \includegraphics[width=\linewidth]{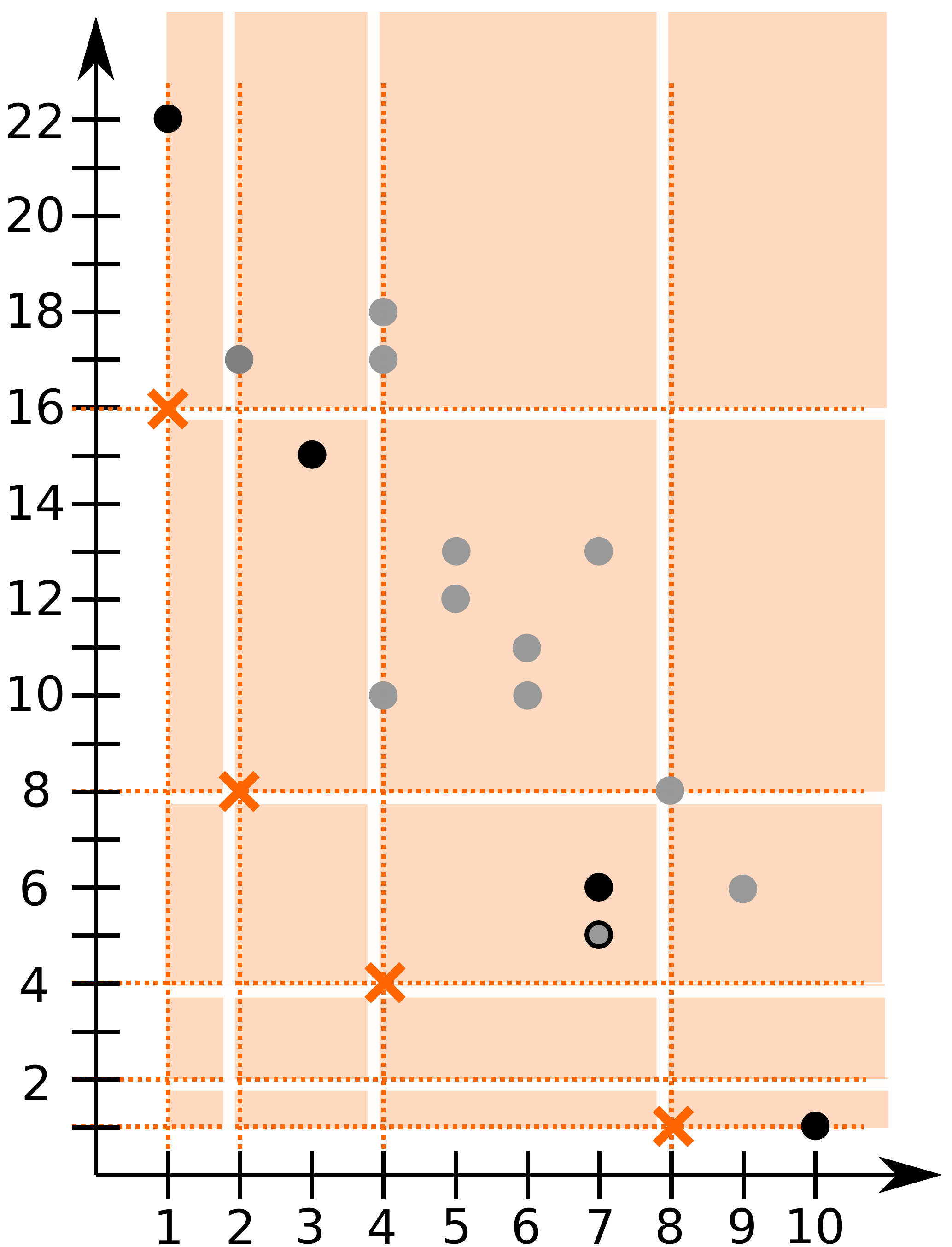}
            \caption[]{$\varepsilon=1$, interval-based}
            \label{fig:ex1:approx1}    
        \end{subfigure}
%         \hfill
        \begin{subfigure}[b]{0.27\linewidth}  
            \centering 
            \includegraphics[width=\linewidth]{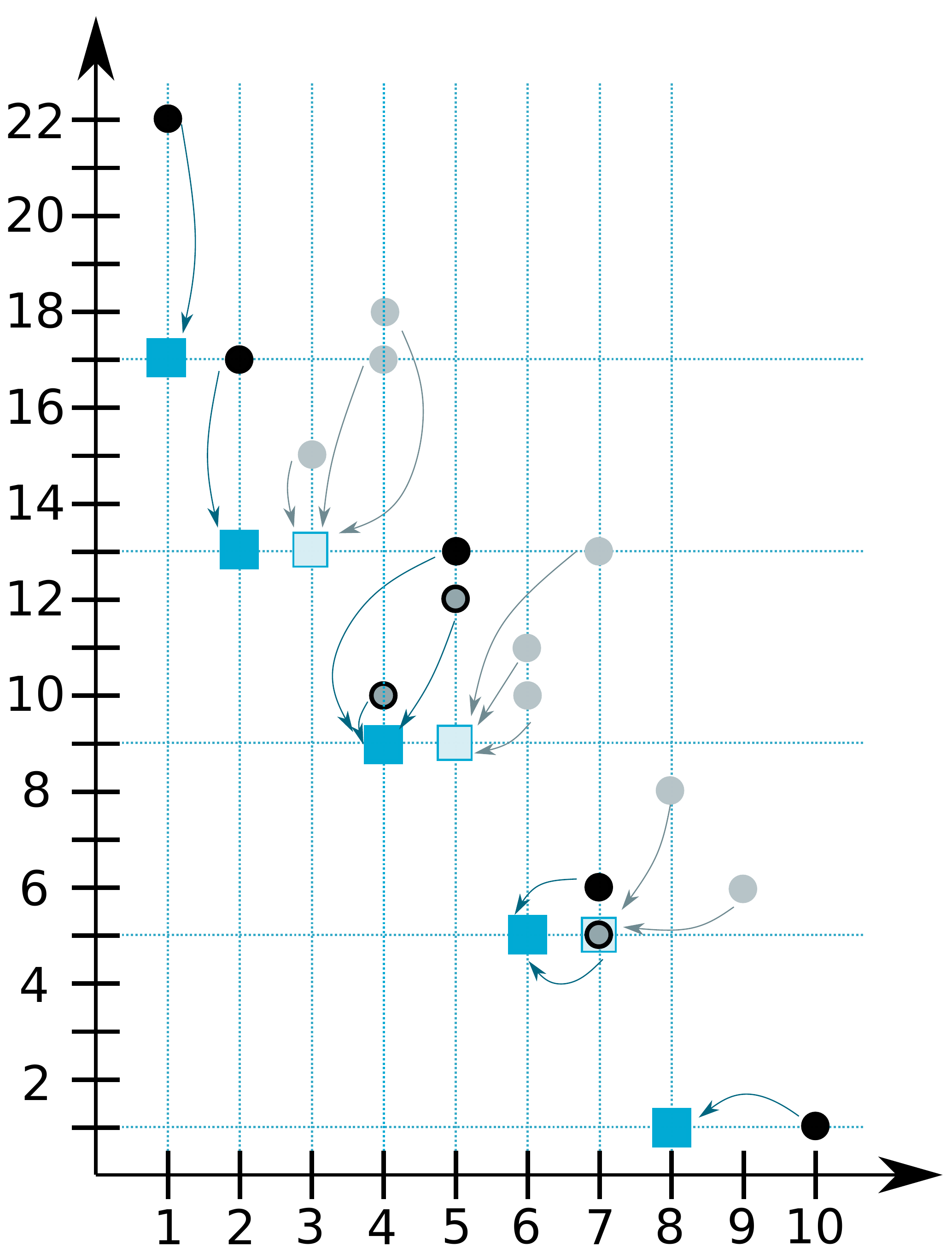}
            \caption[]{$\varepsilon=1$, coefficient-based}
            \label{fig:ex1:approx2}    
        \end{subfigure}
        \caption{Example of the approximation sets returned by Algorithm~\ref{alg:approx} for the problem
        in~\eqref{prob:ex1}, for different settings of $\varepsilon$ and different approximation methods.
        }
        \label{fig:ex1}
    \end{figure}

Figure~\ref{fig:ex1:approx1} shows an example of a $2$-approximation set (black dots),
and the corresponding lower bound set (orange crosses) returned by
Algorithm~\ref{alg:approx} considering the interval-based approximation.
The gray dots with black outline represent other points that could have also been returned.
In such a case, the algorithm input is $\varepsilon=1$, $f=f'$,
and $D_1=\{1,2,4,8,16\}$, $D_2=\{1,2,4,8,16,32\}$. 
The shaded regions show the intersection of the intervals associated to $D_1$ and $D_2$,
which contain all feasible points.
In this case, as there are only four MCSs, Algorithm~\ref{alg:approx} executes four iterations, and in
each one it finds a solution (which may not be an efficient solution) within the $(1+\varepsilon)$ ratio,
and a point of the lower bound set that dominates it. For example, when the
MCS $C=\{y_{1,0},y_{1,2},y_{1,4},y_{2,0},y_{2,2},y_{2,4}\}$
is found, only the solutions mapping to either point
$(7,5)$ or point $(7,6)$
correspond to $C$. In the example, we assume that the solution mapping to the latter point is the
one returned in line~\ref{alg:ap:mcs}. Then, the point added to the lower bound set is $(4,4)$ % $(2,10)$
because this is the representative point of $C$, which is computed in
lines~\ref{alg:ap:forlb:i}-\ref{alg:ap:forlb:e}.
By blocking the MCS $C$ in line~\ref{alg:ap:block} with the clause $(y_{1,4} \vee y_{2,4})$, solutions
mapping to a point weakly dominated by $(4,4)$ become infeasible under $\hc$.
In the example, the algorithm returns in $\Ac$ the approximation set $A_\varepsilon$
and the corresponding mapping to $f(A_\varepsilon)=\{(1,22),(3,15),(7,6),(10,1)\}$, and the
lower bound set $\Lc=\{(1,16),(2,8),(4,4),(8,1)\}$, which provide \emph{a posteriori}
guaranteed approximation ratio of $I_\epsilon(A_\varepsilon,\Lc)=1.875\geqslant I_\epsilon(A_\varepsilon,Y_N)=1.5$.

Figure~\ref{fig:ex1:approx2} shows an example of a $2$-approximation set (black dots) returned by
Algorithm~\ref{alg:approx} considering the coefficient-based approximation.
The gray dots with black outline also represent other points that could have also been returned.
In this approximation version with $\varepsilon=1$, the approximate objective functions are:
$f'_1(x) = 2x_1 + 2x_2 + 1x_3 + 2x_4 + 1$ and
$f'_2(x) = 4\bar{x_1} + 4\bar{x_2} + 4\bar{x_3} + 4\bar{x_4} + 1$.

In this case, $D_1=\{1,\ldots,9\}$, $D_2=\{1,5,9,13,17,18\}$ are complete with respect to $f'_1$ and $f'_2$,
respectively (instead of $f_1$ and $f_2$).
In Figure~\ref{fig:ex1:approx2}, the arrows represent, for each
solution $x\in\Bmp$, the mapping of $f(x)$ to $f'(x)$, and the blue squares represent all
feasible points of $f'$ whereas the darkest ones represent the nondominated points under $f'$.
For each of these dark blue squares, \ie, for each nondominated point under $f'$, a solution $\nu$
mapping into it and the corresponding evaluation under $f$ (the original objective functions) are
added to $\Ac$. 
The representative point, which in this case must be $f'(\nu)$, is added to the lower bound set.
Blocking the MCS in this case is equivalent to blocking the region weakly dominated by $f'(\nu)$.
In the example, the algorithm returns in $\Ac$ the approximation set $A_\varepsilon$
and the corresponding mapping to $f(A_\varepsilon)=\{(1,22),(2,17),(5,13),(7,6),(10,1)\}$,
and the lower bound set $\Lc=\{(1,17),(2,13),(4,9),(6,5),(8,1)\}$, \ie, the set of black points, and
the set of dark-blue squares, respectively.
In this case, the \emph{a posteriori}
guaranteed approximation ratio is $I_\epsilon(A_\varepsilon,\Lc)\approx1.4\geqslant I_\epsilon(A_\varepsilon,Y_N)=1.3$.

\subsubsection{Enumerating the efficient set}
% Enumerate the efficient set
\noindent Although Algorithm~\ref{alg:approx} only finds one solution for each feasible point returned in
$\Ac$, the algorithm could be adapted to find all solutions mapping to each such
point (assuming $\varepsilon=0$).
This is achieved by replacing the blocking clause in line~\ref{alg:ap:clause}
by $p$ clauses blocking, in the objective space, the region dominated
by the representative point $r$ of the last MCS found, and a clause blocking, in the decision space,
the assignment $\nu$.
Thus, the representative point $\hat{r}$ of the next MCS found cannot be dominated by $r$
(note that $\hat{r}=r$ is allowed) and the new assignment must be a newly one. Hence, all assignments associated to each MCS would be enumerated with this modified version.
For $\varepsilon=0$, this results in enumerating the whole efficient set.
For $\varepsilon>0$, at least one efficient solution would be enumerated per
MCS but non-efficient solutions could be enumerated as well, which makes this modified version of
lesser value for $\varepsilon>0$.

\subsection{Enumerating all nondominated points through re-approximations}
\label{s:algs:reap}

\noindent The value of $\varepsilon$ has impact on the size of the encoding of the objective functions
and on the search process, that depends on which of the two described approximation methods is used.
In the interval-based approximation, $|D_k|$ represents the number of intervals into which the
range of objective values $\{\ell_k,\ldots,u_k\}$ is split, for $k\in\{1,\ldots,p\}$.
The larger $\varepsilon$ is, the smaller is $|D_k|$, whereas $|D_k|$ tends to $u_k-\ell_k+1$ as
$\varepsilon$ tends to zero. Hence, setting a variable $y_{k,d}$ to 1
allows to skip larger regions of the objective space when $\varepsilon$ is large, than when it is small. 
In the coefficient-based approximation, the coefficients of the approximate
function $f'_k$ tend to 1 and the upper limit $u'_k$ tends to $n$, as $\varepsilon$ grows.
Consequently, the size of the encoding of the objective functions,
and consequently the computational burden thereof, is expected to decrease when increasing $\varepsilon$.

This section describes two algorithms taking advantage of the above observations,
one using coefficient-based approximations (Section~\ref{s:algs:enum:coeff}) and another using
interval-based approximations (Section~\ref{s:algs:enum:int}).
The main idea is to start with a large $\varepsilon$, then to iteratively use Algorithm~\ref{alg:approx}
to find a $(1+\varepsilon)$-approximation for progressively smaller values of $\varepsilon$,
until $\varepsilon$ is small enough.
The approximation ratio of the approximation set is improved from iteration to iteration.
If $\varepsilon=0$ in the last iteration, then an optimal set is returned.

\subsubsection{Coefficient-based re-approximations}
\label{s:algs:enum:coeff}

\DontPrintSemicolon

\SetKwFunction{iapprox}{intervalApproximation}
\SetKwFunction{capprox}{CoeffApprox}
\SetKwFunction{approx}{Approximate}
\SetKwFunction{computeD}{ComputeD}
\SetKwFunction{updateeps}{UpdateRatio}
\SetKwFunction{encodelt}{EncodeLT}
\SetKwFunction{encodeof}{EncodeObjFunction}
\SetKwFunction{satf}{SAT}
\SetKwFunction{enumalg}{\mcskp}
\SetKwFunction{stopc}{StopCriterion}
\SetKwData{stop}{stop}
\SetKwData{break}{break}
\SetKwData{false}{false}
\SetKwData{true}{true}
\SetVlineSkip{1pt}

\begin{algorithm}[!t]
  \small
  \KwIn{$f$, $\phi$, $\varepsilon$}
  \KwOut{An optimal set $\Ac$}
  
  $\Ac \gets \{\}$\;

  \Repeat{$f=f' \textbf{ or}$ \stopc{$\varepsilon$}}{ \label{alg:reap:coeff:r}
        $\hc \gets \phi$\;
        $\Lc \gets \{\}$\;
        \For{$(k=1)$ \KwTo $(p)$}{                                      \label{alg:reap:coeff:f:i}
            $f_k' \gets$ \capprox{$f_k, 1+\varepsilon$}                 \label{alg:reap:coeff:ap}
                \tcp*[r]{\footnotesize Update approximation of $k$-th objective function}
                
            $\phi' \gets$ \encodeof($f_k'$)\;                           \label{alg:reap:coeff:of}
            $\hc \gets \hc \cup \phi'$\;                                \label{alg:reap:coeff:uenc}

            $D_k \gets$ \computeD{$f'_k, 1$}                            \label{alg:reap:coeff:D}
                \tcp*[r]{\footnotesize Compute $D_k$ for $f'_k$}
            \For{$d \in D_k$}{                                          \label{alg:reap:coeff:fy}
%                 \agt{$(y_{k,d}, \phi'') \gets$} \encodelt{$\sov_k,d$}     \label{alg:reap:coeff:fy:y}
                $(y_{k,d}, \phi'') \gets$ \encodelt{$k,d$}              \label{alg:reap:coeff:fy:y}
                     \tcp*[r]{\footnotesize Encode $y_{k,d}=1 \implies f'_k(x) < d$}
                     
                $\hc \gets \hc \cup \phi''$\;                           \label{alg:reap:coeff:fy:o}
            }
    }
    \For{$(\nu, f(\nu)) \in \Ac$}{                                      \label{alg:reap:coeff:forb}
        $\hc \gets \hc \cup \{(y_{1,f'_{1}(\nu)} \vee \ldots \vee y_{p,f'_{p}(\nu)})\}$              \label{alg:reap:coeff:forb:ufb}
        \tcp*[r]{\footnotesize Block region weakly dominated by $f'(\nu)$}
        $\Lc \gets \Lc \cup \{f'(\nu)\}$                                \label{alg:reap:coeff:forb:ufl}
    }
%     \If{\satf{$\hc$}}{                                               \label{alg:reap:coeff:sat}
        $(\Ac', \Lc') \gets$ \enumalg{$f, f',\hc, D$}\;                 \label{alg:reap:coeff:mcs}
        $\Ac \gets$ \nondom{$\Ac \cup \Ac'$}                            \label{alg:reap:coeff:uA}
            \tcp*[r]{\footnotesize $\Ac$ stores a $(1+\varepsilon)$-approximation set} 
        $\Lc \gets$ \nondom{$\Lc \cup \Lc'$}                                      
            \label{alg:reap:coeff:uL}
            \tcp*[r]{\footnotesize $\Lc$ is a lower bound set} 
%     }
        $\varepsilon \gets$ \updateeps{$\varepsilon$}\;                 \label{alg:reap:coeff:ueps}
  }                                                                     \label{alg:reap:coeff:r:e}
  \Return{$\Ac,\Lc$}\;
  \caption{\rcb algorithm: Enumerate all nondominated points for $f(x)$ subject to the problem constraints
  encoded in $\phi$, through re-approximations using the coefficient-based approximation method.
  The initial approximation ratio is $\varepsilon$.
  }\label{alg:reap:coeff}
\end{algorithm}

\noindent Algorithm~\ref{alg:reap:coeff} describes a procedure to sequentially update an
approximation set $\Ac$ while improving the corresponding approximation ratio. 
Given the objective functions $f$, a set of clauses $\phi$ representing the problem constraints,
and an initial value of $\varepsilon$, the algorithm runs until $\Ac$ represents an optimal
solution set.
In each iteration, the algorithm performs four main steps: 
1) encodes the (approximate) objective functions  in lines~\ref{alg:reap:coeff:f:i}-\ref{alg:reap:coeff:fy:o}
given the current value of $\varepsilon$; 
2) a region in objective space within a $(1+\varepsilon)$ ratio of $\Ac$
is blocked in lines~\ref{alg:reap:coeff:forb}-\ref{alg:reap:coeff:forb:ufb};
3) in lines~\ref{alg:reap:coeff:mcs}-\ref{alg:reap:coeff:uL}, $\Ac$ and $\Lc$ are updated using Algorithm~\ref{alg:approx} to search for solutions mapping to points outside
the blocked region, and to make sure that $\Ac$ becomes a $(1+\varepsilon)$-approximation
and $\Lc$ the respective lower bound set;
4) the value $\varepsilon$ is updated (it is decreased) at the end of each iteration,
in line~\ref{alg:reap:coeff:ueps}.

In the first step, the approximate version, $f'_k$, of the objective function $f_k$ is
computed first for each $k\in\{1,\ldots,p\}$ in line~\ref{alg:reap:coeff:ap},
then the set of clauses $\phi'$ encoding $f'_k$ (\eg, as explained in Section~\ref{sec:mcs-ur})
is computed in line~\ref{alg:reap:coeff:of},
and is added to $\hc$. 
The set $D_k$ is computed in line~\ref{alg:reap:coeff:D}, where the second
argument of the procedure $\computeD$ indicates the interval-approximation factor for computing
$D_k$. In this case (when the coefficient-based approximation is used), this factor must be $1$
and therefore, $D_k$ is complete.
The procedure $\encodelt(k,d)$ in line~\ref{alg:reap:coeff:fy:y} computes the set of
clauses $\phi''$ that encodes the implication $y_{k,d}=1 \implies f'_k(x) < d$.
These clauses, $\phi''$, are then added to $\hc$ in line~\ref{alg:reap:coeff:fy:o}.

The second step is skipped (only) in the first iteration of the repeat-until loop as $\Ac$
is empty, \ie, no solutions are known.
Let $\varepsilon'$ be the value of $\varepsilon$ from the previous iteration, hence, $\varepsilon'>\varepsilon$.
At the beginning of a new iteration, the set $\Ac$ is a $(1+\varepsilon')$-approximation.
Although not all the efficient solutions are,
in general, expected to be within the updated, and tighter, $(1+\varepsilon)$ ratio from the solutions
stored in $\Ac$, some may be.
To take advantage of this, for each solution $\nu$ stored in $\Ac$,
the region weakly dominated by $f'(\nu)$ is blocked in line~\ref{alg:reap:coeff:forb:ufb}.
The union of the blocked regions is represented by $\Lc$ in line~\ref{alg:reap:coeff:forb:ufl}, \ie, $\Lc=f'(\Ac)$.
At this point, $\hc$ encodes the problem constraints, the $p$ approximate objective
functions $f'_1,\ldots,f'_p$, the $y$ variables needed, and the objective-space region
that is blocked. % (the region weakly dominated by $f'(\Ac)$).
Hence, $\hc$ is satisfiable only if there is a solution that satisfies the problem constraints, $\phi$,
and which maps to a point outside the blocked region.

In the third step, Algorithm~\ref{alg:approx} is executed in line~\ref{alg:reap:coeff:mcs}
and returns the set $\Ac'$ that represents a $(1+\varepsilon)$-approximation with respect to the set
of feasible solutions mapping to points outside the blocked region, and the respective lower bound set, $\Lc'$.
Hence, $\Ac'\cup\Ac$ is a $(1+\varepsilon)$-approximation set, and $\Lc'\cup\Lc$ is a lower bound set.
The set of solutions that are not dominated by any other solution in $\Ac'\cup\Ac$ is then stored in $\Ac$
in line~\ref{alg:reap:coeff:uA}. Similarly, only the mutually nondominated points in $\Lc'\cup\Lc$ are
stored in $\Lc$ in line~\ref{alg:reap:coeff:uL}.

Finally, in the last step, the approximation ratio $\varepsilon$ is updated (line~\ref{alg:reap:coeff:ueps}).
Different update strategies may be used, where the only requirement is to decrease the
value of $\varepsilon$.
Setting $\varepsilon$ to zero ensures that all points in the Pareto front are enumerated at
the end of the next iteration.
The algorithm will stop when
all approximate functions are equal to the corresponding original objective functions ($f=f'$),
which will happen if $\varepsilon=0$ but may also happen for small enough values of $\varepsilon$
greater than zero.
Note that the algorithm may be stopped earlier using a user-defined criterion such as when
a time limit or a desired approximation ratio is reached.

\begin{figure}[t]
        \centering
        \begin{subfigure}[b]{0.27\linewidth}
            \centering
            \includegraphics[width=\linewidth]{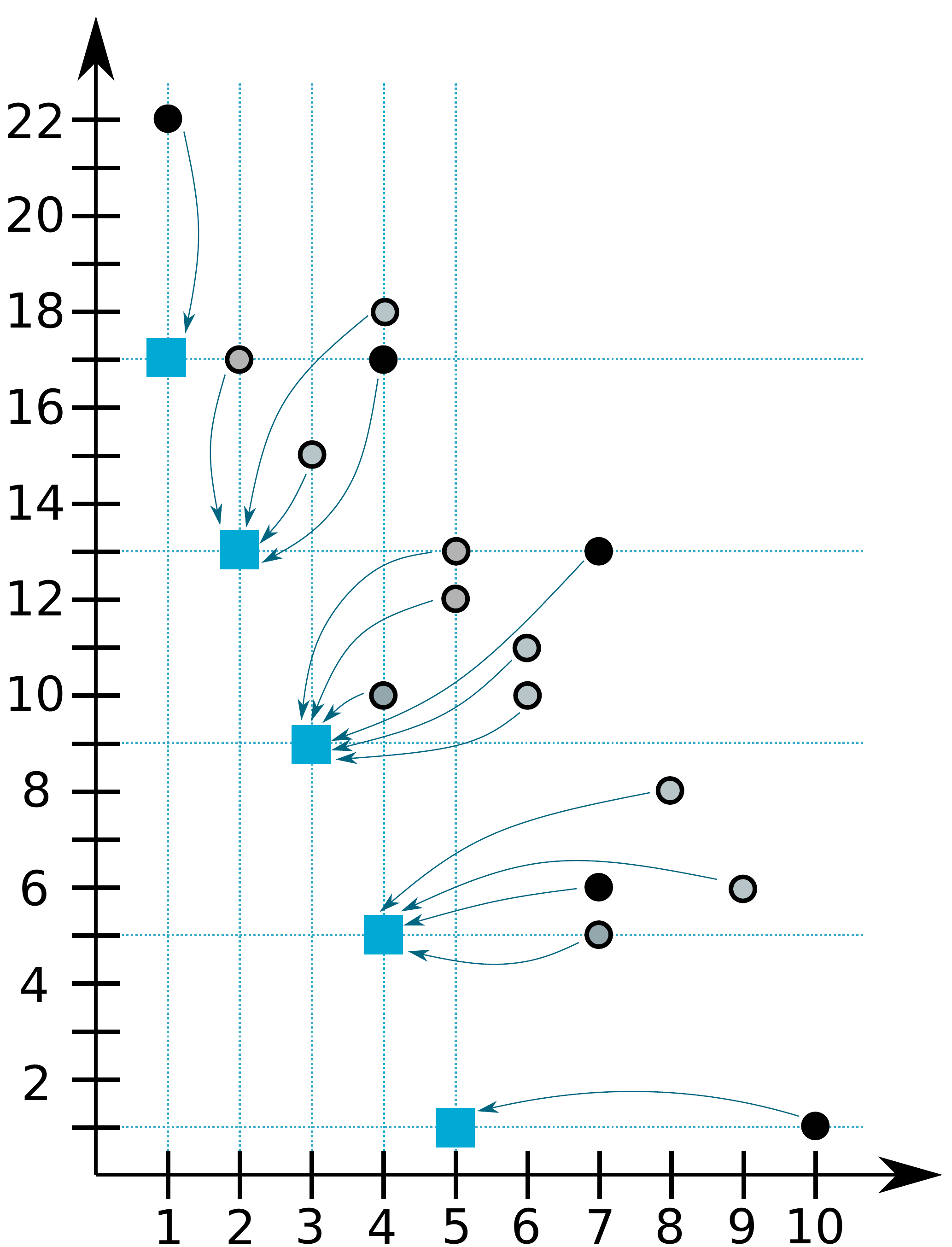}
            \caption[]{$\varepsilon=3$}
            \label{fig:alg:coeffs:eps4}    
        \end{subfigure} %
%         \hfill
        \begin{subfigure}[b]{0.27\linewidth}  
            \centering 
            \includegraphics[width=\linewidth]{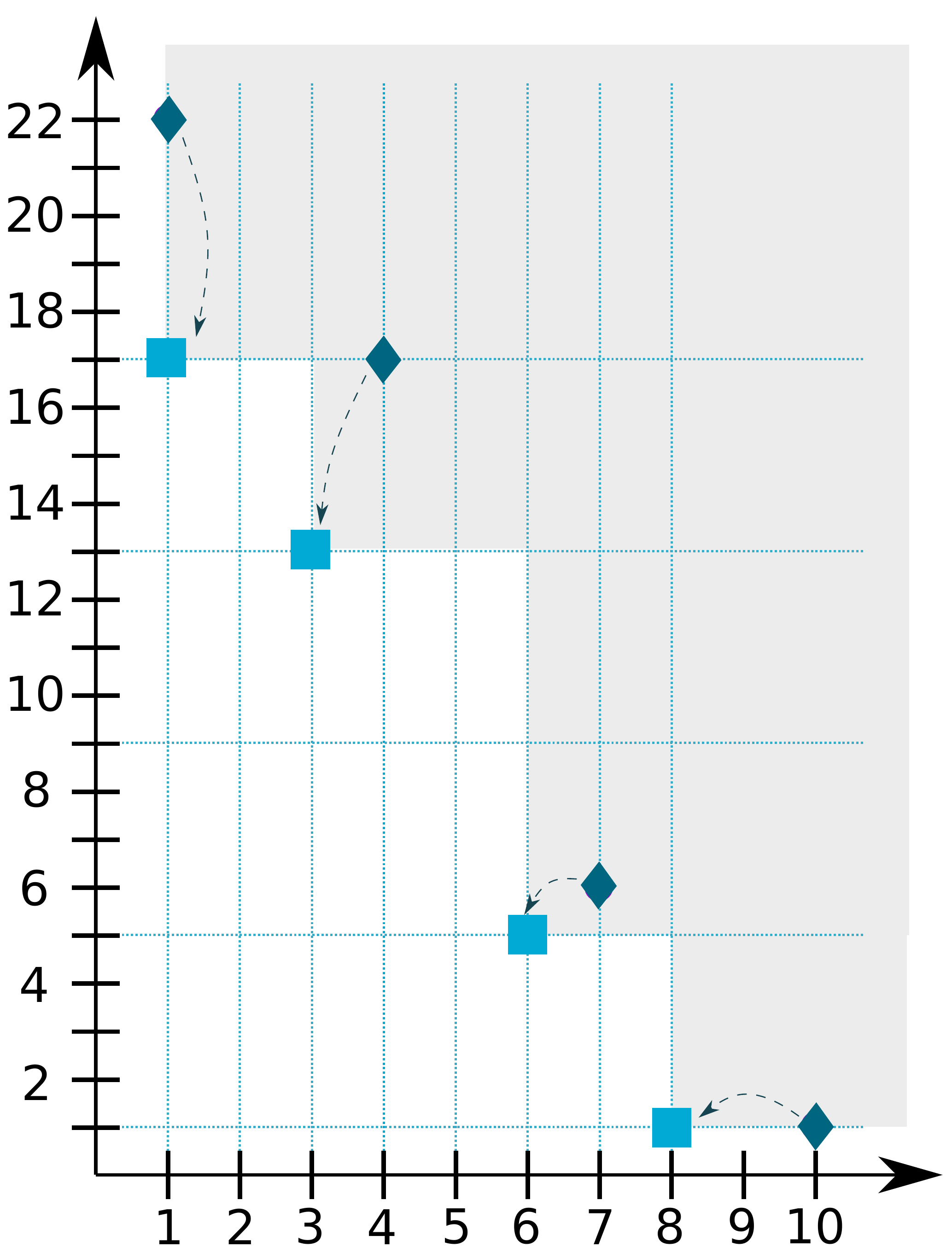}
            \caption[]{$\varepsilon=1$, blocked region}
            \label{fig:alg:coeffs:eps2b}    
        \end{subfigure}
%         \hfill
        \begin{subfigure}[b]{0.27\linewidth}  
            \centering 
            \includegraphics[width=\linewidth]{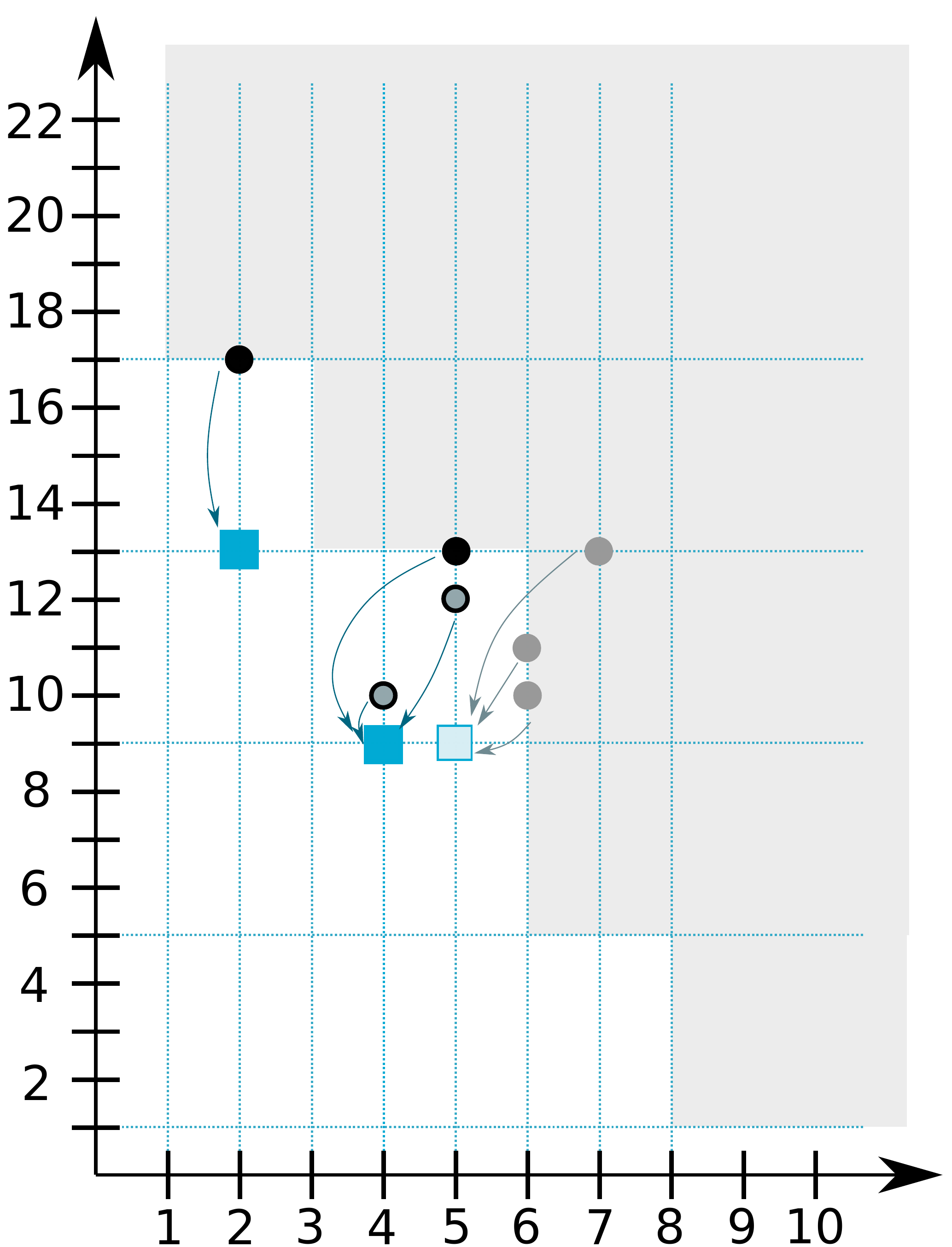}
            \caption[]{$\varepsilon=1$}
            \label{fig:alg:coeffs:eps2}    
        \end{subfigure}
        \caption{Example of the first two iterations of Algorithm~\ref{alg:reap:coeff}:
    (a) illustration of the problem solved in the first iteration for $\varepsilon=3$, 
    (b) the region blocked in the beginning of the second iteration for $\varepsilon=1$, and
    (c) the problem solved after.} 
        \label{fig:alg:coeffs}
    \end{figure}

Figure~\ref{fig:alg:coeffs} illustrates two iterations of Algorithm~\ref{alg:reap:coeff}
for the problem defined in expression~\eqref{prob:ex1}.
Figure~\ref{fig:alg:coeffs:eps4} shows the feasible points under $f'$ given $\varepsilon=3$ (blue squares)
where
$f'_1(x) = x_1 + x_2 + x_3 + x_4 + 1$ and $f'_2(x) = 4\bar{x_1} + 4\bar{x_2} + 4\bar{x_3} + 4\bar{x_4} + 1$.
Since they are all nondominated, Algorithm~\ref{alg:approx} finds one solution mapping to each one
of these points (black circles), which are returned in $\Ac'$ in line~\ref{alg:reap:coeff:mcs}.
Note that, $(7,13)$ is weakly dominated by $(7,6)$ and thus, it is discarded in
line~\ref{alg:reap:coeff:uA}. 
The resulting set $\Ac = \{ ((0,0,0,0),(1,22)),((0,1,0,0),(4,17)),((1,0,1,1),(7,6)),((1,1,1,1),(10,1))\}$ is a $(1+3)$-approximation set. Assuming that
$\varepsilon$ is then set to $1$ in line~\ref{alg:reap:coeff:ueps}, the next iteration encodes the new approximate
objective functions, 
which are now $f'_1(x) = 2x_1 + 2x_2 + x_3 + 2x_4 + 1$ and
$f'_2(x) = 4\bar{x_1} + 4\bar{x_2} + 4\bar{x_3} + 4\bar{x_4} + 1$.
For each solution $\nu$ stored in $\Ac$, $f(\nu)$ and $f'(\nu)$ are represented by a diamond and
square, respectively, in  Figure~\ref{fig:alg:coeffs:eps2b} and they are linked by an arrow.
The shaded region in Figure~\ref{fig:alg:coeffs:eps2b} is the region blocked in lines~\ref{alg:reap:coeff:forb}-\ref{alg:reap:coeff:forb:ufb}. It is lower bounded by the images of the assignments in $\Ac$ according to the newly defined $f'$, namely points (1,17), (3,13), (6,5), and (8,1).
Hence, all solutions $\nu$ such that $f'(\nu)$ is in the shaded region are now considered infeasible.
Figure~\ref{fig:alg:coeffs:eps2} shows the remaining feasible points under $f'$ (squares) and the
points under $f$ associated to them (circles), where the solutions mapping to black circles are
returned in $\Ac'$ in line~\ref{alg:reap:coeff:mcs}.
At the end of the iteration, the set $A_\varepsilon$ represented in $\Ac$,
where $f(A_\varepsilon)=\{(1,22),(2,17),(5,13),(7,6),(10,1)\}$, is a $(1+1)$-approximation and 
the associated lower bound set is $\Lc=\{(1,17),(2,13),(4,9),(6,5),(8,1)\}$.

Note that, in the example, the approximate objective functions $f'_2$ does not change for the two
settings of $\varepsilon$. In this case, there is no need to repeat 
lines~\ref{alg:reap:coeff:of}-\ref{alg:reap:coeff:fy:o} for $k=2$.
The algorithm can be easily adapted to avoid repeating the encoding of approximate functions that do not change.

\subsubsection{Interval-based re-approximations}
\label{s:algs:enum:int}

\DontPrintSemicolon

\SetKwFunction{approx}{Approximate}
\SetKwFunction{computeD}{ComputeD}
\SetKwFunction{updateeps}{UpdateRatio}
\SetKwFunction{encodelt}{EncodeLT}
\SetKwFunction{encodeof}{EncodeObjFunction}
\SetKwFunction{satf}{SAT}
\SetKwFunction{enumalg}{\mcskp}
\SetKwFunction{stopc}{StopCriterion}
\SetKwData{stop}{stop}
\SetKwData{break}{break}
\SetKwData{false}{false}
\SetKwData{true}{true}
\SetKwData{sat}{SAT}
\SetKwData{unsat}{UNSAT}
\SetKwData{st}{st}
\SetVlineSkip{1pt}

\begin{algorithm}[!t]
  \small
  \KwIn{$f$, $\phi$, $\varepsilon$}
  \KwOut{An optimal set $\Ac$}
  $\Ac \gets \Ac' \gets \hc \gets \{\}$\;
    
  \For{$(k=1)$ \KwTo $(p)$}{                                            \label{alg:reap:int:forof}
            $\phi' \gets$ \encodeof{$f_k$}                              \label{alg:reap:int:of}
                    \tcp*[r]{\footnotesize Encode $k$-th objective function}
            $\hc \gets \hc \cup \phi'$\;                                \label{alg:reap:int:uenc}
  }
  \Repeat{$\Ac'=\{\} \textbf{ or}$ \stopc{$\varepsilon$}}{              \label{alg:reap:int:r}
  
    \For{$(k=1)$ \KwTo $(p)$}{                                          \label{alg:reap:int:ford}
        $D_k \gets$ \computeD{$f_k, 1+\varepsilon$}                     \label{alg:reap:int:D}
                    \tcp*[r]{\footnotesize Compute $D_k$ for $f_k$}
        \For{$d \in D_k$}{                                              \label{alg:reap:int:fy}
%             $(y_{k,d}, \phi'') \gets$ \encodelt{$O_k,d$}                  \label{alg:reap:int:fy:y}
            $(y_{k,d}, \phi'') \gets$ \encodelt{$k,d$}                  \label{alg:reap:int:fy:y}
                    \tcp*[r]{\footnotesize Encode $y_{k,d}=1 \implies f_k(x) < d$}
            $\hc \gets \hc \cup \phi''$\;                               \label{alg:reap:int:fy:o}
        }
        
    }

    \For{$(\nu, f(\nu)) \in \Ac'$}{                                     \label{alg:reap:int:forb}
        \For{$(k=1)$ \KwTo $(p)$}{                                      \label{alg:reap:int:forb:forp}
            $d_k \gets f_k(\nu)$\;                                      \label{alg:reap:int:forb:forp:d}
%             $(y_{k,d_k}, \phi'') \gets$ \encodelt{$O_k,d_k$}\;            \label{alg:reap:int:forb:forp:y}
            $(y_{k,d_k}, \phi'') \gets$ \encodelt{$k,d_k$}\;            \label{alg:reap:int:forb:forp:y}
            $\hc \gets \hc \cup \phi''$\;                               \label{alg:reap:int:forb:forp:o}
        
        }
        $\hc \gets \hc \cup \{(y_{1,d_{1}} \vee \ldots \vee y_{p,d_{p}})\}$                                 \label{alg:reap:int:forb:ufb}
                    \tcp*[r]{\footnotesize Block region weakly dominated by $f(\nu)$}
    }
    $\Lc \gets \{f(\nu) \:|\: (\nu,f(\nu)) \in \Ac\}$\;                 \label{alg:reap:int:forb:ul}
%     $\st \gets$ \satf{\agt{$\hc$}}\;                         \label{alg:reap:int:sat}
%     \If{\agt{$st=\true$}}{                                              \label{alg:reap:int:ifsat}
        $(\Ac', \Lc') \gets$ \enumalg{$f, f,\hc, D$}\;                   \label{alg:reap:int:mcs}
        $\Ac \gets$ \nondom{$\Ac \cup \Ac'$}               
        \label{alg:reap:int:uA}
                    \tcp*[r]{\footnotesize $\Ac$ stores a $(1+\varepsilon)$-approximation set}
        $\Lc \gets$ \nondom{$\Lc \cup \Lc'$}         
        \label{alg:reap:int:uL}\;
        $\varepsilon \gets$ \updateeps{$\varepsilon$}\;                 \label{alg:reap:int:ueps}
%     }
  }
  \Return{$\Ac,\Lc$}\;
  \caption{\rib algorithm: Enumerate all nondominated points of $f$ subject to the problem constraints
  encoded in $\phi$, through re-approximations using the interval-based approximation method.
  The initial approximation ratio is $\varepsilon$.}\label{alg:reap:int}
\end{algorithm}

\noindent Algorithm~\ref{alg:reap:int} is similar to Algorithm~\ref{alg:reap:coeff} in the sense that
it follows an identical scheme to iteratively improve an approximation set $\Ac$ and the
corresponding approximation ratio, until $\Ac$ represents an optimal set.
Similarly, the input of Algorithm~\ref{alg:reap:int} are a set of objective functions, $f$,
the problem constraints encoded as a set of clauses $\phi$, and an initial value of $\varepsilon$.
The algorithm has a preprocessing step in lines~\ref{alg:reap:int:forof}-\ref{alg:reap:int:uenc}
where all objective functions are encoded.
Then, in each iteration, the algorithm also performs four main steps:
1) determine the objective functions intervals by computing $D_1,\ldots,D_p$, and encode the
associated $y$ variables, in lines~\ref{alg:reap:int:ford}-\ref{alg:reap:int:fy:o};
2) block the region weakly dominated by $\Ac$, in
lines~\ref{alg:reap:int:forb}-\ref{alg:reap:int:forb:ufb};
3) search for a $(1+\varepsilon)$-approximation set outside the blocked region, and update
$\Ac$ in lines~\ref{alg:reap:int:mcs}-\ref{alg:reap:int:uA};
4) update (\ie, decrease the value of) $\varepsilon$ in line~\ref{alg:reap:int:ueps}.

Unlike Algorithm~\ref{alg:reap:coeff}, Algorithm~\ref{alg:reap:int} directly encodes the original
objective functions $f_1,\ldots,f_p$ %(with the original coefficients)
and does it just once at the beginning of the algorithm (line~\ref{alg:reap:int:of}).
The corresponding clauses are added to $\hc$ (line~\ref{alg:reap:int:uenc}).
In this case, only the sets $D_1,\ldots,D_p$ will (possibly) change from iteration to iteration
as $\varepsilon$ decreases. Therefore, the first main step of the repeat-until loop
is to update the sets $D_1, \ldots, D_p$ accoding to the current value of $\varepsilon$
(line~\ref{alg:reap:int:D}) and the $y$ variables
associated to the intervals they represent (line~\ref{alg:reap:int:fy:y}). Note that, %in this case,
the second argument of $\computeD$ is $1+\varepsilon$, so that the values in $D_1,\ldots,D_p$
represent the approximation intervals. 
The second main step is to block the region weakly dominated by each new solution $\nu$, stored in $\Ac'$
(lines~\ref{alg:reap:int:forb:forp}-\ref{alg:reap:int:forb:ufb}).
After line~\ref{alg:reap:int:forb:ul}, $\Lc=f(\Ac)$ represents the blocked region.
Firstly, the
procedure $\encodelt$ is called (line~\ref{alg:reap:int:forb:forp:y}) for the value of each
component of $f(\nu)$, stored in $d_k$, because $f_k(\nu)$ is not necessarily in $D_k$ and thus,
$y_{k,d_k}$ may not have been encoded in the previous for loop (in line~\ref{alg:reap:int:fy:y}).
Then, the clause $(y_{1,d_1}\vee\ldots\vee y_{p,d_p})$ is permamently added to $\hc$
(line~\ref{alg:reap:int:forb:ufb})
to block the region weakly dominated by $f(\nu)$ in the remaining iterations.

Similarly to Algorithm~\ref{alg:reap:coeff},
Algorithm~\ref{alg:approx} is executed in line~\ref{alg:reap:int:mcs}
to find a $(1+\varepsilon)$-approximation with respect to the solutions outside the blocked
region, $\Ac'$, and the associated lower bound set, $\Lc'$. Then $\Ac$ and $\Lc$ are update
with $\Ac'$ and $\Lc'$, respectively, in lines~\ref{alg:reap:int:uA}-\ref{alg:reap:int:uL}. 
Hence, $\Ac$ becomes a $(1+\varepsilon)$-approximation and $\Lc$ the associated lower bound set.
The value of $\varepsilon$ is then updated in line~\ref{alg:reap:int:ueps}.
If Algorithm~\ref{alg:reap:coeff} returns an empty set as $\Ac'$, then $\Ac=\Lc$ is the Pareto
front, and the algorithm terminates.
As in  Algorithm~\ref{alg:reap:coeff}, Algorithm~\ref{alg:reap:int} may be interrupted when a
user-defined criterion is met.

\begin{figure}[t]
        \centering
        \begin{subfigure}[b]{0.27\linewidth}
            \centering
            \includegraphics[width=\linewidth]{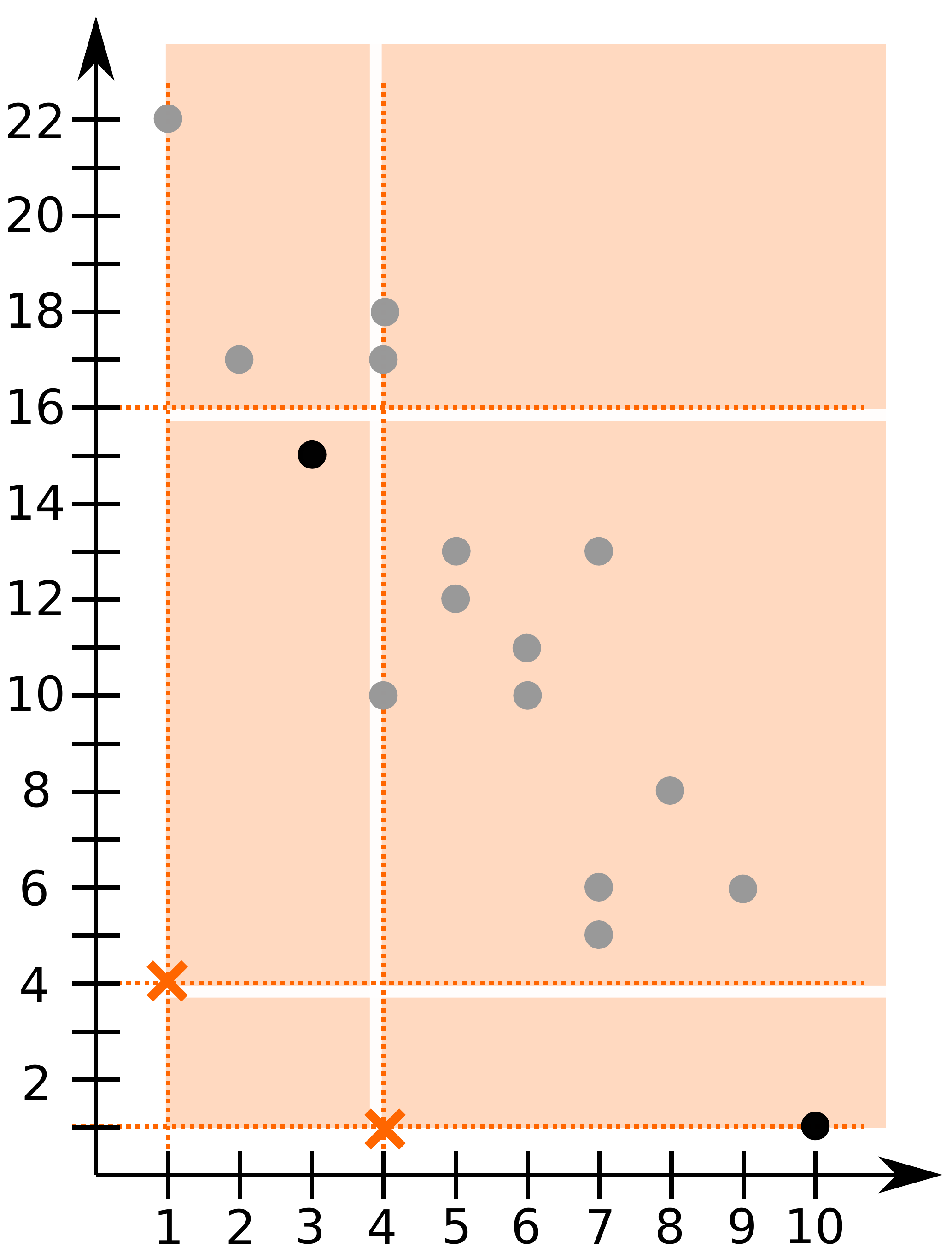}
            \caption[]{$\varepsilon=3$}
            \label{fig:alg:int:eps4}    
        \end{subfigure} %
%         \hfill
        \begin{subfigure}[b]{0.27\linewidth}  
            \centering 
            \includegraphics[width=\linewidth]{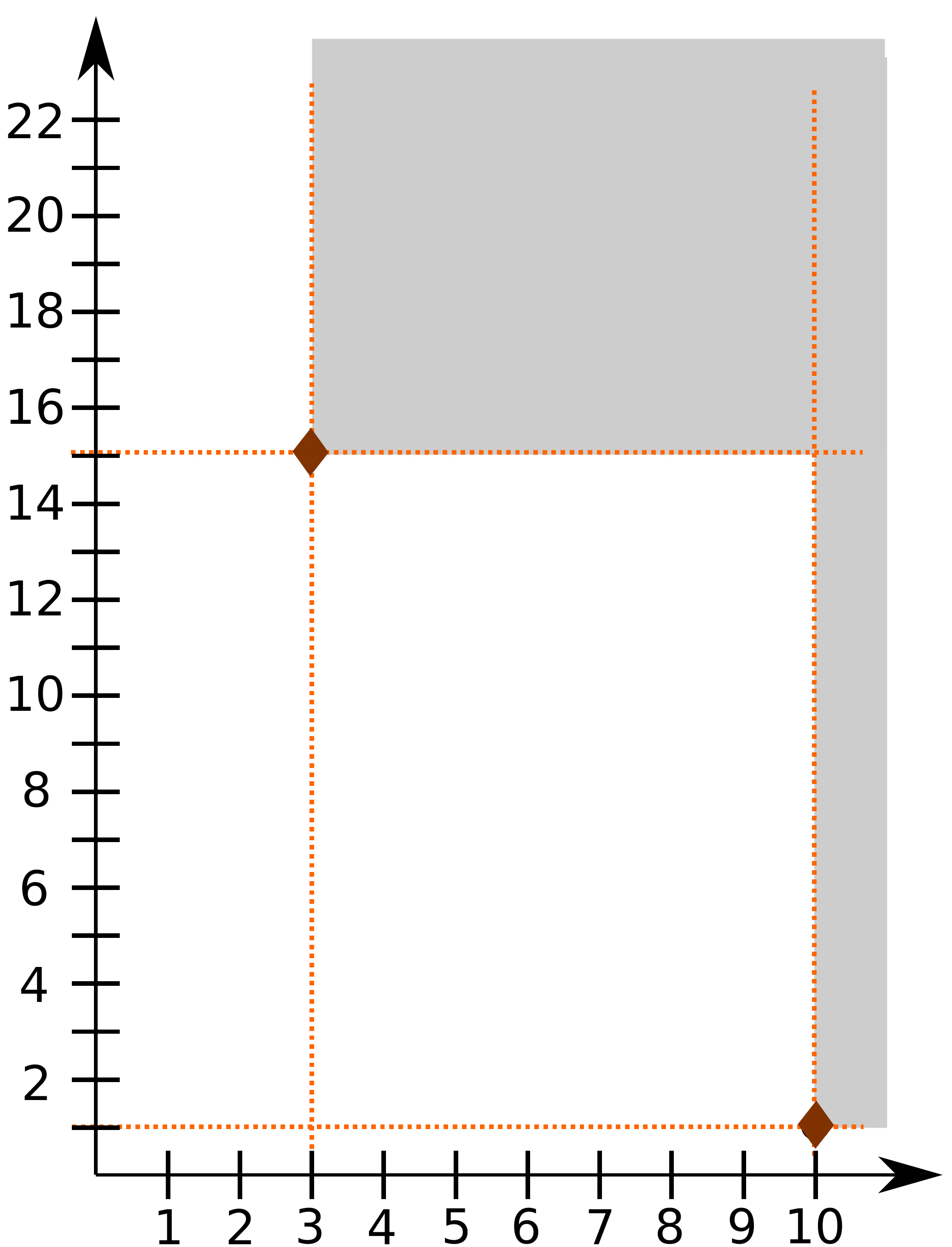}
            \caption[]{blocked region}
            \label{fig:alg:int:eps2b}    
        \end{subfigure}
%         \hfill
        \begin{subfigure}[b]{0.27\linewidth}  
            \centering 
            \includegraphics[width=\linewidth]{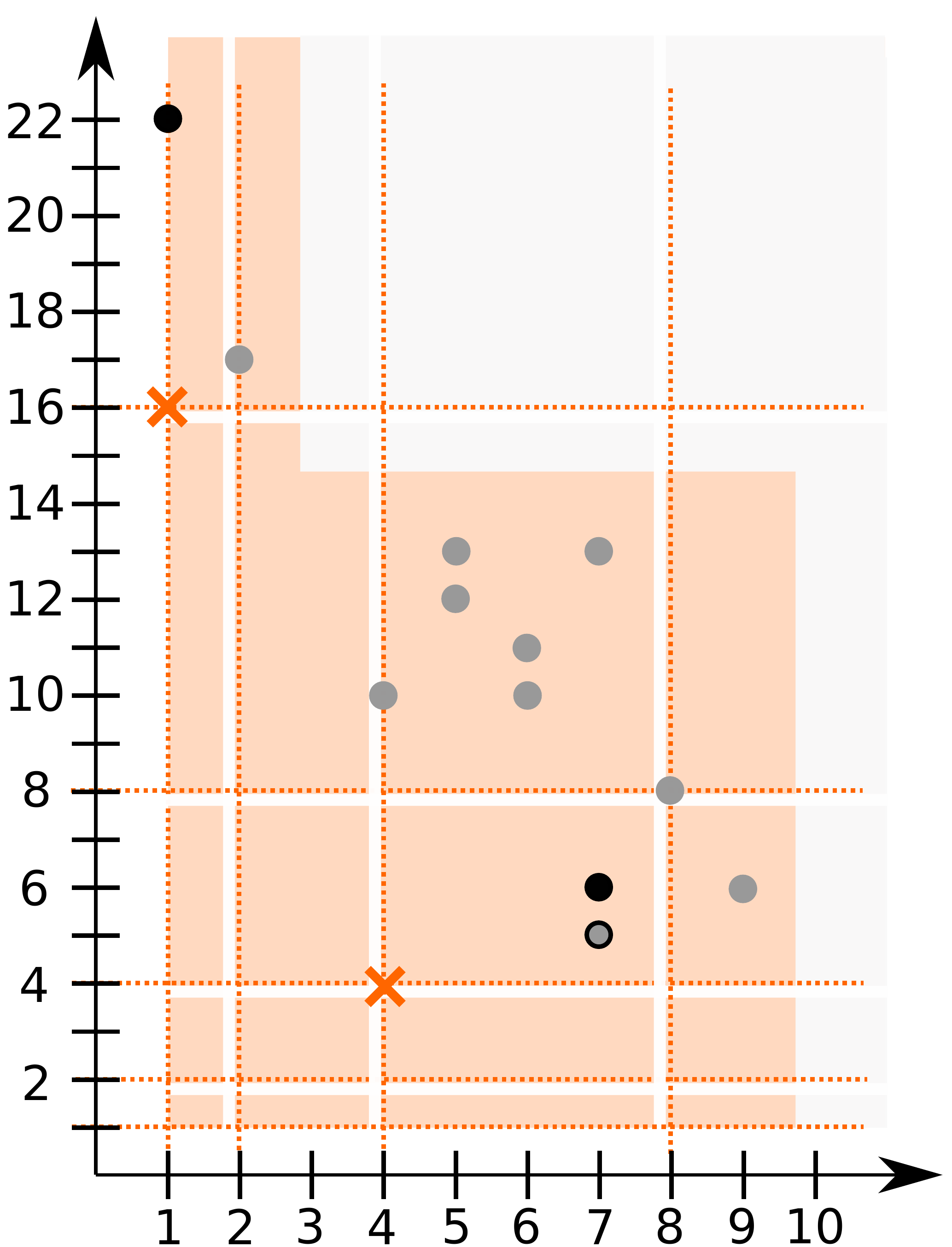}
            \caption[]{$\varepsilon=1$}
            \label{fig:alg:int:eps2}    
        \end{subfigure}
        \caption{Example of the first two iterations of Algorithm~\ref{alg:reap:int}:
    (a) illustration of the problem solved in the first iteration for $\varepsilon=3$, 
    (b) the region blocked in the beginning of the second iteration for $\varepsilon=1$, and
    (c) the problem solved after.} 
        \label{fig:alg:int}
    \end{figure}

Figure~\ref{fig:alg:int} illustrates two iterations of Algorithm~\ref{alg:reap:int} for the
problem in expression~\eqref{prob:ex1}.
With $\varepsilon$ initially set to $3$,
$D_1=\{1,4,16\}$ and $D_2=\{1,4,16,64\}$ in the first iteration (line~\ref{alg:reap:int:D}).
Figure~\ref{fig:alg:int:eps4} 
shows all feasible points (circles), the points returned
in $\Ac'$ in line~\ref{alg:reap:int:mcs} (black circles), \ie, $\{(3,15), (10,1)\}$,
and the points in the respective lower bound set $\Lc'=\{(1,4),(4,1)\}$ (orange crosses).
At the end of the first iteration, the value of $\varepsilon$ is updated to $1$
(line~\ref{alg:reap:int:ueps}).
In the second iteration, after computing $D_1=\{1,2,4,8,16\}$ and $D_2=\{1,2,4,8,16,32\}$, 
the regions weakly dominated by points found in the previous iteration are blocked
in lines~\ref{alg:reap:int:forb}-\ref{alg:reap:int:forb:ufb}. Such regions are shaded
in Figure~\ref{fig:alg:int:eps2b}) where
the dashed lines represent the values $d_1,\ldots,d_p$ computed in
line~\ref{alg:reap:int:forb:forp:d} and used in line~\ref{alg:reap:int:forb:forp:y}
to encode the $y$ variables needed to construct the blocking clauses.
Figure~\ref{fig:alg:int:eps2} shows the remaining feasible points outside the blocked
region, the point set $\{(1,22),(7,6)\}$ returned in $\Ac'$ (black circles) and the respective
lower bound set $\Lc'=\{(1,16), (4,4)\}$ (orange crosses) computed in line~\ref{alg:reap:int:mcs}.
At the end of this iteration, $\Ac=\{(3,15), (10,1), (1,22),(7,6)\}$ is a $(1+1)$-approximation set,
and $\Lc=\{(3,15),(10,1),(1,16),(4,4)\}$ is the associated lower bound set.
Note that, since the region weakly dominated by $\{(3,15), (10,1)\}$ is already blocked,
in the third iteration new clauses will be created to block only the regions weakly
dominated by $(1,22)$ and $(7,6)$.

\subsection{Remarks}

% \begin{whattowrite}
%     \item Discuss the stopping criterion ($\varepsilon_{stop}$, or early interruption) - anytime performance.
%         It provides a lower bound set and an approximation ratio if the first iteration was completed
%     \item Discuss advantages/disadvantages of each approximation method
%     \item Beyond these search strategies and computing $\encodelt$ on-the-fly
%     \item Explain how to enumerate the \emph{efficient set}.
% \end{whattowrite}

% Discuss stopping criterion and early interruption
\noindent 
Algorithm~\ref{alg:approx} finds a $(1+\varepsilon)$-approximation set,
provided that (some of) the values of the objective functions were previously encoded using some unary representation.
The (re-approximation) Algorithms~\ref{alg:reap:coeff} and~\ref{alg:reap:int},
named \rcb and \rib, respectively, are responsible for %doing
such encoding, and iteratively call Algorithm~\ref{alg:approx}.
Both are flexible in the sense that they can
start searching for a $(1+\varepsilon^s)$-approximation set for any $\varepsilon^s \geqslant 0$,
then iteratively refine $\varepsilon^s$, and stop when a $(1+\varepsilon^d)$-approximation set is found,
given a desired value of $\varepsilon^d$. If $\varepsilon^d=\varepsilon^s$, then the algorithm
stops after just one call to Algorithm~\ref{alg:approx}.
\rcb and \rib may also be terminated when a given time budget is reached. %, in which case they
% return the set of the best solutions found so far, and a lower bound set (if one is known).
In any case, they return in $\Ac$ the best solutions found so far, and as long as the first iteration
terminates within the time budget, the
algorithms provide %, along with an approximation set, $\Ac$,
a lower bound set, $\Lc$. 
This delimits the location of the Pareto front, and provides an approximation factor
of $I_\varepsilon(\Ac,\Lc)$ which is possibly tighter than $(1+\varepsilon^c)$,
where $\varepsilon^c$ is the value of $\varepsilon$ in the last completed iteration.
Hence, these algorithms are useful to a Decision Maker interested in
learning about the location of the Pareto front, or in
finding the best approximation possible, as soon as possible or, within a time limit.

% Discuss advantages (motivation) for the reapproximation version
% Anytime performance
% Increasingly reduce the region containing the Pareto front (useful from a DM point-of-view - learning)
%anytime performance

The anytime performance of the algorithms, \ie, the quality of the obtained approximation set
at any time during their execution, depends on multiple factors such as the problem
instance, the approximation method, the starting $\varepsilon$ value and how it is updated,
how many times Algorithm~\ref{alg:approx} and the SAT solver are called, and so on.
For example, a small initial value for $\varepsilon$ may lead to a good approximation set at
the end of the first iteration, but finding it may not be possible within the time limit.
On the other hand, a larger initial value for $\varepsilon$ will possibly lead to a worse first
approximation, but which is found faster, allowing the algorithm to progressively find better
approximation sets.

The runtime of Algorithm~\ref{alg:approx} is also
influenced by the size of the approximation set found.
\cite{PapYan2000} showed that, for a fixed number of objectives, $p$, there is
a $(1+\varepsilon)$-approximation set of polynomial size in $n$ and $\frac{1}{\varepsilon}$ for
any $\varepsilon > 0$. This is the case even with different $\varepsilon_k$ values for each objective function
$k\in\{1,\ldots,p\}$ and with the $i$-th one set to $\varepsilon_i=0$~\citep{DBLP:journals/jgo/HerzelBRTV21}. The proposed algorithms could be easily adapted to consider this case,
even with different update strategies for each $\varepsilon_k$. 
This indicates that the number of MCSs enumerated by Algorithm~\ref{alg:approx} may be polynomial,
even though $Y_N$ may have exponential size.
In fact, the interval-based approximation ensures such polynomial-size approximation sets because
it relies on the same hypergrid partitioning of the objective space as in the proofs of the
aforementioned articles.
Hence, with an initial $\varepsilon > 0$, \rib (and possibly \rcb for some instances)
can first search for an approximation set of polynomial size, giving an idea of the location and extent of the Pareto front,
and then use the time left to enumerate as many nondominated points as possible
(by updating $\varepsilon$ to zero). The lower bound set could also be updated along the run
to tighten the known approximation factor.
In contrast, with an initial $\varepsilon=0$, the algorithms may only be able to return
a (possibly not well distributed) subset of the Pareto front with no known
approximation factor when $|Y_N|$ is large (\eg, exponential in $n$).

% Discuss the two approximation methods (show some results)
% Results: mention the objectives' ranges
There are a few differences distinguishing the re-approximation algorithms that stem from
the approximation method that each one uses.
With the interval-based approximation, the set of clauses that encode the objective functions
is the same independently of the value of $\varepsilon$.
One advantage of this is that the objective functions are encoded just once in \rib and afterwards,
the $y_{k,d}$ variables (\ie, the unary representation of the objective values) 
% \vmm{Not sure the reader easily relates to what is a $y$ variable...\agt{AG: rephrased it}}
are encoded as needed, depending on $\varepsilon$. However, this implies that
a large encoding may be needed even if $\varepsilon$ is large.
Another advantage of \rib is that as new solutions are found, the region weakly dominated by those
points can be permanently blocked.

A disadvantage of the coefficient-based approximation is that, since the
approximate objective functions may change from iteration to iteration, they have to be
re-encoded at the beginning of each iteration, and new clauses have to be added to block the region dominated by the points in $\Ac$ using the new $y_{k,d}$ variables associated
with the newest encodings. Therefore, as $\varepsilon$ tends to zero, the total amount of time spent
in these re-encoding and blocking steps may be more time consuming in \rcb than in \rib. 
An advantage of the coefficient-based method is that the encoding
of the approximate objective functions is potentially much smaller than for $\varepsilon=0$
(and than in \rib).
The size of the encoding depends on three factors: the method used to obtain the unary
encoding, the
coefficients of the function being encoded, and the function upper bound. 
Although the influence of the coefficients may not be easy to
quantify, the smaller they are the smaller is the function upper bound, and the more equal coefficients
in the function the smaller is the number of different values the function can take
(\eg, if they are all equal, then the function can only take up to $n$ different values).
Hence, in general, the larger $\varepsilon$ is, the smaller the encoding is expected to be.

Finally, note that, for some problem instances, running \rcb for any setting of $\varepsilon$
is equivalent to solving the exact case, $\varepsilon=0$, and therefore, the
approximation set returned could be of exponential size. For example, when
all the coefficients of the objective functions are equal, 
or are powers of $(1+\varepsilon)$. In such cases, \rib may be preferred. 
It is also important to note that if $c^{max}$ is the largest coefficient in the
objective functions, then running \rcb with $(1+\varepsilon)>c^{max}$ is equivalent to
running it with $(1+\varepsilon)=c^{max}$ because in either case, the non-zero coefficients of
the approximate functions will all be one.
Additionally, if the $p$ approximate functions are all equal, then the corresponding
approximation sets will have size one.

% Encode \encodelt on-the-fly -- future work?

%\breakpage

\section{Computational experiments}
\label{s:exp}

\noindent This section evaluates the performance of the algorithms proposed in Section~\ref{s:algs}.
In particular, we show some preliminary results on the anytime performance
of \rib and \rcb (see Section~\ref{s:exp:anyt}) to highlight their advantages
and potential,
and then we compare them with state-of-the-art solvers (see Section~\ref{s:exp:comp}).

\subsection{Experimental Setup}
\label{s:exp:setup}
% \begin{whattowrite}
%     \item Machines
%     \item Data sets
%     \item Memory and time limits
%     \item What information is recorded
% \end{whattowrite}

\noindent All experimental results were obtained on a server with processor Intel(R)
Xeon(R) CPU E5-2630 v2 @ 2.60GHz with 64GB of memory. 
The \rib and \rcb algorithms were implemented in C++ by extending Open-WBO~\citep{martins-sat14} and
adapting it for the multi-objective case. Glucose SAT solver (version 4.1) was used
in \mcskp and the CLD algorithm~\citep{DBLP:conf/ijcai/Marques-SilvaHJPB13} was used
to compute MCSs.
In the implementations we have some preprocessing steps to reduce the size of the
encodings of the objective functions. For example, if, 
at some point during the execution,
the SAT solver infers the value of some variable, \ie, that the variable holds that
value for any feasible solution, then this variable can be excluded when (re-)encoding
the objective functions. Another example is the computation of an upper bound on the feasible
values of each objective function, tighter than $u_k$, and encoding the function only
up to that value.

The algorithms \rib and \rcb were tested with different input parameters and different
update strategies of the approximation factor.
In the plots shown in this section, the name of the algorithms is followed
by a tuple indicating such settings, which may include one or two components.
The first component corresponds to the value of $(1+\varepsilon)$, and the second,
if present, indicates the value by which $\varepsilon$ is divided in each iteration.
For example, $\rib (2, 10)$ corresponds to running \rib with a starting value
of $\varepsilon=1$ (\ie, $2=1+\varepsilon$), and in the second iteration $\varepsilon$
will be set to $1/10$, in the third it will be set to $0.1/10$, and so on.
Another example is $\rcb (4)$ which corresponds to running
\rcb with $\varepsilon=3$, and terminating the algorithm after the first iteration.
Since \rib and \rcb are equivalent when $\varepsilon$ is initially set to $0$, which
corresponds to encoding the objective functions in an exact way and call
Algorithm~\ref{alg:approx} just once to enumerate all nondominated points,
this particular case will be referred as $\mcskp(1)$.

The following MOBO problems were considered: the Multiobjective Set Covering Problem
(MSCP) \citep{DBLP:conf/cp/BergmanC16,DBLP:conf/cp/SohBTB17}, and
the Multiobjective Development Assurance Problem (MDAP)~\citep{DBLP:conf/safecomp/BieberDS11}.
MSCP is a generalization of the classical set covering problem that consists of
deciding which subsets of a ground set $X$, and pre-specified in a set $A$, to select such
that the total cost associated to selecting the subsets is minimized given that all
elements of $X$ are covered (\ie, each element is in at least one of the selected
subsets).
MDAP encodes different levels of rigor of the development of a software or
hardware component of an aircraft. The Development Assurance Level (DAL) defines the
assurance activities aimed at eliminating design and coding errors that could affect the
safety of an aircraft. The goal is to allocate the smallest DAL to functions in order to
decrease the development costs~\footnote{The benchmark instances are available online at
\url{https://www.lifl.fr/LION9/challenge.php}. Although they define a lexicographic order
to the objective functions, in the context of this paper, we ignore this order and compute
the Pareto frontier.}.

\subsection{Anytime Performance Example}
\label{s:exp:anyt}
% \begin{whattowrite}
%     \item explain plots
%     \item discuss
% \end{whattowrite}
\noindent 
To show the potential of \rib and \rcb, we show some results on the anytime performance of 
these algorithms on a $2$-objective instance of MSCP, named ``2scp43A"
\footnote{Instance from ~\url{https://github.com/vOptSolver/vOptLib/blob/master/SCP/scp1998.md}},
for which $n=200$ and $m=40$, and where $l_1=0$, $u_1=21075$, $l_2=0$, and $u_2=21366$.
We compare, against the exact version (\ie, $\varepsilon=0$),
the performance of \rib and \rcb with different initial value of $\varepsilon\in\{1,100\}$,
which is divided by $10$ at each iteration until $\varepsilon < 1e^{-4}$, in which case
$\varepsilon$ is set to $0$. The algorithm runs until it finishes enumerating the Pareto
front, or the algorithm achieves the time limit of $3600$ seconds or the memory limit of 8GB.
For Figure~\ref{f:sc:apC}, we recorded the following information:
each new feasible solution found and the time when it was found, the time when a new
iteration of \rib/\rcb started, and to which value $\varepsilon$ was set.
Figures~\ref{f:sc:apC:hv} shows the quality of the set of solutions found along the run
according to the hypervolume indicator~\citep{ZitzlerT98}. The indicator was computed with the reference point
set to $r=(1,1)$ and after normalizing the objective values of each point by dividing them by
the maximum of the respective objective function plus one, \ie, by 21076 and 21367,
in the case of the  first and the second objectives, respectively.
Figure~\ref{f:sc:apC:eps}
shows the value of $(1+\varepsilon)$ along the run.

\begin{figure}[t]
        \centering
        \begin{subfigure}[b]{0.3\textwidth}
            \centering
            \includegraphics[width=\textwidth]{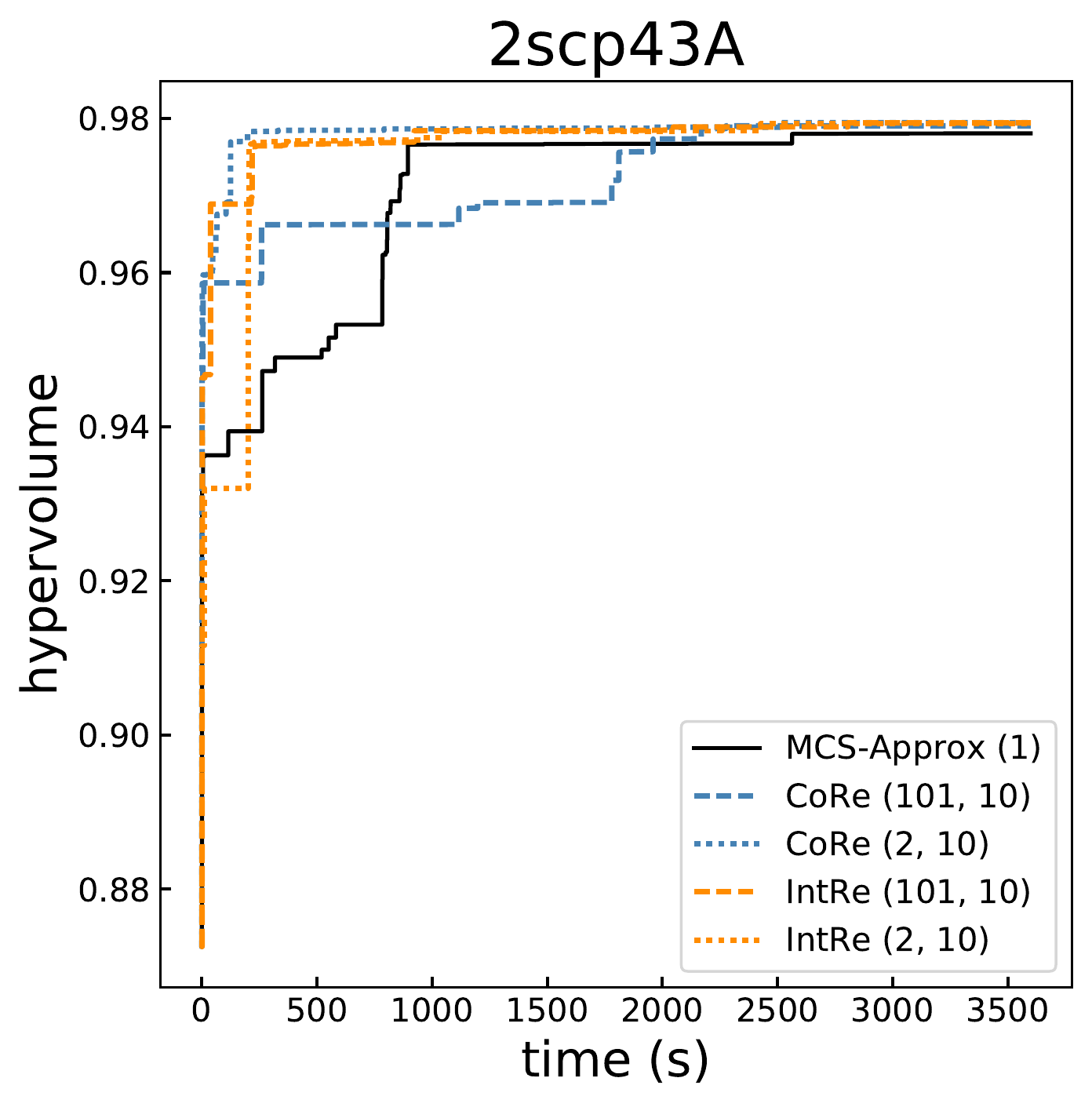}
            \caption[]{}
            \label{f:sc:apC:hv}
        \end{subfigure}
%         \hfill
        \begin{subfigure}[b]{0.3\textwidth}  
            \centering 
            \includegraphics[width=\textwidth]{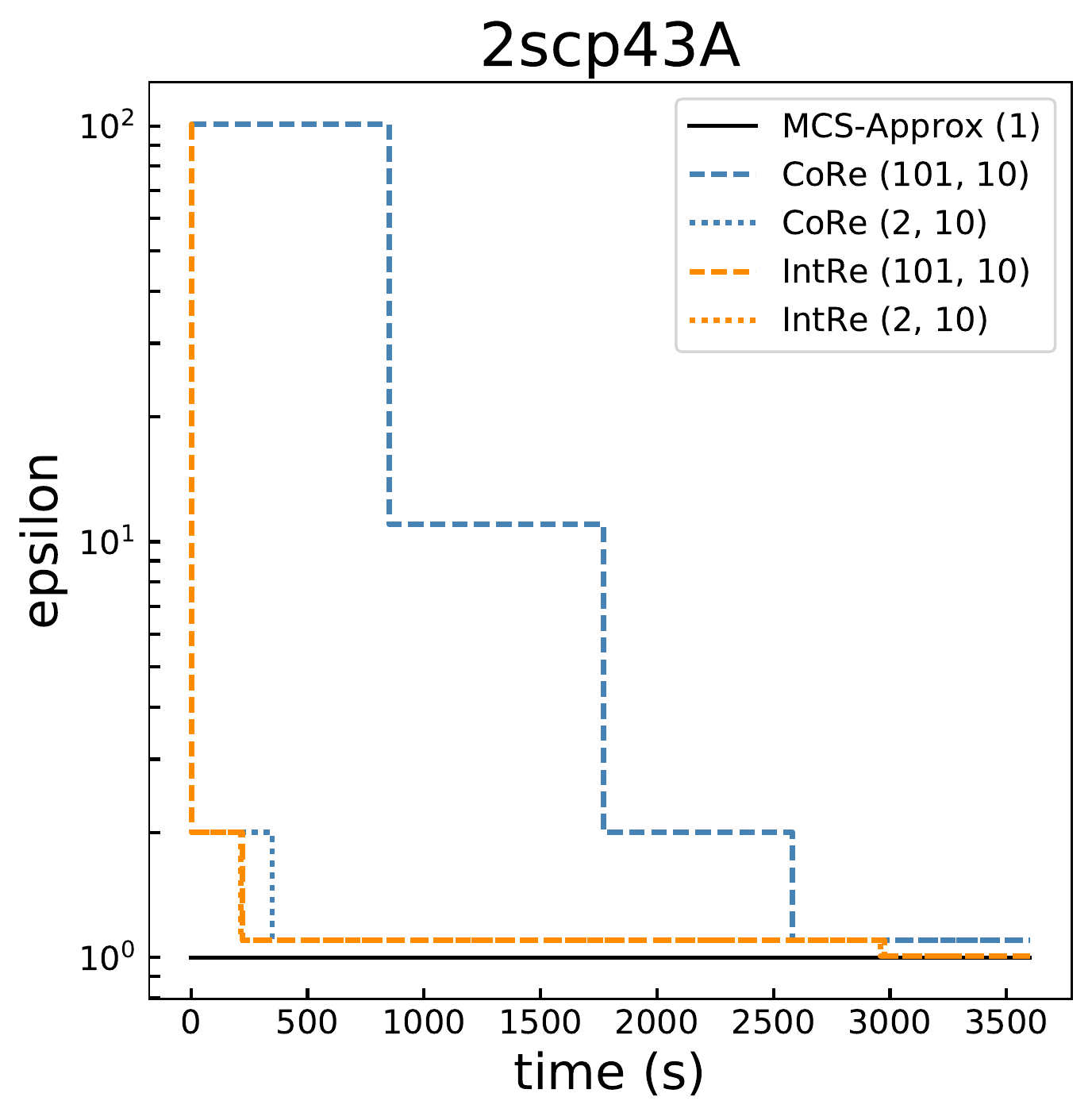}
            \caption[]{}    
            \label{f:sc:apC:eps}
        \end{subfigure}
        \caption{Performance of \rib and \rcb algorithms along the run, at time $t\in[0,3600]$ seconds, on a given MCSP instance:
        \subref{f:sc:apC:hv} the hypervolume indicator of the set of solutions known;
        \subref{f:sc:apC:eps} the value of the current setting of $(1+\varepsilon)$.
        }
        \label{f:sc:apC}
    \end{figure}

Figures~\ref{f:sc:apC:hv} shows that, in this instance, the algorithms that first search
for an approximation of the Pareto front can find better sets of solutions sooner than
the algorithm enumerating all nondominated points from the start. 
For example, the quality of the set of solutions found by $\rib(101,10)$ and 
$\rcb(2,10)$ were always better than $\mcskp(1)$ for any time $t\in[0,3600]$.
Figure~\ref{f:sc:apC:eps} shows that, for \rib, the value of $\varepsilon$ rapidly
decreased to $1$, and therefore, somewhat justifies the high hypervolume values
of the sets of solutions found by $\rib$ early in the run.

Figure~\ref{f:sc:os} 
shows in more detail the sets found by each call of
Algorithm~\ref{alg:approx} within each algorithm tested.
The symbols identify the iteration and the setting of $1+\varepsilon$ for that iteration.
The opaque symbols represent the points found in that iteration, and the faded symbols
represent the lower bound set at the end of the iteration.
For example, in Figure~\ref{f:sc:os:rib:2}, $\rib(2,10)$
found the four blue circle points in the first iteration for
which $1+\varepsilon$ was 2, it found the green triangles filled white points in the
second iteration where $1+\varepsilon=1.1$, and found the yellow-star points in the third
iteration for which $1+\varepsilon=1.01$. The three faded blue circles represent the lower
bound set at the end of the first iteration, and the faded green triangles represent
the lower bound set at the end of the second iteration.
Note that the $(1+\varepsilon)$-approximation set at the end of the $i$-th iteration
is the set of points found in the first $i$ iterations (excluding the dominated ones).
Hence, at the end of the second iteration of $\rib(2,10)$, the set $\Ac$ contains the
points represented by triangles, plus the points represented by circles and which are
not dominated by the triangles.

Figure~\ref{f:sc:os:mcskp} illustrates the nondominated points found
by $\mcskp(1)$. Note that the algorithm did not finish, and therefore, the Pareto front
may be incomplete. The other plots show, from iteration to iteration, how both \rib and \rcb
narrow down the region where the Pareto front is located, and show that, even by setting
$\varepsilon$ to large values, allows to exclude a large portion of the objective space.
Note that, in this instance, $\rcb(101,10)$ found points closer to the Pareto front than
$\rib(101,10)$ did when considering the same iteration. However, since more points were
needed to guarantee the approximation factor by $\rcb$, it took more time to achieve the
same hypervolume values that $\rib$ achieved.

\begin{figure}[t]
        \centering
        \begin{subfigure}[b]{0.3\textwidth}  
            \centering 
            \includegraphics[width=\textwidth]{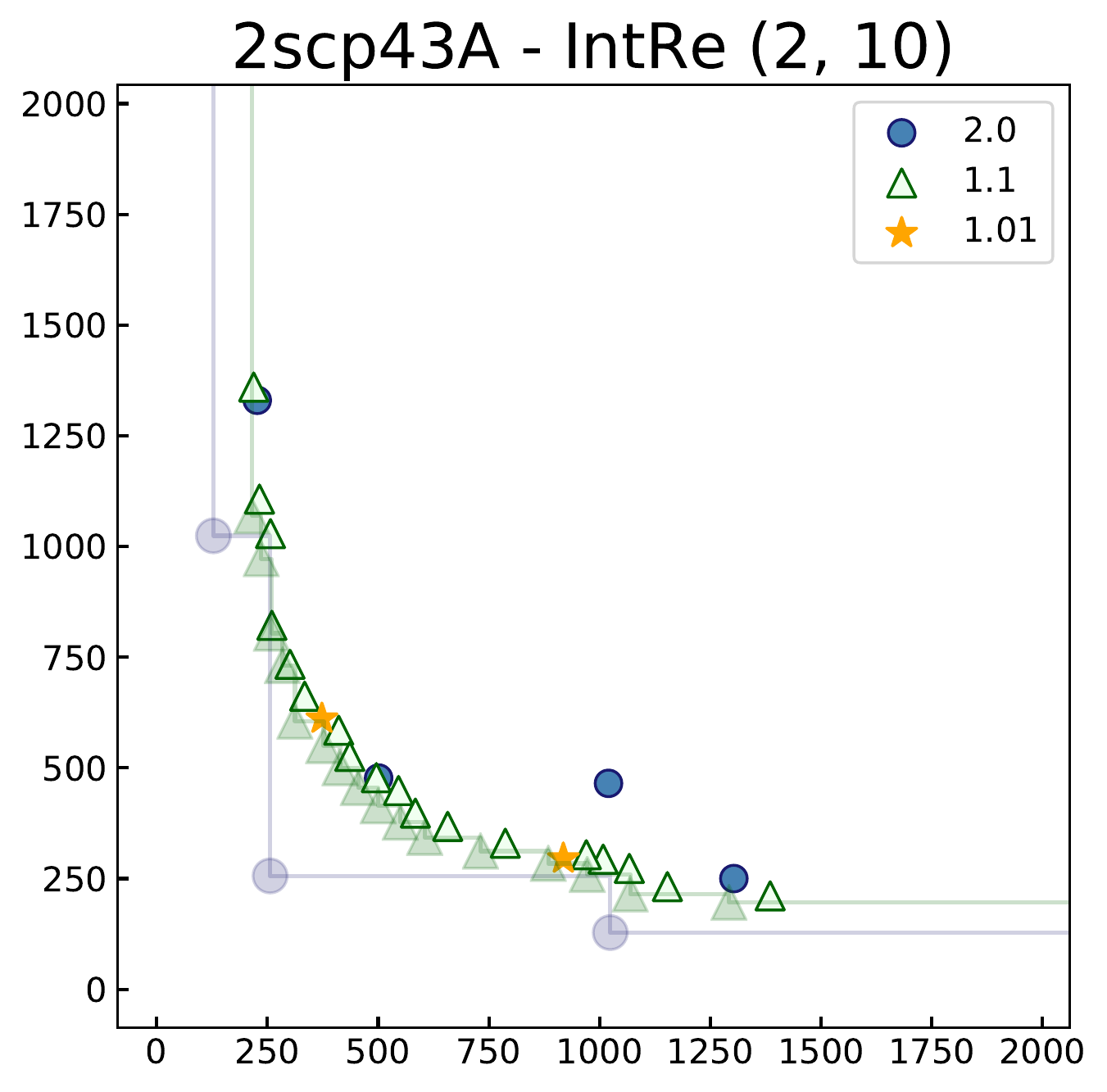}
            \caption[]{}    
            \label{f:sc:os:rib:2}
        \end{subfigure}
%         \hfill
        \begin{subfigure}[b]{0.3\textwidth}  
            \centering 
            \includegraphics[width=\textwidth]{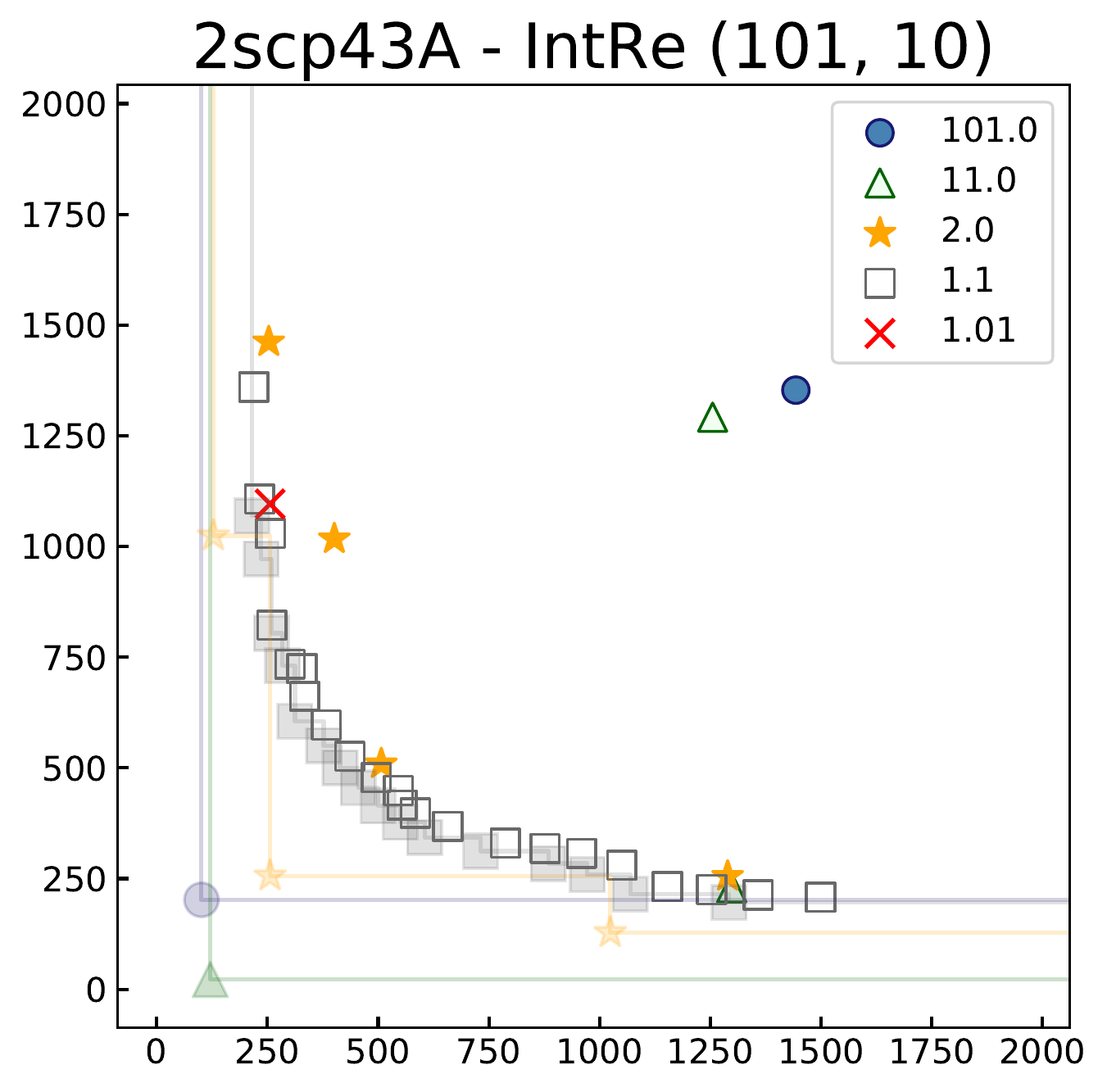}
            \caption[]{}
            \label{f:sc:os:rib:101}
        \end{subfigure}
%         \hfill
        \begin{subfigure}[b]{0.3\textwidth}
            \centering
            \includegraphics[width=\textwidth]{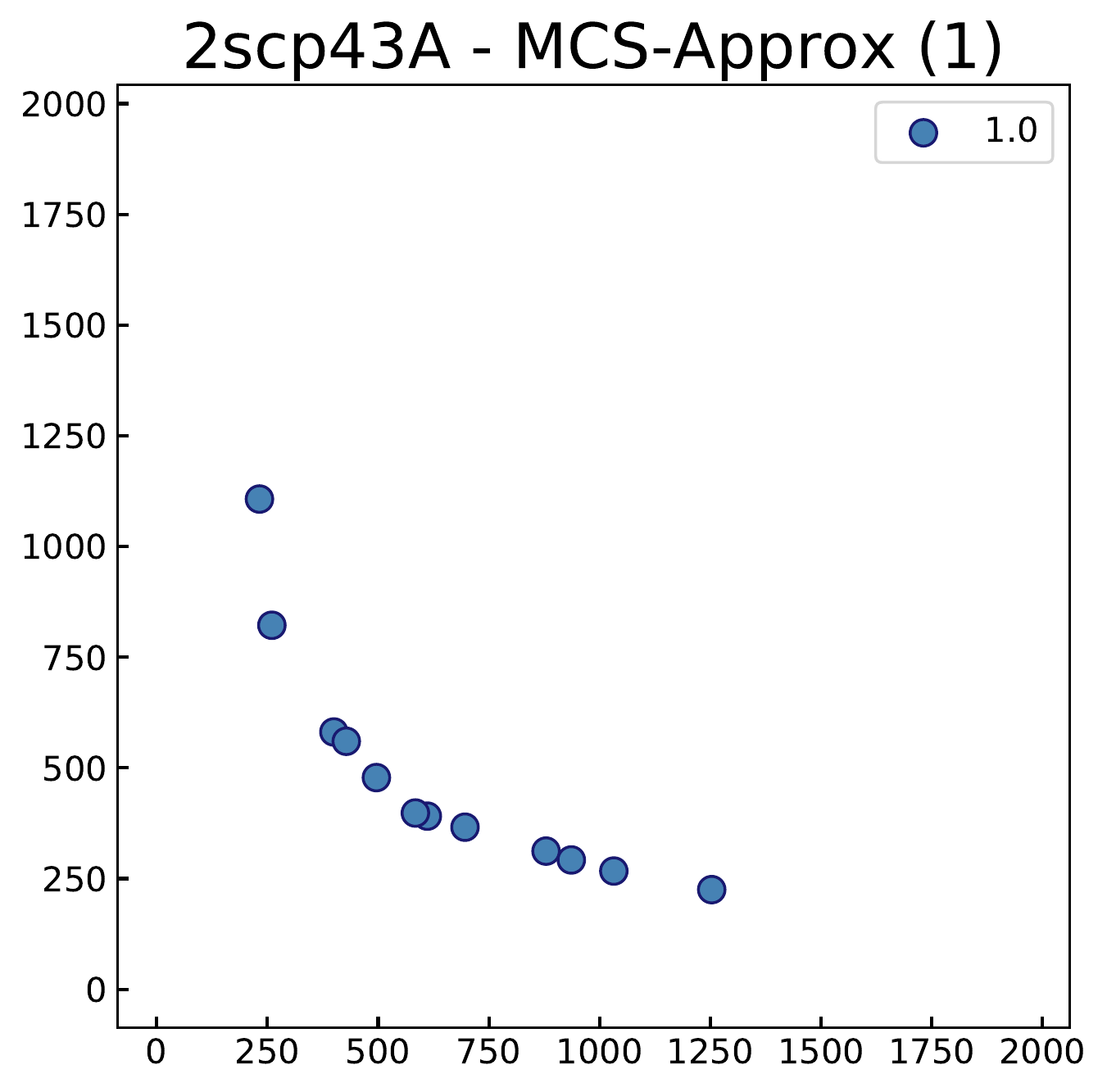}
            \caption[]{}    
            \label{f:sc:os:mcskp}
        \end{subfigure}
        \\
        \begin{subfigure}[b]{0.3\textwidth}  
            \centering 
            \includegraphics[width=\textwidth]{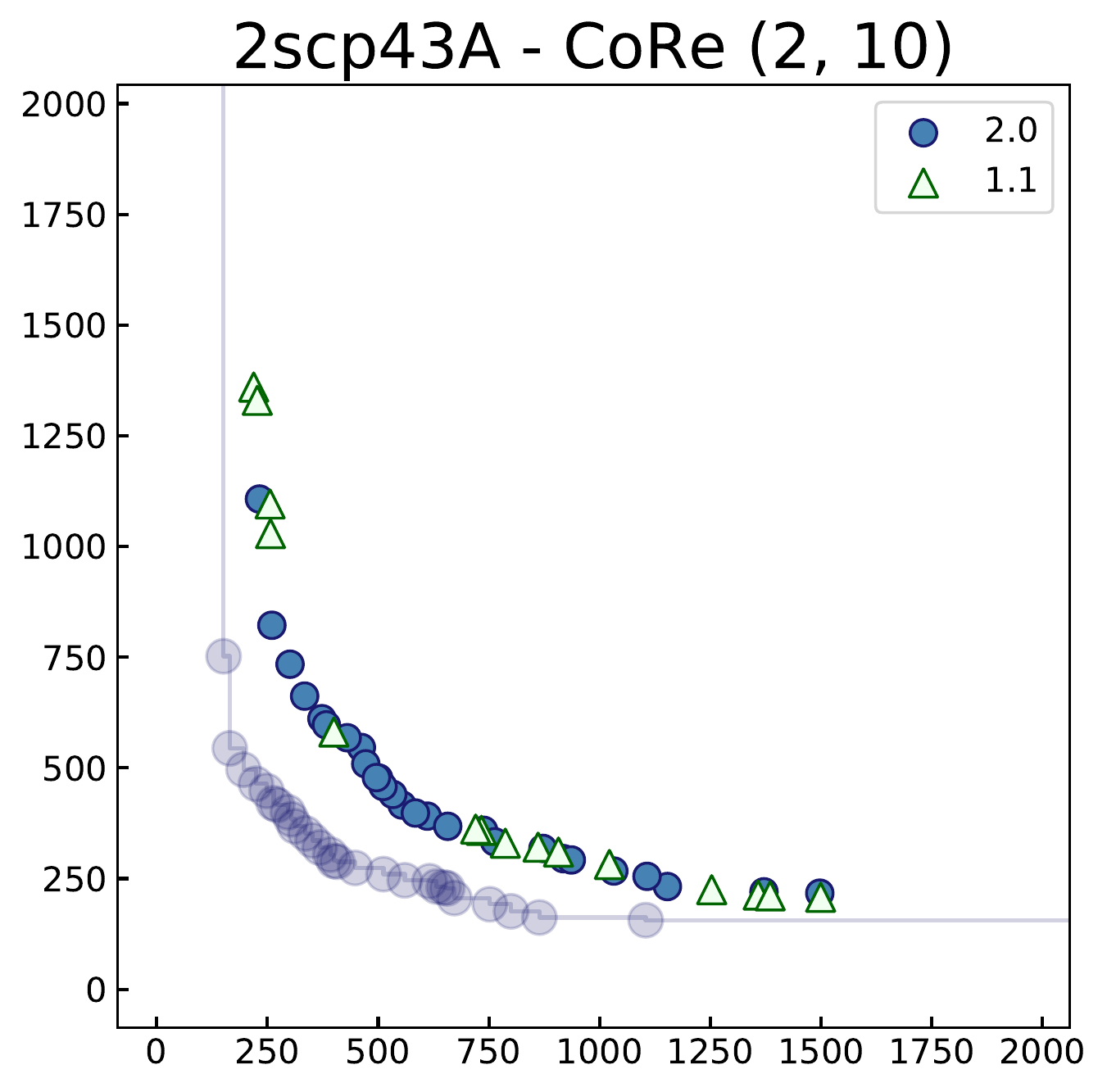}
            \caption[]{}
            \label{f:sc:os:rcb:2}
        \end{subfigure}
        \begin{subfigure}[b]{0.3\textwidth}  
            \centering 
            \includegraphics[width=\textwidth]{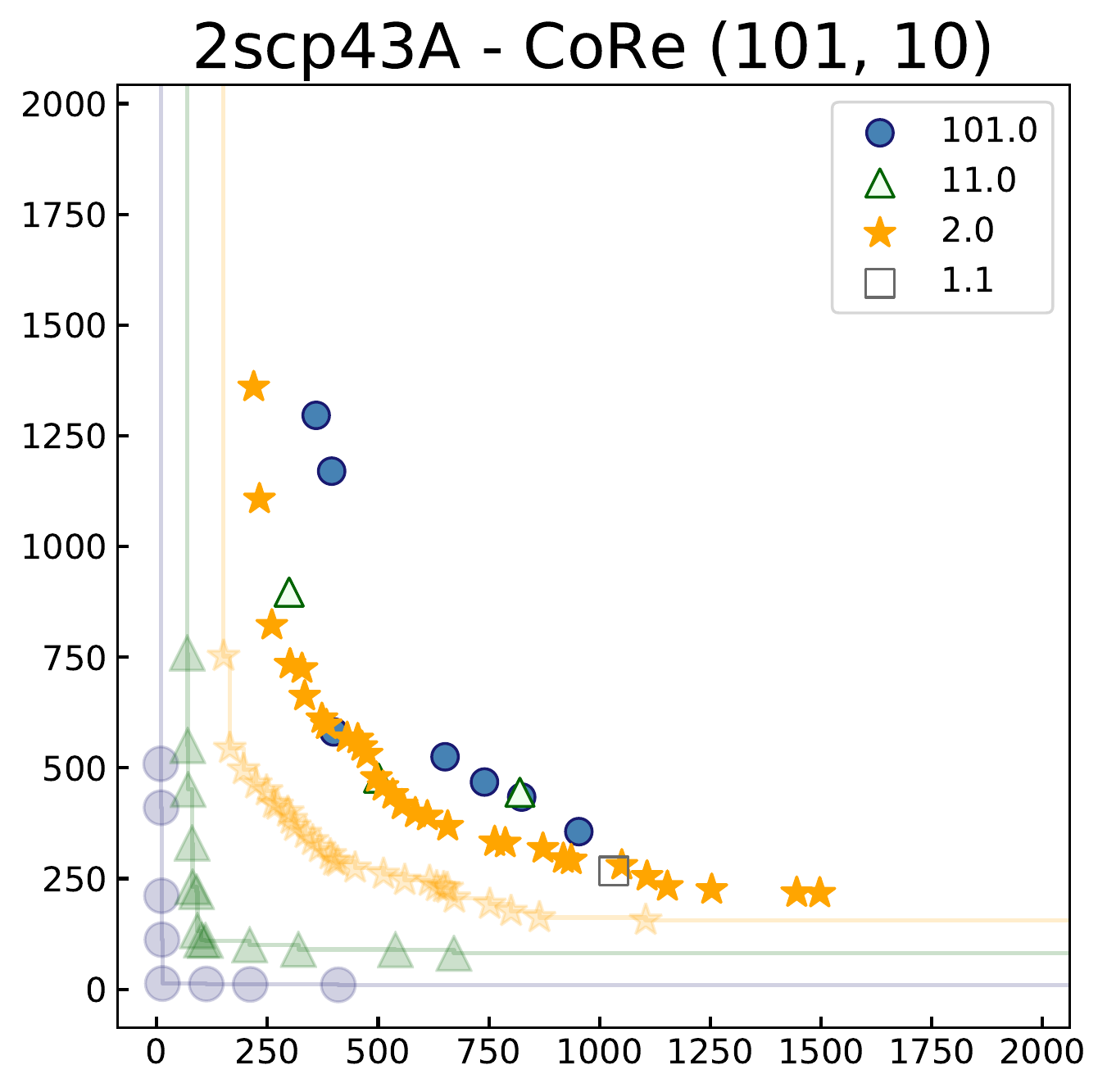}
            \caption[]{}    
            \label{f:sc:os:rcb:101}
        \end{subfigure}
        \hspace{0.3\textwidth}  
        \caption{Illustration of the points found by each algorithm at each iteration.
        The name of the algorithm is in the title of the plot, and the legend indicates
        the symbol used to represent the points found in each iteration,
        and the setting of $(1+\varepsilon)$ on that iteration.} 
        \label{f:sc:os}
    \end{figure}

These results illustrate the potential of \rib and \rcb. In particular, they show that
both \rib and \rcb can outperform algorithms that focus on finding points from the
Pareto frontier from the start, with respect to the quality of the set, and at any 
point in time.
However, the observed behavior might not generalize for other types of problems 
or problem instances. This is discussed in more detail in the next section.

\subsection{Algorithm Comparison}
\label{s:exp:comp}

\noindent The proposed algorithms are evaluated next on multiple instances of MSCP
(see Section~\ref{s:exp:comp:mscp}) and of MDAP (see Section~\ref{s:exp:comp:mdap}).
The size of the formulas needed in each of the proposed algorithms with different settings
are compared, and the quality of the approximation sets computed by them within a time budget is
evaluated and compared against those obtained with state-of-the art algorithms.

\subsubsection{Multiobjective Set Covering Problem}
\label{s:exp:comp:mscp}

\noindent For MSCP,  we used the instances~\footnote{Instances
at~\url{http://www.andrew.cmu.edu/user/vanhoeve/mdd/code/multiobjective\_cp2016.tar.gz}}
by~\cite{DBLP:conf/cp/BergmanC16}.
An instance was randomly selected for each triple $(n,m,p)$ of $n\in\{100,150\}$ variables, 
$m\in\{n-20,n-30,\ldots,n-90\}$ constraints,
and $p\in\{3,4\}$ objectives for a total of 48 instances.
In these instances, each constraint involves 5 non-zero coefficients,
and each coefficient in the first objective function is $1$, and
is an integer between $1$ and $100$ in the remaining objective functions.
The \rib and \rcb algorithms were tested in these instances with different settings
(different starting value of $\varepsilon\in\{1.1,2,11,101\}$) and
the most relevant to the discussion are shown in Figures~\ref{f:sc:comp:hv}
and~\ref{f:sc:comp:eps}. These algorithms are compared against a state-of-the-art algorithm
based on binary decision diagrams~\citep{DBLP:conf/cp/BergmanC16}
which is here referred to as \bdd.

\begin{figure}[t]
        \centering
        \begin{subfigure}[b]{0.4\textwidth}
            \centering
            \includegraphics[width=\textwidth]{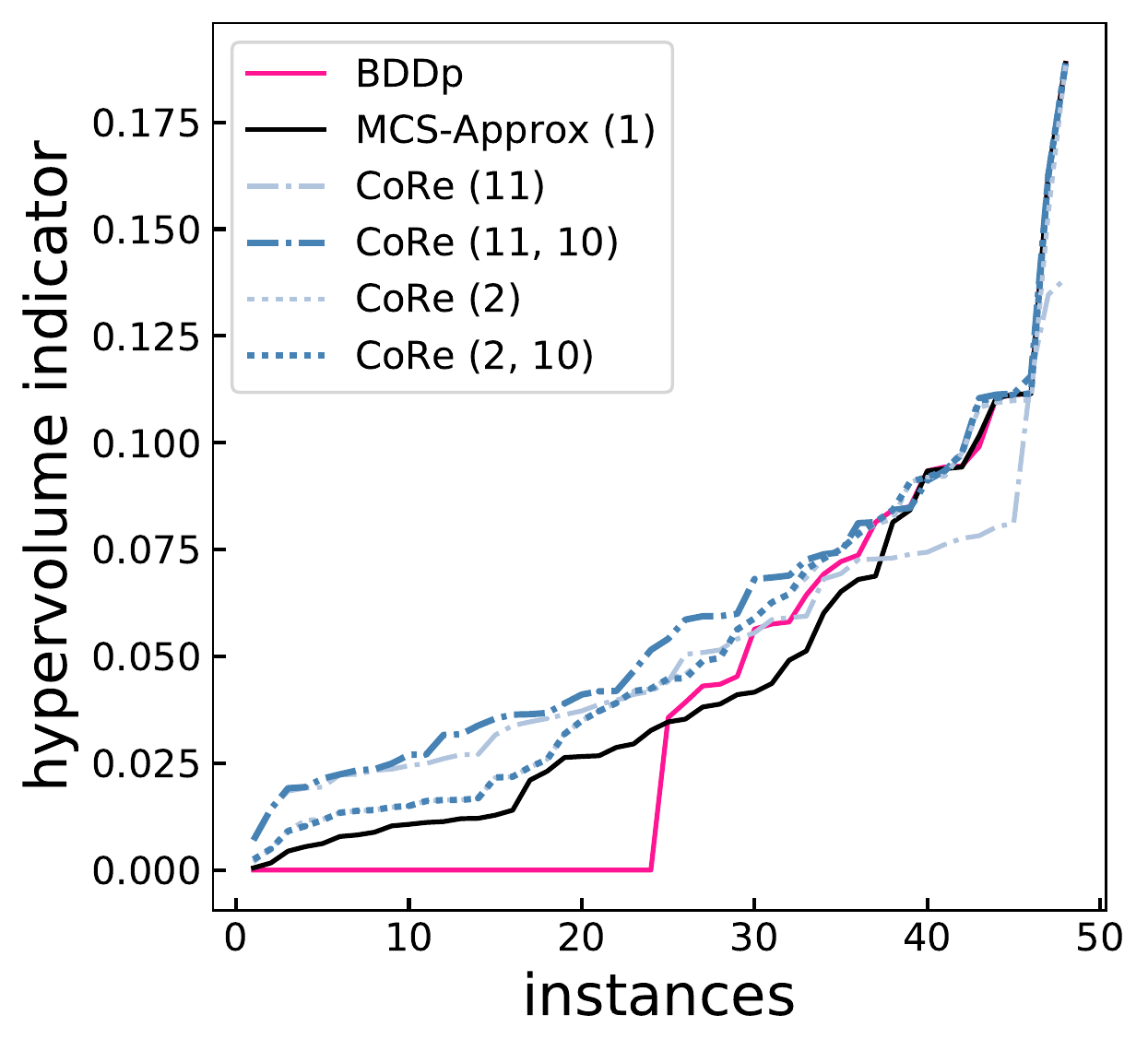}
            \caption[]{}    
            \label{f:sc:comp:hv:rcb}
        \end{subfigure}
        \begin{subfigure}[b]{0.4\textwidth}
            \centering
            \includegraphics[width=\textwidth]{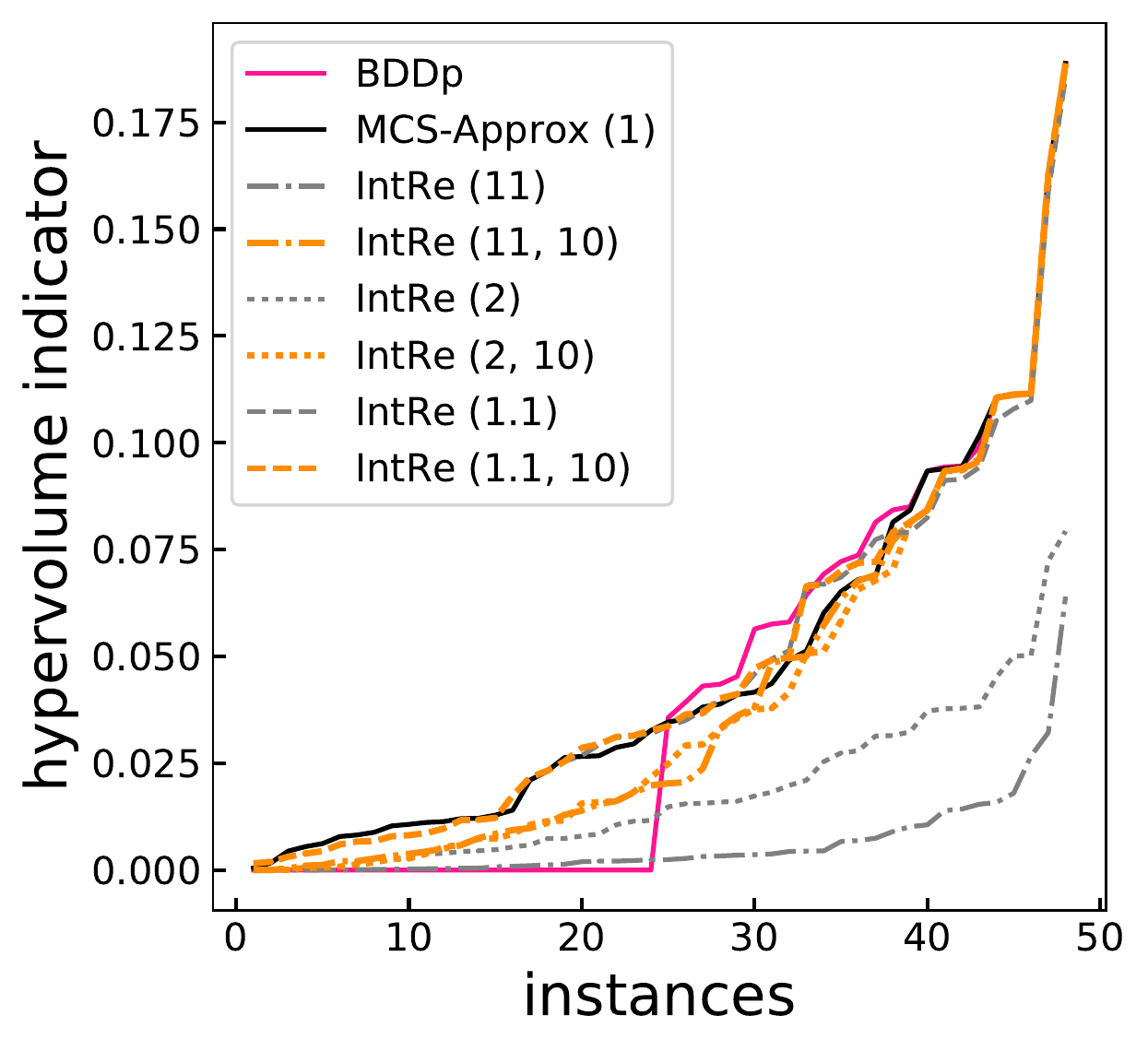}
            \caption[]{}    
            \label{f:sc:comp:hv:rib}
        \end{subfigure}
        \caption{Hypervolume indicator of the output sets obtained with: \subref{f:sc:comp:hv:rcb} $\rcb$ algorithms; and \subref{f:sc:comp:hv:rib} \rib
        algorithms on MSCP instances.} 
        \label{f:sc:comp:hv}
    \end{figure}

\begin{figure}[t]
        \centering
        \begin{subfigure}[b]{0.3\textwidth}  
            \centering 
            \includegraphics[width=\textwidth]{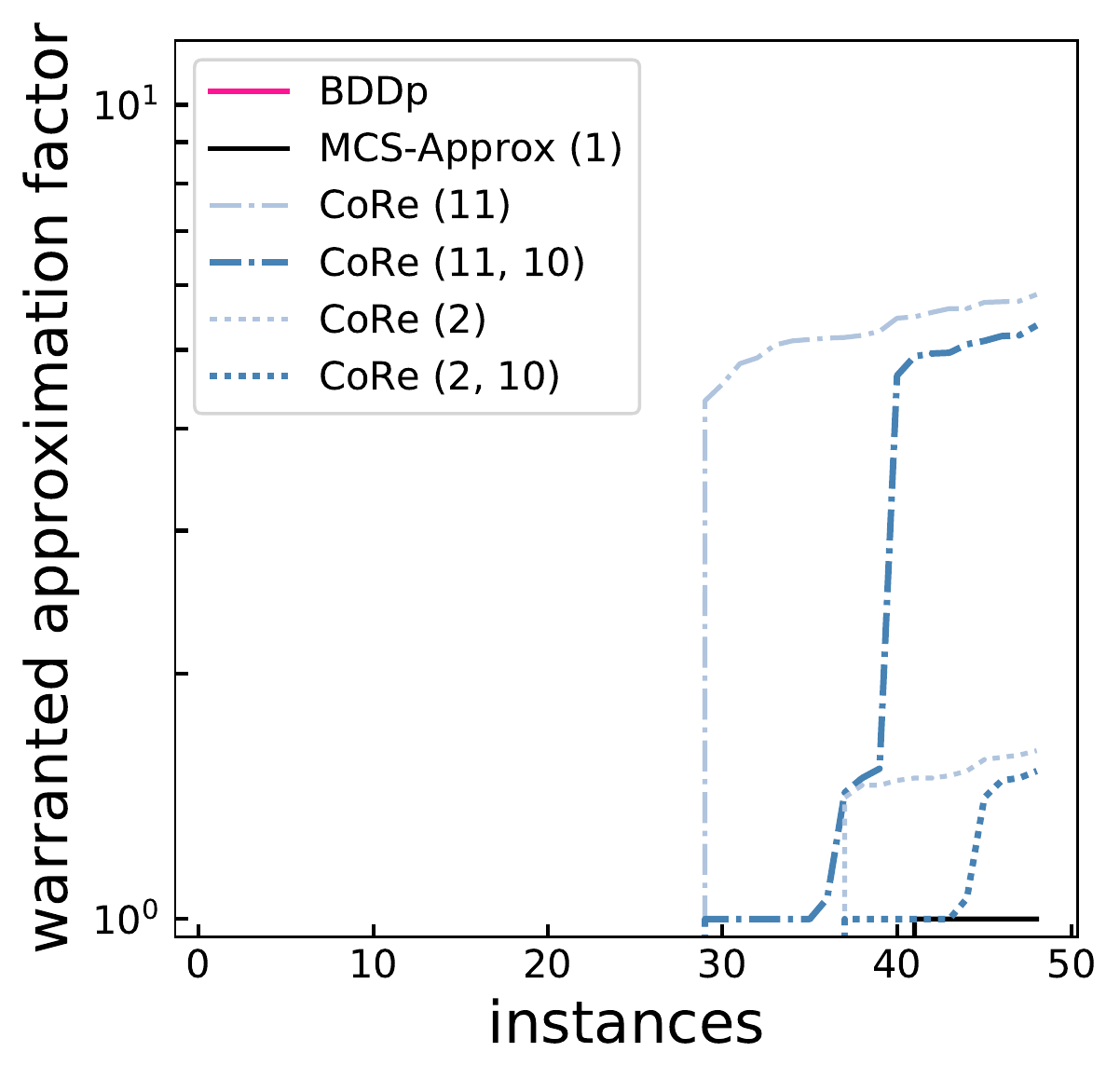}
            \caption[]{}    
            \label{f:sc:comp:eps:o:rcb}
        \end{subfigure}
        \begin{subfigure}[b]{0.3\textwidth}  
            \centering 
            \includegraphics[width=\textwidth]{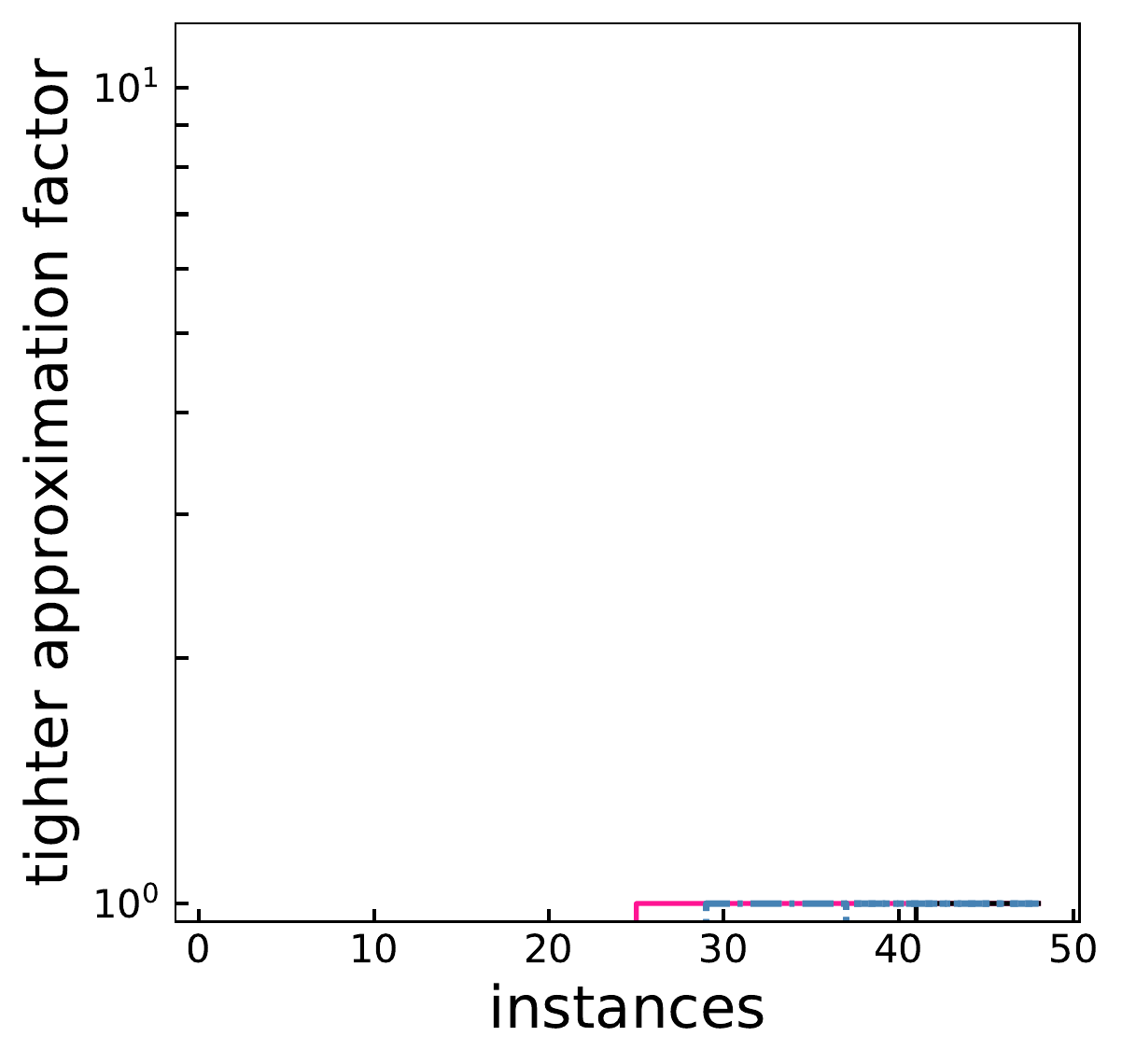}
            \caption[]{}    
            \label{f:sc:comp:eps:r:rcb}
        \end{subfigure}
        \begin{subfigure}[b]{0.3\textwidth}  
            \centering 
            \includegraphics[width=\textwidth]{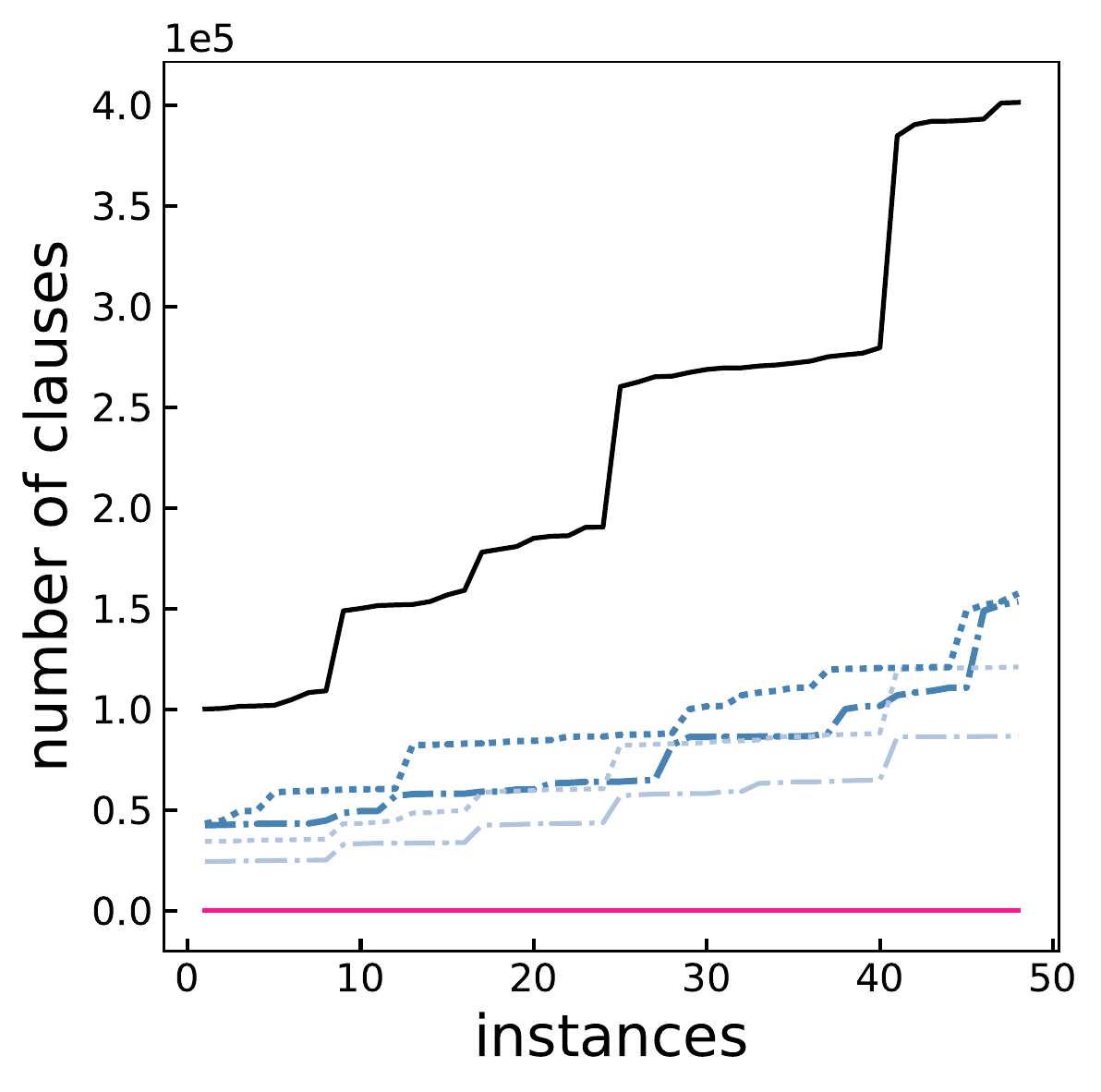}
            \caption[]{}    
            \label{f:sc:comp:nc:rcb}
        \end{subfigure}
        \\
        \begin{subfigure}[b]{0.3\textwidth}  
            \centering 
            \includegraphics[width=\textwidth]{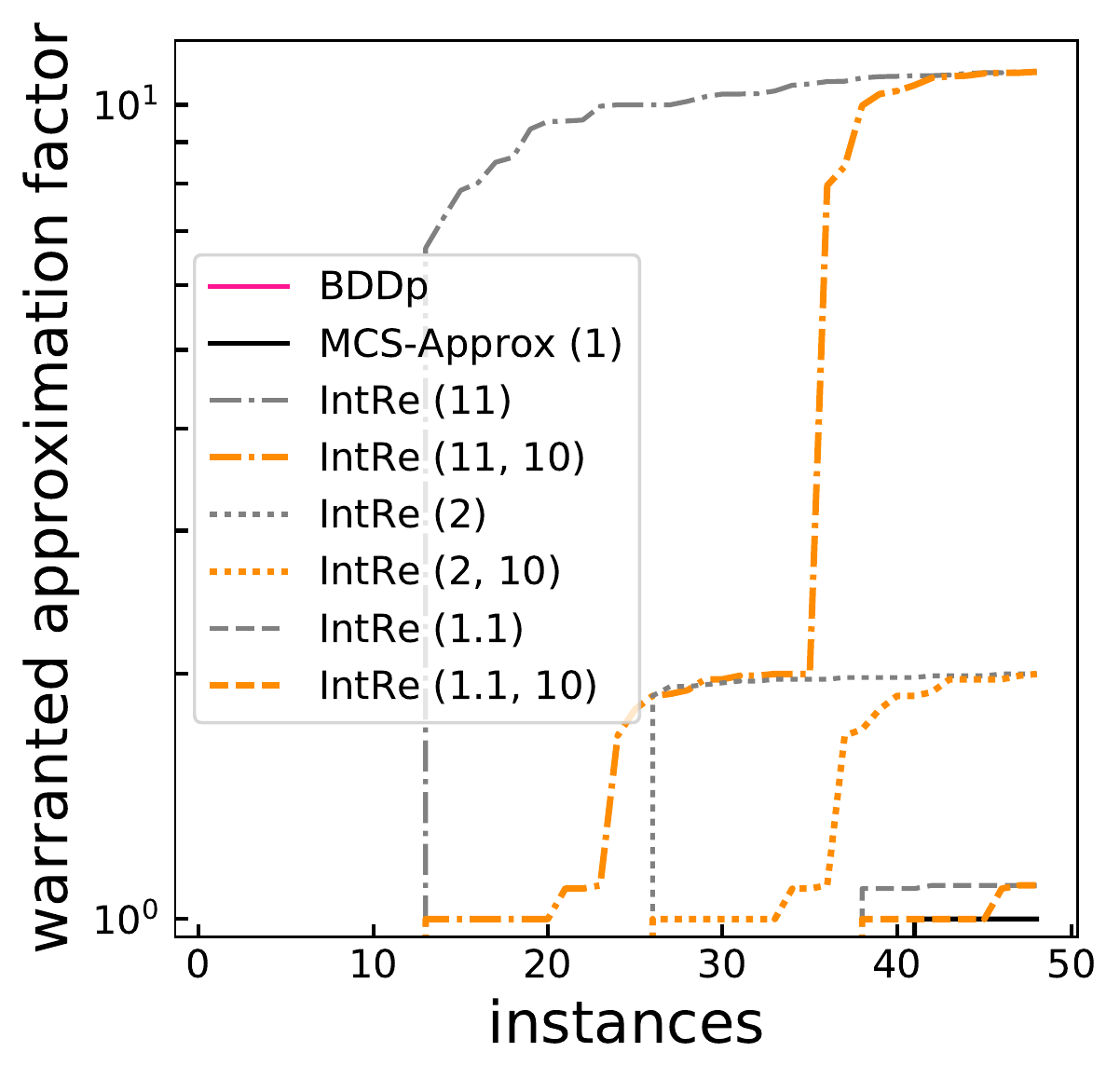}
            \caption[]{}    
            \label{f:sc:comp:eps:o:rib}
        \end{subfigure}
        \begin{subfigure}[b]{0.3\textwidth}  
            \centering 
            \includegraphics[width=\textwidth]{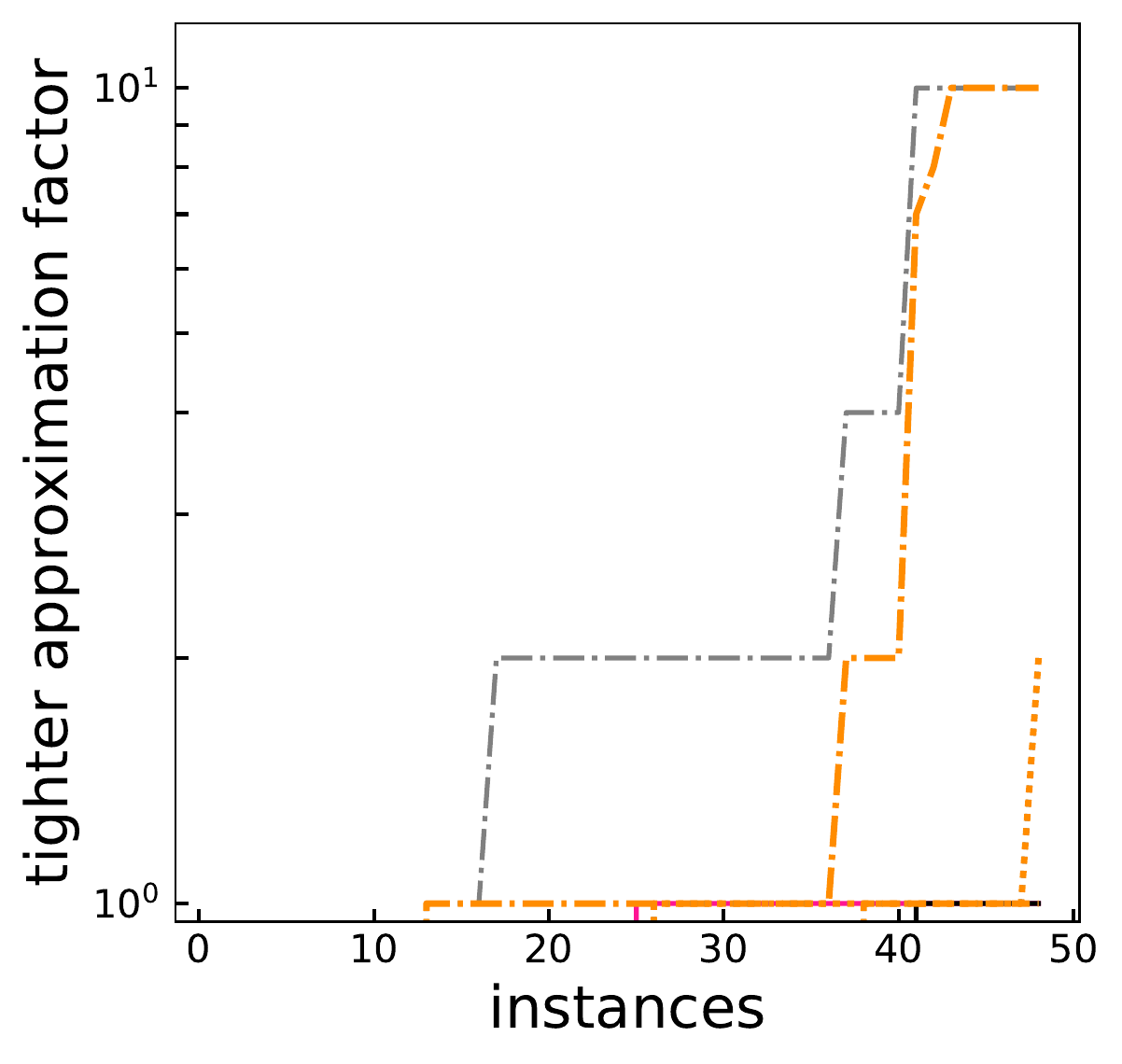}
            \caption[]{}    
            \label{f:sc:comp:eps:r:rib}
        \end{subfigure}
        \begin{subfigure}[b]{0.3\textwidth}  
            \centering 
            \includegraphics[width=\textwidth]{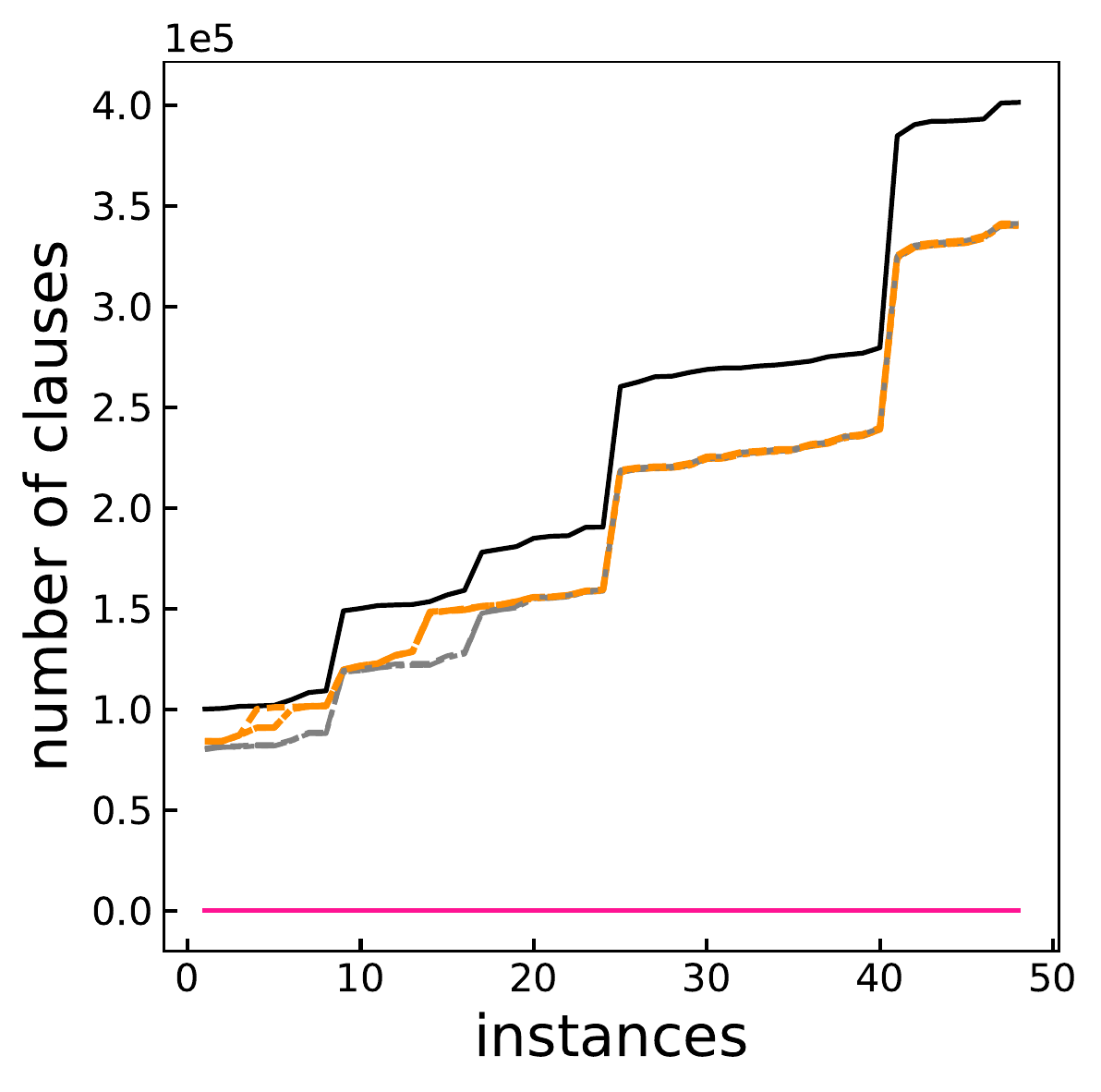}
            \caption[]{}    
            \label{f:sc:comp:nc:rib}
        \end{subfigure}
        \caption{Results on the approximation factor (first two columns) of the output sets,
        and the number of clauses needed to encode the objective functions (last column)
        by the \rcb algorithms (first row) and by the \rib algorithms (last row) on the
        MSCP instances.}
        \label{f:sc:comp:eps}
    \end{figure}

The results in the figures evaluate the algorithms
outputs and total encoding sizes considering
a time limit of 10 minutes per instance and a memory limit of 8GB.
The figures show the empirical cumulative distributive function over the number of instances
for: the hypervolume indicator of the output sets (Figures~\ref{f:sc:comp:hv:rcb} and~\ref{f:sc:comp:hv:rib});
the total number of clauses for the encoding of the objective functions
(Figures~\ref{f:sc:comp:nc:rcb} and~\ref{f:sc:comp:nc:rib});
the warranted approximation factor given by the $\epsilon$-indicator value for the output sets
considering the respective lower bound set, if known
(Figures~\ref{f:sc:comp:eps:o:rcb} and~\ref{f:sc:comp:eps:o:rib});
and a tighter approximation factor given by the smallest $\epsilon$-indicator value obtained considering the Pareto front (if known) and the lower bound sets returned by all algorithms
(Figures~\ref{f:sc:comp:eps:r:rcb} and~\ref{f:sc:comp:eps:r:rib}).
Hence, for each algorithm, the value $t$ for the second axis at value $i\in\{1,\ldots,48\}$
of the first axis, indicates that the evaluation for $i$ instances was less or equal than $t$,
and was greater than $t$ for $48-i$ instances. Therefore, the closer the line is to the
top and to the left of the plot, the better.
 
Figures~\ref{f:sc:comp:hv:rcb} and~\ref{f:sc:comp:hv:rib} show the hypervolume indicator
after normalizing the objective values of the points in the output sets,
and computed considering the reference point $r=(1,\ldots,1)$.
The normalization was performed by considering the largest values for each coordinate among
all output sets from all algorithms and multiplying them by 1.1.
This ensures that, when considering
$r=(1,\ldots,1)$ as the reference point, all points in the output sets contribute to the
hypervolume indicator. 
Figure~\ref{f:sc:comp:hv:rcb} shows that $\rcb(2)$ and $\rcb(11)$ already
perform better than $\mcskp(1)$, \ie, when \rcb searches for an approximation set and stops
after the first iteration, the returned approximation sets
(for $\varepsilon=1$ and $\varepsilon=10$) have greater hypervolume value
than the set of nondominated solutions found by $\mcskp(1)$ in these instances.
The reason for this seems to be related to the fact that the true approximation factor
(\ie, the $\epsilon$-indicator value with the Pareto front as the reference set)
of the returned approximation sets seems to be very close to 1 even though the initial
$\varepsilon$ value may be high. This can be observed in Figures~\ref{f:sc:comp:eps:o:rcb}
and Figure~\ref{f:sc:comp:eps:r:rcb}, where Figure~\ref{f:sc:comp:eps:r:rcb} shows that,
for approximately 20 instances, the tighter approximation factor is very close to 1 and much
smaller than the approximation factor warranted by the algorithms, as observed
in Figure~\ref{f:sc:comp:eps:r:rcb}.
If the \rcb algorithm is allowed to continue to run for tighter approximation factors then the
quality of the returned sets may be improved, and this is visible in the case of $\rcb(11,10)$.
The hypervolume values achieved with \rcb indicate that, even though the points
in the returned sets may not belong to the Pareto front, such sets dominate regions of
the objective space that are not dominated by $\mcskp(1)$. 

The \rib did not perform as well as \rcb in these instances
(see Figures~\ref{f:sc:comp:hv:rcb} and~\ref{f:sc:comp:hv:rib}),
even though it was able to warrant an approximation factor for more instances than \rib
(see Figure~\ref{f:sc:comp:eps:o:rib}), but these factors were not as good as those warranted
by \rcb (see Figures~\ref{f:sc:comp:eps:r:rib} and~\ref{f:sc:comp:eps:r:rcb}).
One reason for this observation may the fact that \rcb is considering the exact
representation of the first objective function (because all coefficients are 1)
and therefore, the approximation factor for the first objective is one.
Another advantage of \rcb over \rib in these instances was the much smaller number
of clauses needed by the former to encode the objective functions, and even though one of
them is encoded in an exact way, as can be observed in Figures~\ref{f:sc:comp:nc:rcb}
and~\ref{f:sc:comp:nc:rib}.

Although \bdd solved, within the time limit, about half of the instances
(see Figure~\ref{f:sc:comp:eps:r:rcb}) including all instances with $n=100$,
for the other half it was not able to finish within the time and memory limits.
For this reason, in the overall, \rcb(11,10) performed much better than \bdd.

\subsubsection{Development Assurance Problem}
\label{s:exp:comp:mdap}

\noindent The algorithms were ran on 48 instances%
\footnote{\url{https://www.cristal.univ-lille.fr/LION9/sampleB.gz}}
of MDAP, which contain 7000 to 18000 variables, 20000 to 70000 constraints, and 7 objective functions.
Since all coefficients of all objective functions are either -1 or 1 in these instances, only \rib
was tested and compared against a state-of-the-art logic-based
algorithm~\citep{DBLP:conf/sat/Terra-NevesLM17}, referred here as \pmcs.

Figure~\ref{f:devAVA} shows plots analogous to those shown in Section~\ref{s:exp:comp:mscp}.
It shows \rib with two initial settings for $\varepsilon$: 1.1 and 101. Though 101 is a very
large value for $\varepsilon$, this allowed \rcb to find, for most of the instances,
an approximation set within this approximation factor (and sometimes providing a much tighter one). This was much more difficult to achieve for an initial setting of 1.1 (see \rib(1.1) and \rib(101) in Figures~\ref{f:devAVA:eps:o} and~\ref{f:devAVA:eps:r}).
Consequently, \rib(101,10) was able to really improve upon the approximation sets found by
\rib(101), and to outperform \mcskp and \pmcs
(see Figures~\ref{f:devAVA:hv} to~\ref{f:devAVA:eps:r}).
The only downside is that \rib needed almost the same number of clauses as \mcskp(1)
to encode the objective functions.

\begin{figure}[t]
        \centering
        \begin{subfigure}[b]{0.24\textwidth}
            \centering
            \includegraphics[width=\textwidth]{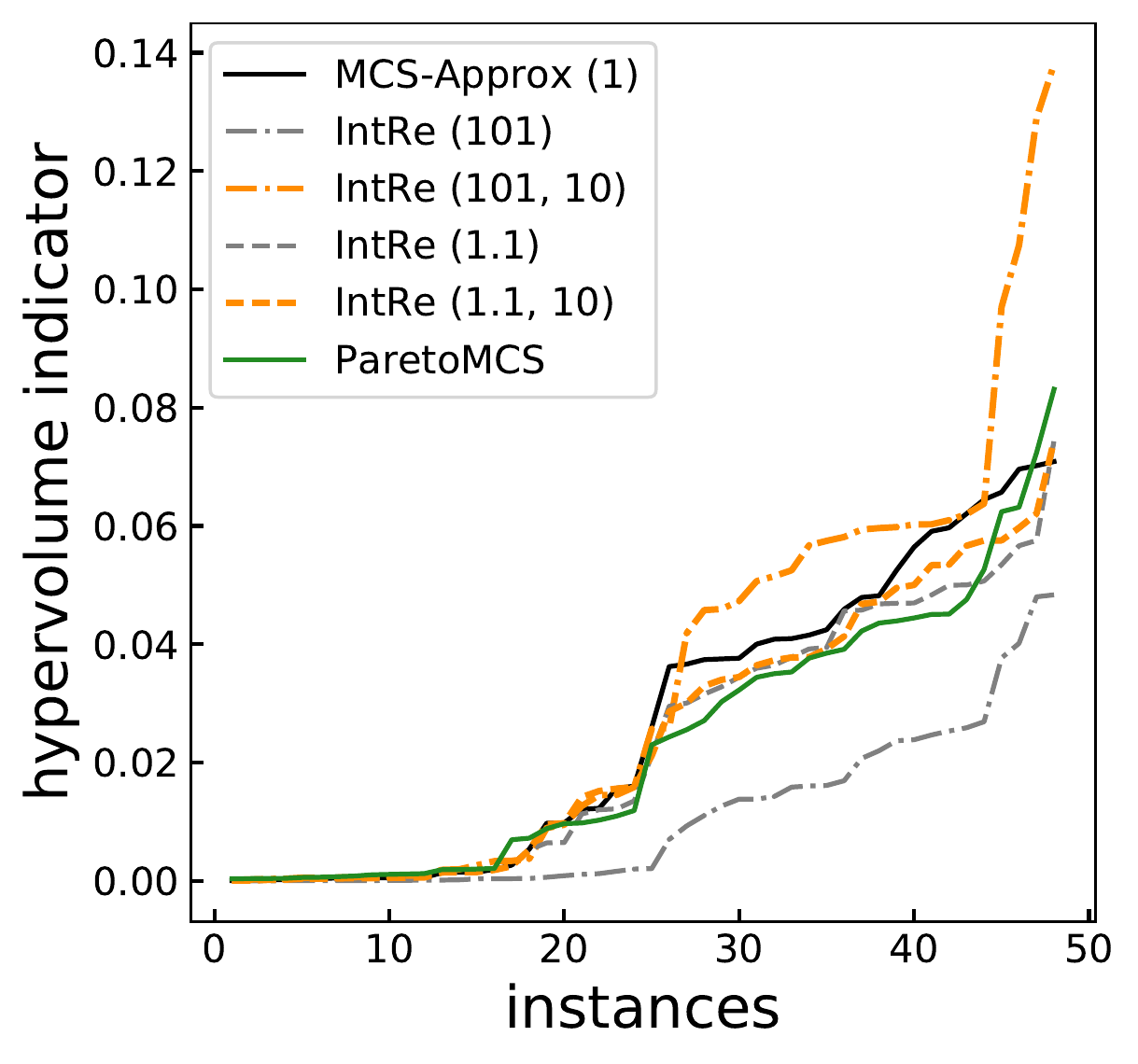}
            \caption[]{}    
            \label{f:devAVA:hv}
        \end{subfigure}
        \begin{subfigure}[b]{0.24\textwidth}  
            \centering 
            \includegraphics[width=\textwidth]{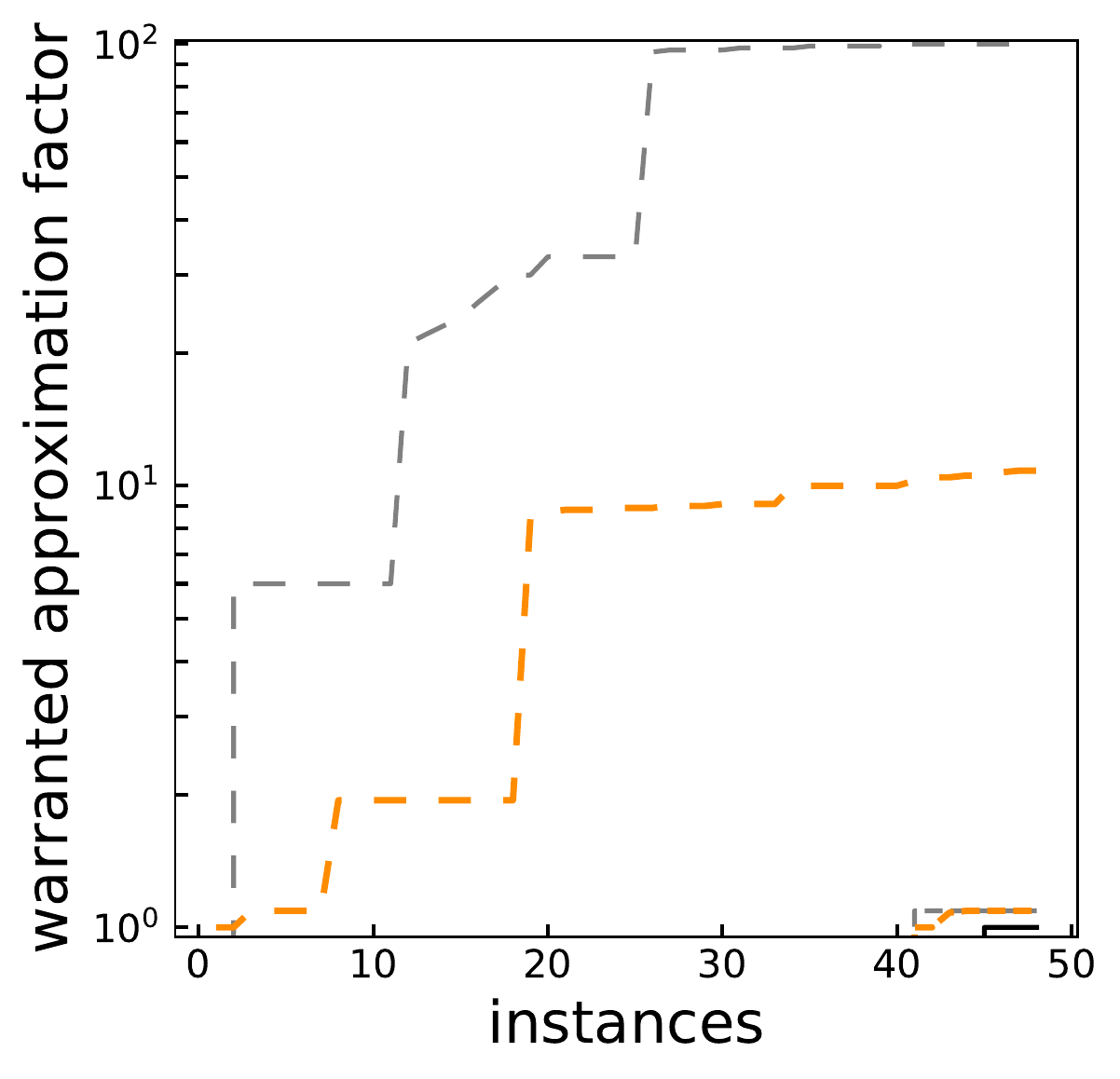}
            \caption[]{}    
            \label{f:devAVA:eps:o}
        \end{subfigure}
        \begin{subfigure}[b]{0.24\textwidth}  
            \centering 
            \includegraphics[width=\textwidth]{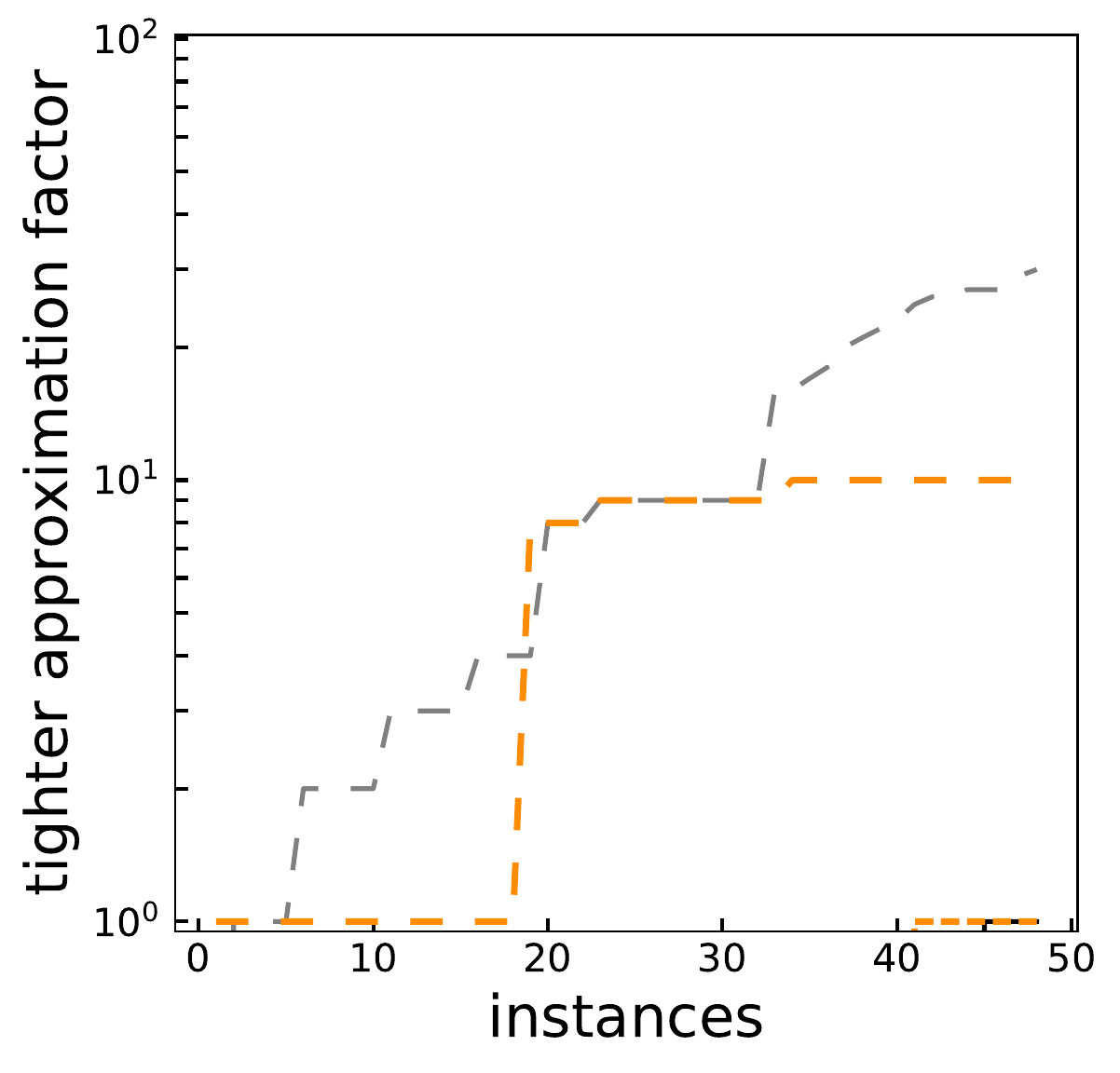}
            \caption[]{}    
            \label{f:devAVA:eps:r}
        \end{subfigure}
        \begin{subfigure}[b]{0.23\textwidth}  
            \centering 
            \includegraphics[width=\textwidth]{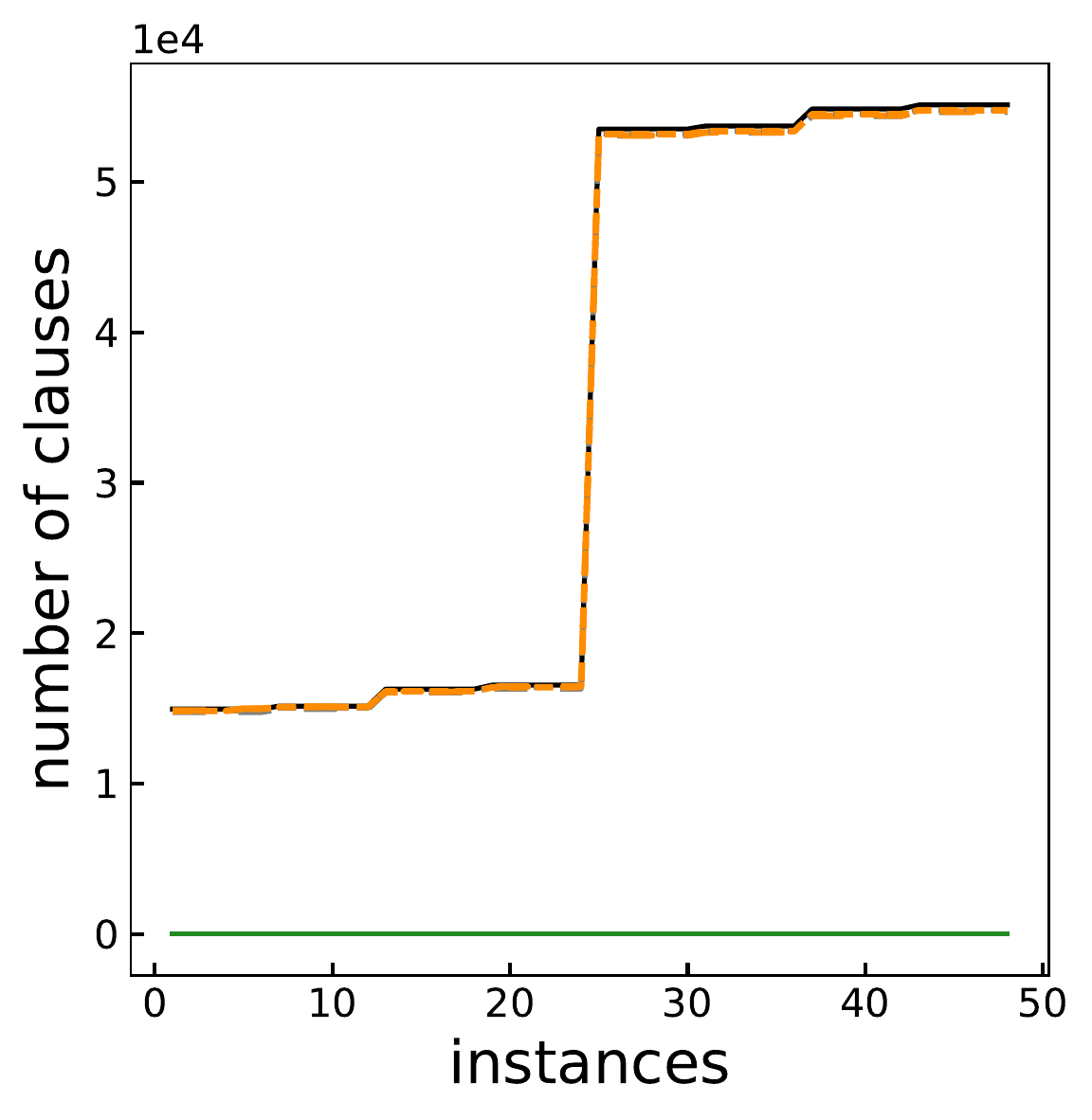}
            \caption[]{}    
            \label{f:devAVA:nc}
        \end{subfigure}
        \caption{Results on the: \subref{f:devAVA:hv} hypervolume indicator;
        \subref{f:devAVA:eps:o}-\subref{f:devAVA:eps:r} approximation factor;
        and \subref{f:devAVA:nc} number of clauses
        of \rib algorithms on the MDAP instances.} 
        \label{f:devAVA}
    \end{figure}

It is not always obvious whether \rib or \rcb will perform better than the other for a given
problem instance, nor with which parameters settings, though for very large objective
functions, \ie, many variables and large coefficients, \rcb may be preferable due to the
smaller number of clauses required.
However, the results shown in this section indicate that they can be competitive against the
state-of-the-art while warranting that the returned approximation set is within some
approximation factor, provided that they are able to complete, at least,
one iteration within the time limit. Moreover, depending on how the value of the desired
approximation factor is updated within \rib and \rcb, the algorithm is able to tighten
the warranted approximation factor along the run.
Additionally, the results also show that it can be helpful to first search for an approximation
set before searching for the whole Pareto front. This allows to ignore a larger portion of
the objective space sooner, and to find better approximation sets, with respect to the
hypervolume indicator, than those found by the algorithms that search for nondominated points
from the beginning of its execution.

\noindent 

\section{Conclusions}
\label{s:conc}

\noindent This paper shows that with a unary encoding of the objective function values, there is
a one-to-one correspondence between the nondominated points and Minimal Correction
Subsets (MCSs). This implies that, with such encoding, any MCS enumeration algorithm can be
used to enumerate all the nondominated points.
This paper also proposes two approximation versions of such encoding which ensure
that MCS enumeration algorithms provide a $(1+\varepsilon)$-approximation set of the
Pareto front, where $\varepsilon\geqslant 0$.
Additionally, two algorithms based on such approximations and on the idea of sequentially
tightening $\varepsilon$ until the set of solutions found has reached a desired approximation
factor, or corresponds to the Pareto front, were proposed. Even if interrupted sooner, \eg, due to reaching
a pre-defined time limit, the algorithms return the set of solutions found, and if at least
the first iteration is completed, then they also provide a lower bound set and, consequently,
an approximation factor tighter than the parameterized one.
In fact, the approximation based simply on modifying the objective function coefficients
can be used with any MOBO solver to obtain a $(1+\varepsilon)$-approximation set.

The preliminary results presented here show that such algorithms based on re-approximations
can provide better anytime performance than the version dedicated to searching only
nondominated points ($\varepsilon=0$) while requiring smaller encodings.
These algorithms find particular relevance for instances for which the Pareto front is very
large, or for which enumeration algorithms are not able to find the full Pareto front within
a reasonable amount of time.
These results show the potential of using such approximation versions of the encoding and the
idea of iteratively finding better and better approximation sets. It remains unclear for
which type of instances each approximation is the best for and with which parameter settings.
This analysis is left for future work.

\noindent 

\vspace{0.5cm}

\section*{Acknowledgements}
\addcontentsline{toc}{section}{\numberline{}Acknowledgements}
\noindent This work was supported by Portuguese national funds through FCT - Fundação para a 
Ciência e Tecnologia, I.P., under projects UIDB/50021/2020, PTDC/CCI-COM/31198/2017 and 
DSAIPA/AI/0044/2018. 
The work by José Rui Figueira was also supported by Portuguese national funds through the FCT 
under the project UIDB/00097/2020. 

%%\vfill\newpage

\vspace{0.5cm}

\section*{References}

%%\textcolor{red}{The title References is missing}

\addcontentsline{toc}{section}{\numberline{}References}

\bibliographystyle{model2-names}
%%\bibliography{Bib_MSC}
\bibliography{Bib_MSC,sat,references}

%%\vfill\newpage
%%%%%%%%%%%%%%%%%%%%%
%%\section*{Appendix A}%%
%%%%%%%%%%%%%%%%%%%%%

%\paragraph{Illustrative example}
\end{document}